%% file: neurips_2020.tex
\documentclass{article}

    \PassOptionsToPackage{numbers, sort, compress}{natbib}

    \usepackage[final]{neurips_2020}

\usepackage[utf8]{inputenc} 
\usepackage[T1]{fontenc}    
\usepackage{hyperref}       
\usepackage{url}            
\usepackage{booktabs}       
\usepackage{amsfonts}       
\usepackage{nicefrac}       
\usepackage{microtype}      

\usepackage[dvipsnames]{xcolor}

\include{mlVecMat}
\newcommand{\methodname}{TunaMH}
\newtheorem{statement}{Statement}[]

\newcommand{\bigTheta}{\mathcal{O}\!\!\!\!\raisebox{1pt}{\text{--}}\,}
\newcommand*\samethanks[1][\value{footnote}]{\footnotemark[#1]}

\title{Asymptotically Optimal Exact Minibatch Metropolis-Hastings}

\author{%
    Ruqi Zhang\thanks{Equal contribution.} \\
    Cornell University\\
    \texttt{rz297@cornell.edu} \\
    \And
    A. Feder Cooper\samethanks \\
    Cornell University \\
    \texttt{afc78@cornell.edu} \\
    \And
    Christopher De Sa\\
    Cornell University \\
    \texttt{cdesa@cs.cornell.edu} \\
}

\begin{document}

\maketitle

\begin{abstract}
Metropolis-Hastings (MH) is a commonly-used MCMC algorithm, but it can be intractable on large datasets due to requiring computations over the whole dataset. In this paper, we study \emph{minibatch MH} methods, which instead use subsamples to enable scaling. We observe that most existing minibatch MH methods are inexact (i.e. they may change the target distribution), and show that this inexactness can cause arbitrarily large errors in inference. We propose a new exact minibatch MH method, \emph{\methodname{}}, which exposes a tunable trade-off between its batch size and its theoretically guaranteed convergence rate. We prove a lower bound on the batch size that any minibatch MH method \emph{must} use to retain exactness while guaranteeing fast convergence---the first such bound for minibatch MH---and show \methodname{} is asymptotically optimal in terms of the batch size. Empirically, we show \methodname{} outperforms other exact minibatch MH methods on robust linear regression, truncated Gaussian mixtures, and logistic regression.

\end{abstract}

\section{Introduction} \label{sec:intro}
Bayesian inference is widely used for probabilistic modeling of data. Specifically, given a dataset $\mathcal{D} = \{x_i\}_{i=1}^N$ and a $\theta$-parameterized model, it aims to compute the posterior distribution 
\[\textstyle
\pi(\theta) \propto \exp\left(-\sum_{i=1}^N U_i(\theta)\right), \text{where } U_i(\theta) = - \log p(x_i|\theta) - \frac{1}{N} \log p(\theta).
\]
Here $p(\theta)$ is the prior and the $p(x_i|\theta)$ give the likelihood of observing $x_i$ given the parameter $\theta$. We assume the data are conditionally independent given $\theta$. The $U_i$ have a natural interpretation as component \emph{energy functions} with $\pi$ acting as a Gibbs measure. In practice, computing $\pi(\theta)$ is often intractable and thus requires using approximate methods, such as Markov chain Monte Carlo (MCMC). MCMC uses sampling to estimate the posterior and is guaranteed to converge asymptotically to the true distribution, $\pi$ \cite{brooks2011handbook}.  

The Metropolis-Hastings (MH) algorithm \citep{hastings1970mh,metropolis1953equation} is one of the most commonly used MCMC methods. In each step, MH generates a proposal $\theta'$ from a distribution $q(\cdot|\theta)$, and accepts it with probability 
\begin{equation}
\label{eqnMHaccprob}
\textstyle
a(\theta,\theta') = \min\left(1, \frac{\pi(\theta')q(\theta|\theta')}{\pi(\theta)q(\theta'|\theta)}\right) = \min\left(1, \exp\big( \sum_{i=1}^N (U_i(\theta) - U_i(\theta')) \big) \cdot \frac{q(\theta|\theta')}{q(\theta'|\theta)}\right).
\end{equation}
If accepted, the chain transitions to $\theta'$; otherwise, it remains at the current state $\theta$. This accept/reject step can be quite costly when $N$ is large, since it entails computing a sum over the entire dataset. 

Prior work has proposed many approaches to mitigate the cost of this decision step \citep{bardenet2017mcmc}. One popular approach involves introducing stochasticity: instead of computing over the entire dataset, a subsample, or \emph{minibatch}, is used to compute an approximation. These minibatch MH methods can be divided into two classes, \emph{exact} and \emph{inexact}, depending on whether or not the target distribution $\pi$ is necessarily preserved. 
Inexact methods introduce asymptotic bias to the target distribution, trading off correctness for speedups~\cite{bardenet2014towards,korattikara2014austerity,seita2016efficient,quiroz2019speeding,quiroz2016block}. Exact methods either require impractically strong constraints on the target distribution \cite{maclaurin2015firefly, zhang2019poisson}, limiting their applicability in practice, or they negatively impact efficiency, counteracting the speedups that minibatching aims to provide in the first place \cite{banterle2015accelerating,cornish2019scalable}. Moreover, all existing exact methods operate on the belief that there is a trade-off between batch size and convergence rate---between scalability and efficiency. Yet no prior work formally exposes this trade-off, and most prior work gives no convergence rate guarantees. Given these  various considerations, it is not entirely clear how to evaluate which minibatch MH method to use.

In this paper we forge a path ahead to untangle this question. While inexact methods have been prominent recently due to their efficiency, they are not reliable: we show that the stationary distribution of any inexact method can be arbitrarily far from the target $\pi$. This means they can yield disastrously wrong inference results in practice, and it is difficult to tell just how bad those results can be. 

We therefore turn our attention to exact methods and introduce \emph{\methodname{}}.\footnote{\methodname{} since it \emph{tunes} the efficiency-scalability trade-off and uses a Poisson (French for ``fish") variable.} Compared to prior work, we make milder assumptions, which enables \methodname{} to apply to a wider variety of inference tasks. More specifically, we require local rather than global bounds on the target distribution~\cite{maclaurin2015firefly,zhang2019poisson} and do not rely on the Bernstein-von Mises approximation~\cite{cornish2019scalable,bardenet2017mcmc,bierkens2019zig}. \methodname{} is guaranteed to retain sample efficiency in the presence of minibatching: its convergence rate (measured by the spectral gap) is within a constant factor of standard, non-minibatch MH. More importantly, \methodname{} also enables us to rigorously characterize the trade-off between scalability and efficiency. It has a hyperparameter $\chi$, which enables tuning the trade-off between expected batch size and convergence rate.

By exposing this trade-off, our analysis raises the natural question: \emph{is \methodname{} optimal for this trade-off?} That is, could another exact algorithm use an asymptotically smaller average batch size while having the same convergence rate guarantees? We explore this in Section \ref{sec:optimality}; under the same mild assumptions we use to derive \methodname{}, we prove a lower bound on the expected batch size for \emph{any} exact minibatch MH method that can keep a reasonable convergence rate. To our knowledge, we are the first to prove a lower bound of this nature for minibatch MH. Moreover, \methodname{} is \emph{asymptotically optimal} in balancing the expected batch size and convergence rate. It remains exact and efficient while on average using the smallest possible number of samples. In summary:
\begin{itemize}[noitemsep,topsep=0pt, leftmargin=.4cm]
    \item We demonstrate that any inexact minibatch MH method can be arbitrarily inaccurate (Section \ref{sec:inexactproblems}).
    \item We introduce a new exact method, \methodname{} (Section \ref{sec:tuna}), with a lower bound on its convergence rate (in terms of the spectral gap) and a tunable hyperparameter to balance the trade-off between convergence rate and batch size.
    \item We prove a lower bound on the batch size for any exact minibatch MH method given a target convergence rate---the first such lower bound in this area. This result indicates that the expected batch size of \methodname{} is asymptotically optimal in terms of the problem parameters (Section \ref{sec:optimality}).
    \item We show empirically that \methodname{} outperforms state-of-the-art exact minibatch MH methods on robust linear regression, truncated Gaussian mixture, and logistic regression (Section \ref{sec:exp}).
\end{itemize}

\section{Preliminaries and Drawbacks of Prior Minibatch MH Methods}\label{sec:drawbacks}
We first formally define the class of methods that we study theoretically in this paper: minibatch MH methods of the form of Algorithm~\ref{alg:subsampledMH}. This class contains methods that sample a proposal from distribution $q$ (which we always assume results in the chain being ergodic), and choose to accept or reject it by calling some randomized subroutine, $\texttt{SubsMH}$, which outputs $1$ or $0$ for ``accept" or ``reject," respectively. Algorithms in this class have several notable properties. First, $\texttt{SubsMH}$ is \emph{stateless}: each acceptance decision is made independently, without carrying over local state associated with the MH procedure between steps. Many prior methods are stateless~\cite{korattikara2014austerity,bardenet2014towards,seita2016efficient,cornish2019scalable}. We do not consider \emph{stateful} methods, in which the decision depends on previous state; they are difficult to analyze due to running on an extended state space~\cite{andrieu2009pm, quiroz2019speeding}.
Second, $\texttt{SubsMH}$ takes a function that computes energy \emph{differences} $U_i(\theta) - U_i(\theta')$ and outputs an acceptance decision. We evaluate efficiency in terms of how many times $\texttt{SubsMH}$ calls this function, which we term the \emph{batch size} the method uses.
Third, $\texttt{SubsMH}$ takes parameters that bound the maximum magnitude of the energy differences. Specifically, as in \citet{cornish2019scalable}, we assume:
\begin{assumption}\label{assump}
For some constants $c_1, \ldots, c_N \in \R_+$, with $\sum_i c_i = C$, and symmetric function $M: \Theta \times \Theta \rightarrow \R_+$, for any $\theta, \theta' \in \Theta$, the energy difference is bounded by
$|U_i(\theta) - U_i(\theta')|\le c_i M(\theta,\theta')$.
\end{assumption}
One can derive such a bound, which can be computed in $O(1)$ time, for many common inference problems: for example, if each energy function $U_i$ is $L_i$-Lipschitz continuous, then it suffices to set $c_i = L_i$ and $M(\theta, \theta') = \|\theta - \theta'\|$ (See Appendix~\ref{app:experiments} for examples of $c_i$ and $M$ on common problems). Note that the $\texttt{SubsMH}$ method may choose \emph{not} to use these bounds in its decision. We allow this so the form of Algorithm \ref{alg:subsampledMH} can include methods that do not require such bounds.
Most existing methods can be described in this form \cite{korattikara2014austerity,bardenet2014towards,seita2016efficient,cornish2019scalable,banterle2015accelerating}. For example, standard MH can be written by setting $\texttt{SubsMH}$ to a subroutine that computes the acceptance rate $a$ as in (\ref{eqnMHaccprob}) and outputs $1$ (i.e., accept) with probability $a$. 

Such minibatch MH methods broadly come in two flavors: \emph{inexact} and \emph{exact}. We next establish the importance of being exact and demonstrate how \methodname{} resolves drawbacks in prior work.

\begin{algorithm}[t]
\caption{Stateless, Energy-Difference-Based Minibatch Metropolis-Hastings}
\label{alg:subsampledMH}
\begin{algorithmic}
\STATE \textbf{given: } state space $\Theta$, energy functions $U_1, \ldots, U_N: \Theta \rightarrow \R$, proposal dist. $q$, initial state $\theta \in \Theta$ 
\STATE \textbf{given: } parameters $c_1,\ldots,c_N$, $C$, $M$ from Assumption~\ref{assump}, randomized algorithm \texttt{SubsMH}
\LOOP
    \STATE \textbf{sample} $\theta' \sim q(\cdot|\theta)$
    \STATE \textbf{define function} $\Delta U: \{1, \ldots, N\} \rightarrow \R$, such that
    $\Delta U(i) = U_i(\theta) - U_i(\theta')$
    \STATE \textbf{call subroutine} $o \leftarrow \texttt{SubsMH}(\Delta U, N, q(\theta|\theta') / q(\theta'|\theta), c_1,\ldots,c_N, C, M(\theta,\theta'))$
    \STATE \textbf{if} $o = 1$, \textbf{update} $\theta  \leftarrow \theta'$
\ENDLOOP
\end{algorithmic}
\end{algorithm}

\subsection{The Importance of Being Exact}\label{sec:inexactproblems}

Inexact methods are popular due to helping scale MH to new heights~\cite{bardenet2014towards,korattikara2014austerity,seita2016efficient,quiroz2019speeding}. They approximate the MH acceptance ratio to within an error tolerance ($> 0$), trading off exactness for efficiency gains. Surprisingly, the bias from inexactness can be arbitrarily large even when the error tolerance is small.

\begin{theorem}\label{statement:counterexample}
Consider any minibatch MH method of the form in Algorithm~\ref{alg:subsampledMH} that is inexact (i.e. does not necessarily have $\pi$ as its stationary distribution for all $\pi$ satisfying Assump.~\ref{assump}). For any constants $\delta\in (0,1)$ and $\rho>0$, there exists a target distribution $\pi$ and proposal distribution $q$ such that if we let $\tilde{\pi}$ denote a stationary distribution of the inexact minibatch MH method on this target, it satisfies
\[
\operatorname{TV}(\pi,\tilde{\pi})\ge\delta \text{ and } \operatorname{KL}(\pi,\tilde{\pi})\ge \rho.
\]
where TV is the total variation distance and \text{KL} is the Kullback–Leibler divergence.
\end{theorem}
Theorem \ref{statement:counterexample} shows that when using any inexact method, there always exists a target distribution $\pi$ (factored in terms of energy functions $U_i$) and proposal distribution $q$ such that it will approximate $\pi$ arbitrarily poorly. This can happen even when individual errors are small; they can still accumulate a very large overall error. We prove Theorem~\ref{statement:counterexample} via a simple example---a random walk along a line, in which the inexact method causes the chain to step towards one direction more often than the other, even though its steps should be balanced (Appendix~\ref{app:proof:counterexample}). Note that it may be possible to avoid a large error by using some specific proposal distribution, but such a proposal is hard to know in general. 

We use AustereMH~\cite{korattikara2014austerity} and MHminibatch~\cite{seita2016efficient} to empirically validate Theorem~\ref{statement:counterexample}. For these inexact methods, we plot density estimates with the number of states $K=200$ in Figure \ref{fig:counter-example}a (see Appendix \ref{app:experiments:counterexample} for using other $K$); the stationary distribution diverges from the target distribution significantly. Moreover, the TV distance between the density estimate and the true density increases as $K$ increases on this random walk example (Figure \ref{fig:counter-example}b). By contrast, our exact method (Section \ref{sec:tuna}) keeps a small TV distance on all $K$ and estimates the density accurately with an even smaller average batch size. We also tested AustereMH on robust linear regression, a common task, to show that the error of inexact methods can be large on standard problems (Appendix~\ref{app:experiments:counterexample}).

\begin{figure*}[t!]
    \centering
    \begin{tabular}{cccc}		
    \includegraphics[width=4.3cm]{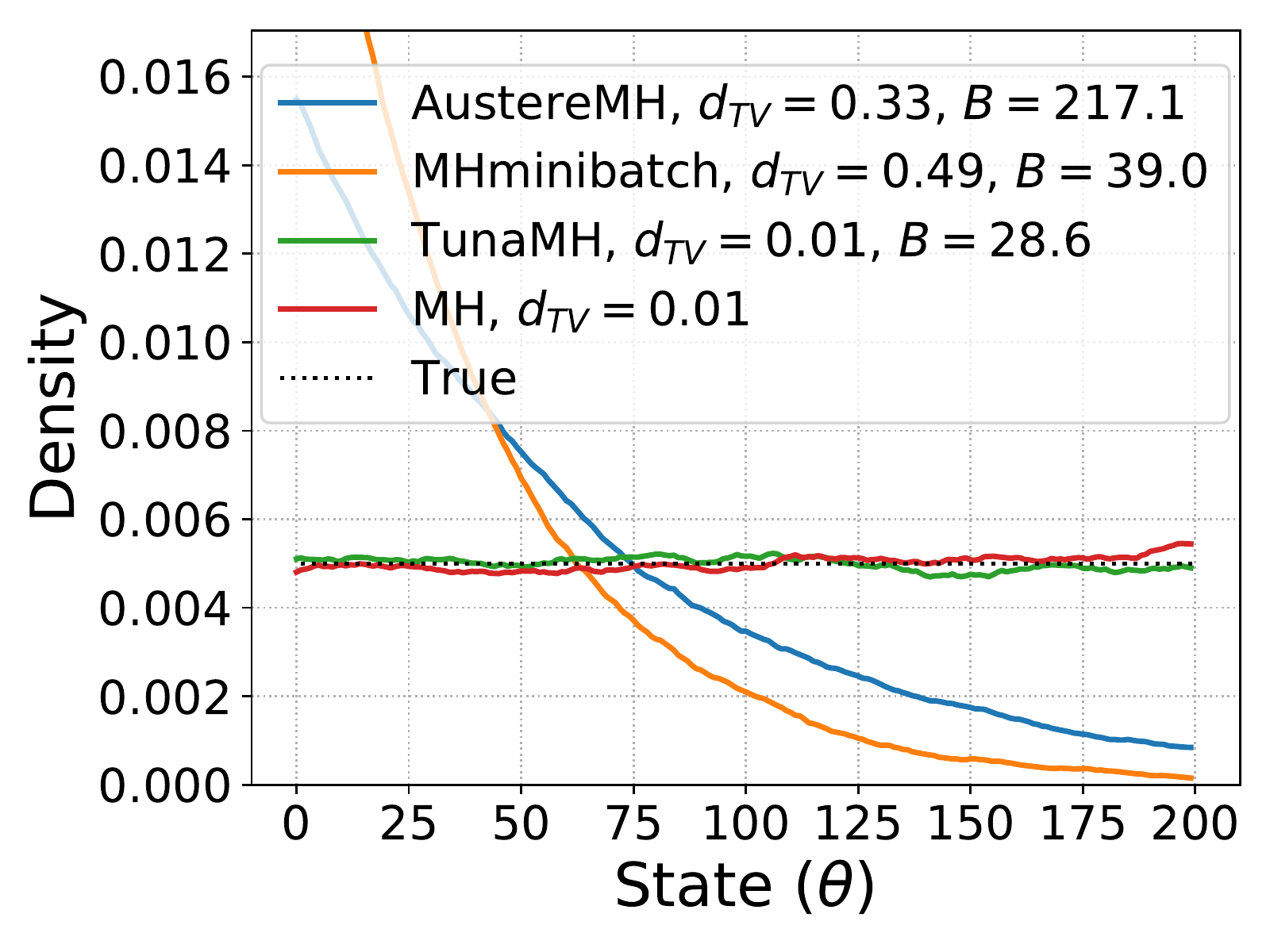} &
    \hspace{-4mm}
    	\includegraphics[width=4.3cm]{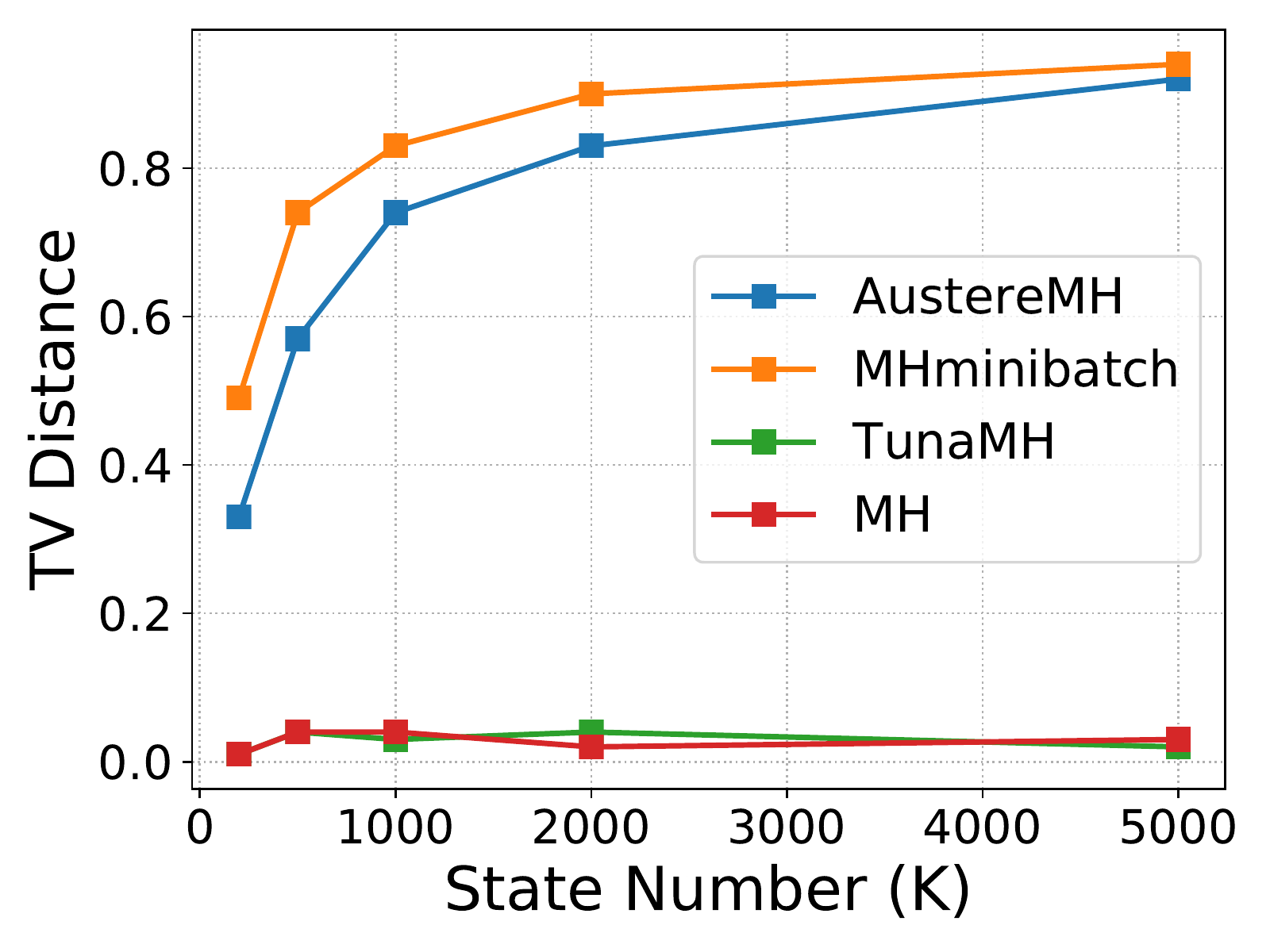}  &
    		\hspace{-4mm}
    	\includegraphics[width=4.3cm]{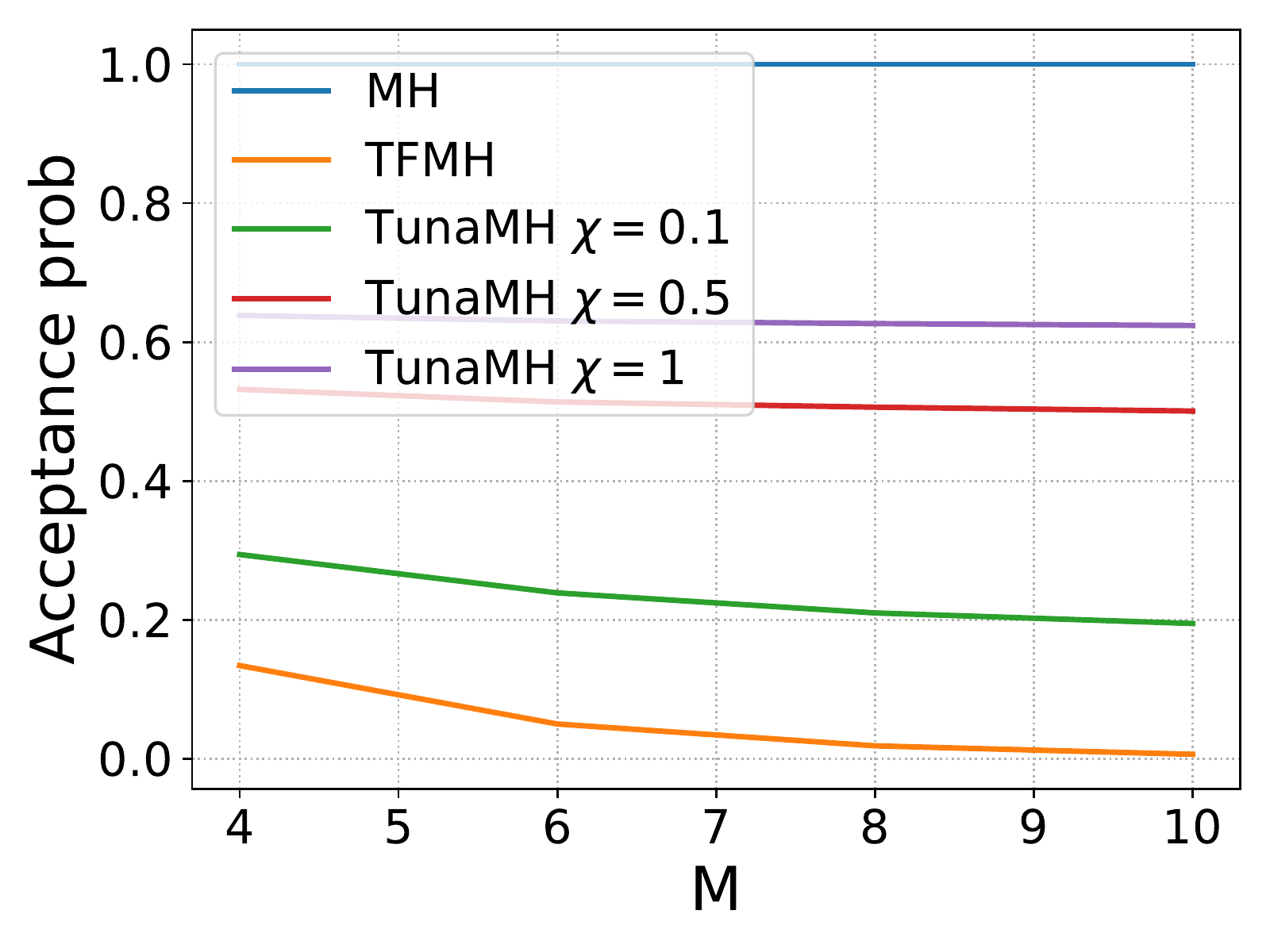} &
    	\\		
    	(a) &
    	(b) &
    	(c) 
    	\hspace{-0mm}\\		
    \end{tabular}
    \caption{Existing MH method issues. (a)-(b) Inexact methods can diverge a lot from true distribution. ``$d_{TV}$'' and ``$B$'' denote the TV distance and the batch size respectively. (c) SMH has low and \methodname{} with different values of hyperparameter $\chi$ has high acceptance rates.}
    \label{fig:counter-example}
\end{figure*}

\subsection{Issues with Existing Exact Methods} \label{sec:issue-of-exact}
This observation suggests that we should be using exact methods when doing minibatch MH. However, existing approaches present additional drawbacks, which we discuss below.

\textbf{Factorized MH and Scalable MH} 
are stateless, exact minibatch methods. Factorized MH (FMH) decomposes the acceptance rate into a product of factors, which allows for rejecting a proposal based on a minibatch of data  \citep{ceperley1995path,christen2005markov,banterle2015accelerating}. Truncated FMH (TFMH) is a FMH variant that maintains geometric ergodicity; it falls back on standard MH in a step when the bound on the factors reaches a certain threshold \citep{cornish2019scalable}. No matter how this threshold is set, we can construct tasks where TFMH is either arbitrarily inefficient (rejecting arbitrarily often, slowing convergence), or degrades entirely to standard MH.
\begin{statement}\label{statement:fmh}
For any constant $p\in (0,1)$, there exists a target distribution such that TFMH either has an acceptance rate which is less than p times that of standard MH, or it completely degrades to standard MH (summing over the whole dataset at each step).
\end{statement}
We prove this statement in Appendix \ref{app:proof:smh} using an example of a uniform distribution along a line, where we let $x_i$ take one of two values, $\{-M/N, M/N\}$ with $M>0$. We show that the acceptance rate of TFMH can be arbitrarily low by increasing $M$, which we also empirically verify in Figure \ref{fig:counter-example}c.

To improve the acceptance rate of TFMH, Scalable MH (SMH) introduces control variates, which approximate $U_i$ with a Taylor series around the mode \citep{cornish2019scalable}. However, it only works with unimodal posteriors and high-quality Bernstein-von Mises approximations---conditions that do not hold for many common inference tasks.

\textbf{PoissonMH} is a stateless minibatch MH method adapted from an algorithm designed for scaling Gibbs sampling on factor graphs \cite{zhang2019poisson}. However, unlike our method, it requires strong assumptions---specifically, a global upper bound on the energy. Such an upper bound usually does not exist and, even if it does, can be very large, resulting in an impractically large batch size.

\textbf{FlyMC} is a stateful method, which means it uses auxiliary random variables to persist state across different MH steps \cite{maclaurin2015firefly}. It requires a lower bound on the likelihood function, which is typically more demanding than Assumption~\ref{assump} and does not have theoretical performance guarantees.

\textbf{Other exact methods} exist based on Piecewise Deterministic Markov Processes \cite{bouchard2018bouncy,bierkens2019zig}. They require regularity conditions only available for some problems, so their practical utility is limited.

\section{\methodname: Asymptotically Optimal Exact MH} \label{sec:tuna}

In this section, we present our method, \methodname{}, which evades the issues of prior exact methods discussed in Section \ref{sec:issue-of-exact}.
Like SMH \cite{cornish2019scalable}, our method works on distributions for which an \emph{a priori} bound on the energy differences is known (Assumption \ref{assump}). 

Our algorithm, presented in Algorithm~\ref{alg:poisson-mh}, takes as parameters $c_1, \ldots, c_N$, $C$, and $M$ from Assumption \ref{assump}, along with an additional hyperparameter, $\chi>0$.
It proceeds in four steps.
First, like any MH method, it generates a proposal $\theta'$ from given distribution $q$.
Second, it samples a batch size $B$ from a Poisson distribution. This makes the expected number of energy functions $U_i$ evaluated by our method at each step
$\mathbf{E}[B] = \chi C^2 M^2(\theta, \theta') + C M(\theta, \theta')$\footnote{Note that $\mathbf{E}[B]$ is typically  $<<$ $N$ and can be decreased using small step sizes. If, however, $\mathbf{E}[B]>N$, then we can simply use standard MH in that iteration, similar to TFMH.}.
Importantly, this means the batch size may vary from iteration to iteration, and the expected size depends on $\theta$ and $\theta'$. For example, \methodname{} may tend to set $B$ larger for larger-distance proposals with a higher $M(\theta, \theta')$.
Third, it samples (with replacement) a minibatch of size $B$, but for each data point it samples, it has some probability of \emph{ejecting} this point from the minibatch.
Finally, it accepts the proposed $\theta'$ with some probability, computed using a sum over the post-ejection minibatch.
Our method can be derived by carefully replacing the auxiliary variables in PoissonMH with \emph{local} Poisson variables whose distributions change each iteration depending on the pair $(\theta, \theta')$ (Appendix \ref{app:algo-derivation}). By construction \methodname{} is exact; it preserves the target distribution $\pi$ as its stationary distribution.
This is because \methodname{} is \emph{reversible}, meaning its transition operator $T$ satisfies $\pi(\theta) T(\theta , \theta') = \pi(\theta') T(\theta' , \theta)$ for any $\theta, \theta' \in \Theta$. This is a common condition that guarantees that a MCMC method has $\pi$ as its stationary distribution~\cite{levin2017markov, brooks2011handbook}.


Compared to previous exact methods, a significant benefit of \methodname{} is that we can prove theoretical guarantees on its efficiency. Specifically, its convergence speed is guaranteed to be close to standard MH and $\chi$ allows us to control how close. To show this, we lower bound the convergence rate of \methodname{} in terms of the \emph{spectral gap}, which is commonly used to characterize convergence speed in the MCMC literature \cite{rudolf2011explicit,hairer2014spectral,levin2017markov,zhang2019poisson,zhang2020amagold}. The larger the spectral gap, the faster the chain converges.

\begin{algorithm}[t]
  \caption{\methodname{}}
  \begin{algorithmic}
    \label{alg:poisson-mh}
    \STATE \textbf{given:} initial state $\theta \in \Theta$; proposal dist. $q$; hyperparameter $\chi$; Asm.~\ref{assump} parameters $c_i$, $C$, $M$
    \LOOP
      \STATE \textbf{propose} $\theta'\sim q(\cdot|\theta)$ and \textbf{compute} $M(\theta, \theta')$
      \vspace{0.5em}
      \STATE $\triangleright$ Form minibatch $\mathcal{I}$
      \STATE \textbf{sample} $B \sim \text{Poisson}\left( \chi C^2M^2(\theta,\theta')  + CM(\theta,\theta')\right)$
      \STATE \textbf{initialize minibatch indices} $\mathcal{I} \leftarrow \emptyset$  (an initially empty multiset)
      \FOR{$b \in \{1,\ldots,B\}$}
        \STATE \textbf{sample} $i_b$ such that $\mathbf{P}(i_b = i) = c_i/C$, for $i=1\ldots N$
        \STATE \textbf{with probability} $\frac{\chi c_{i_b} C M^2(\theta, \theta') + \frac{1}{2}(U_{i_b}(\theta') - U_{i_b}(\theta) + c_{i_b}M(\theta,\theta'))}{\chi c_{i_b}C  M^2(\theta, \theta') + c_{i_b} M(\theta, \theta')}$ \textbf{add} $i_b$ to $\mathcal{I}$ 
      \ENDFOR
      \vspace{0.5em}
      \STATE $\triangleright$ Accept/reject step based on minibatch $\mathcal{I}$
      \STATE \textbf{compute MH ratio} $r \leftarrow
        \exp\left(2 \sum_{i \in \mathcal{I}} \operatorname{artanh}\left(
        \frac{U_i(\theta) - U_i(\theta') }{c_{i}M(\theta,\theta') (1 + 2 \chi CM(\theta,\theta'))} 
    \right) \right)
        \cdot \frac{q(\theta'|\theta)}{q(\theta|\theta')}$
      \STATE \textbf{with probability} $\min(1,r)$, set $\theta \leftarrow \theta'$
    \ENDLOOP
  \end{algorithmic}
\end{algorithm}
  
\begin{definition}
The \emph{spectral gap} of a reversible Markov chain is the distance between the largest and second-largest eigenvalues of its transition operator. That is, if the eigenvalues of the transition operator are $1 = \lambda_1 > \lambda_2 \ge \lambda_3 \cdots$, then the spectral gap is $\gamma = 1 - \lambda_2$. 
\end{definition}

\begin{theorem}\label{thm:spectral-gap}
\methodname{} (Algorithm~\ref{alg:poisson-mh}) is reversible with stationary distribution $\pi$. Let $\bar\gamma$ denote the spectral gap of \methodname{}, and let $\gamma$ denote the spectral gap of standard MH with the same target distribution and proposal distribution. Then,
  \[
    \textstyle
    \bar{\gamma}
    \ge
    \exp \left(-\frac{1}{\chi} - 2\sqrt{\frac{\log 2}{\chi}} \right)\cdot\gamma.
  \]
\end{theorem}
Intuitively, this theorem (proof in Appendix \ref{app:proof:spectral-gap}) suggests the convergence rate of \methodname{} is at most a constant slower than that of standard MH, and can be increased by adjusting the hyperparameter $\chi$.
Recall that $\chi$ also controls the batch size of \methodname{}. Effectively, this means $\chi$ is a \emph{dial} that allows us to directly tune the trade-off between convergence rate and batch size. When $\chi$ is large, the batch size $B$ is large and the spectral gap ratio, $\bar \gamma / \gamma$, is close to 1: the larger batch size is less scalable but keeps a high convergence rate. Conversely, when $\chi$ is small, the batch size is small and the spectral gap ratio is close to 0: we trade off slow-downs in convergence rate for scalability.
For example, for any $0 < \kappa < 1$, to guarantee the spectral gap ratio $\bar \gamma / \gamma \ge \kappa$ it suffices to set (Appendix~\ref{app:chi-value})
\begin{equation}
    \label{eq:TunaMHEB}
    \textstyle
    \chi = \frac{4}{(1-\kappa)\log(1/\kappa)}, \;\; \text{giving an average batch size of} \;\; \mathbf{E}[B] = \frac{4 C^2 M^2(\theta,\theta')}{(1-\kappa)\log(1/\kappa)} + C M(\theta,\theta').
\end{equation}
In practice, we usually want to minimize the wall-clock time to achieve a certain estimate error, which requires tuning $\chi$ to optimally balance scalability and efficiency. We attempt to derive a theoretically optimal value of $\chi$ in Appendix~\ref{app:optimal-value} by minimizing the product of the relaxation time---a measure of the number of steps needed---and the expected wall-clock time per step. Note that this product may be loose in bounding the total wall-clock time (we leave tightening this bound to future work), making the derived $\chi$ larger than necessary. In Section \ref{sec:exp} we give a simple heuristic to tune $\chi$, which works well and is generally better than the derived value.

Theorem \ref{thm:spectral-gap} only requires the mild constraints of Assumption \ref{assump} on the target distribution, so applies in many scenarios and compares well to other exact methods. SMH further requires a Bernstein-von Mises approximation to have guarantees on its batch size and acceptance rate. PoissonMH provides convergence rate guarantees, but demands the strong assumption that the target distribution has a global upper bound on the energy. FlyMC does not have any theoretical guarantees on performance. 

\section{Towards Optimal Exact Minibatch MH}\label{sec:optimality}

In Theorem \ref{thm:spectral-gap}, we expose the trade-off between convergence rate and batch size in \methodname{}. Here, we take this analysis a step further to investigate the limits of how efficient an exact minibatch MH method can be. To tackle this problem, we derive a lower bound on the batch size for any minibatch MH method that retains exactness and fast convergence. We then show that \methodname{} is asymptotically optimal in terms of its dependence on the problem parameters $C$ and $M$. In other words, it is not possible to outperform \methodname{} in this sense with a method in the class described by Algorithm~\ref{alg:subsampledMH}.

\begin{theorem}\label{thm:optimality}
Consider any stateless exact minibatch MH algorithm described by Algorithm~\ref{alg:subsampledMH}, any state space $\Theta$ (with $|\Theta| \ge 2$), any $C > 0$, and any function $M: \Theta \times \Theta \rightarrow \R^+$.
Suppose that the algorithm guarantees that, for some constant $\kappa \in (0,1)$, for any distribution, the ratio between the spectral gap of minibatch MH $\hat \gamma$ and the spectral gap of standard MH $\gamma$ is bounded by $\hat \gamma \ge \kappa \gamma$.
Then there must exist a distribution $\pi$ over $\Theta$ and
proposal $q$ such that the batch size $B$ of that algorithm, when deciding whether to accept any transition $\theta 
\rightarrow \theta'$, is bounded from below by
\begin{align}\label{eq:lower-bound}
    \mathbf{E}[B] \ge \zeta \cdot \kappa \cdot \left(C^2 M^2(\theta,\theta') + C M(\theta,\theta') \right)
\end{align}
for some constant $\zeta > 0$ independent of algorithm and problem parameters.
\end{theorem}

To prove this theorem, we construct a random walk example over two states, then consider the smallest batch size a method requires to distinguish between two different stationary distributions~(Appendix~\ref{app:proof:optimality}). The impact of Theorem \ref{thm:optimality} is three-fold: 

First, it provides an upper bound on the performance of algorithms of Algorithm \ref{alg:subsampledMH}'s form: in each iteration, the average batch size of any exact minibatch MH method of the form of Algorithm \ref{alg:subsampledMH} must be set as in (\ref{eq:lower-bound}) in order to maintain a reasonable convergence rate. To the best of our knowledge, this is the first theorem that rigorously proves a ceiling for the possible performance of minibatch MH. 

Second, \methodname{} achieves this upper bound. In fact, Theorem \ref{thm:optimality} suggests that \methodname{} is \emph{asymptotically optimal} in terms of the problem parameters, $C$ and $M$. To see this, observe that when we ignore $\kappa$, both expressions that bound $\mathbf{E}[B]$ in (\ref{eq:TunaMHEB}) and (\ref{eq:lower-bound}) are $\bigTheta(C^2 M^2(\theta, \theta') + C M (\theta, \theta'))$.
Thus \methodname{} reaches the lower bound, achieving asymptotic optimality in terms of $C$ and $M$. (Of course, this sense of ``optimality'' does not rule out potential constant-factor improvements over \methodname{} or improvements that depend on $\kappa$.) 

Lastly, this result suggests directions for developing new exact minibatch MH algorithms: to be significantly faster than \methodname, we either need to introduce additional assumptions to the problem or to develop new stateful algorithms.

In prior work, when assuming a very concentrated posterior, some methods' batch size can scale in $\mathcal{O}(1)$ \cite{bardenet2017mcmc,bierkens2019zig,cornish2019scalable} or $\mathcal{O}(1/\sqrt{N})$ \cite{cornish2019scalable} in terms of the dataset size $N$ while maintaining efficiency. Theorem \ref{thm:optimality} is compatible with these results, further 
demonstrating this is essentially the \emph{best} dependency on $N$ an exact minibatch MH method can achieve. We show this by explicitly assuming the dependency of $C$ and $M$ on $N$, as in SMH \cite{cornish2019scalable}, yielding the following corollary (proof in Appendix~\ref{app:proof:cor1}): 
\begin{corollary}\label{col:bound}
Suppose that $C$ increases linearly with $N$ ($C = \bigTheta(N)$) and $M(\theta, \theta')$ scales in $\bigTheta(N^{-(h+1)/2})$ for some constant $h > 0$. Then the lower bound in Theorem \ref{thm:optimality} becomes $\bigTheta(N^{(1-h)/2})$. In particular, it is $\bigTheta(1)$ when $h=1$, and $\bigTheta(1/\sqrt{N})$ when $h=2$.
\end{corollary}

That is, \methodname{} matches the state-of-the-art's dependency on $N$, and this dependency is optimal. Similarly, since $C$ and $M$ are the only problem parameters in the lower bound in Theorem~\ref{thm:optimality}, we can also get the optimal dependency on the other problem parameters by explicitly assuming the relation of them with $C$ and $M$.

\section{Experiments}\label{sec:exp}
We compare \methodname{} to MH, TFMH, SMH (i.e. TFMH with MAP control variates) and FlyMC. We only include PoissonMH in the Gaussian mixture experiment, as it is not applicable in the other tasks. All of these methods are unbiased, so they have the same stationary distribution. To ensure fair wall-clock time comparisons, we coded each method in Julia; our implementations are at least as fast as, if not faster than, prior implementations. For each trial, we use Gaussian random walk proposals. We tune the proposal stepsize separately for each method to reach a target acceptance rate, and report averaged results and standard error from the mean over three runs. We set $\chi$ to be roughly the largest value that keeps $\chi C^2M^2(\theta,\theta')<1$ in most steps; we keep $\chi$ as high as possible while the average batch size is around its lower bound $CM(\theta,\theta')$. We found this strategy works well in practice. We released the code at \url{https://github.com/ruqizhang/tunamh}.
\subsection{Robust Linear Regression}\label{sec:rlr}
We first test \methodname{} on robust linear regression \cite{cornish2019scalable,maclaurin2015firefly}. We use a Student's t-distribution with degree of freedom $v=4$ and set data dimension $d=100$ (Appendix \ref{app:experiments}). We tune each method separately to a 0.25 target acceptance rate. To measure efficiency, we record effective sample size (ESS) per second---a common MCMC metric for quantifying the number of effectively independent samples a method can draw from the posterior each second \cite{brooks2011handbook}. Figure \ref{fig:linear}a shows \methodname{} is the most efficient for all dataset sizes $N$; it has the largest ESS/second. For minibatch MH methods, Figure \ref{fig:linear}b compares the average batch size. \methodname's batch size is significantly smaller than FlyMC's---about 35x with $N=10^5$. TFMH has the smallest batch size, but this is because it uses a very small step size to reach the target acceptance rate (Table \ref{tab:stepsize} in Appendix~\ref{app:experiments:rlr}). This leads to poor efficiency, which we can observe in its low ESS/second.

\textbf{MAP variants} Since TFMH and FlyMC have variants that use the \emph{maximum a posteriori} (MAP) solution to boost performance, we also test \methodname{} in this scheme. SMH uses MAP to construct control variates for TFMH to improve low acceptance rates. We consider both first- and second-order approximations (SMH-1 and SMH-2). FlyMC uses MAP to tighten the lower bound (FlyMC-MAP). For our method (\methodname-MAP) and MH (MH-MAP), we simply initialize the chain with the MAP solution. Figure \ref{fig:linear}c shows that \methodname{} performs the best even when previous methods make use of MAP. With control variates, SMH does increase the acceptance rate of TFMH, but this comes at the cost of a drastically increased batch size (Figure \ref{fig:linear}d) which we conjecture is due to the control variates scaling poorly in high dimensions ($d=100$).\footnote{Control variates worked well in the SMH paper \cite{cornish2019scalable} because all experiments had small dimension ($d=10$).} FlyMC-MAP tightens the bounds, entailing a decrease in the batch size. However, as clear in the difference in ESS/second, it is still less efficient than \methodname{} due to its strong dependence between auxiliary variables and the model parameters---an issue that previous work also documents~\cite{quiroz2019speeding}.

\begin{figure*}[t!]
    \centering
    \begin{tabular}{cc}
    	\includegraphics[width=4.5cm]{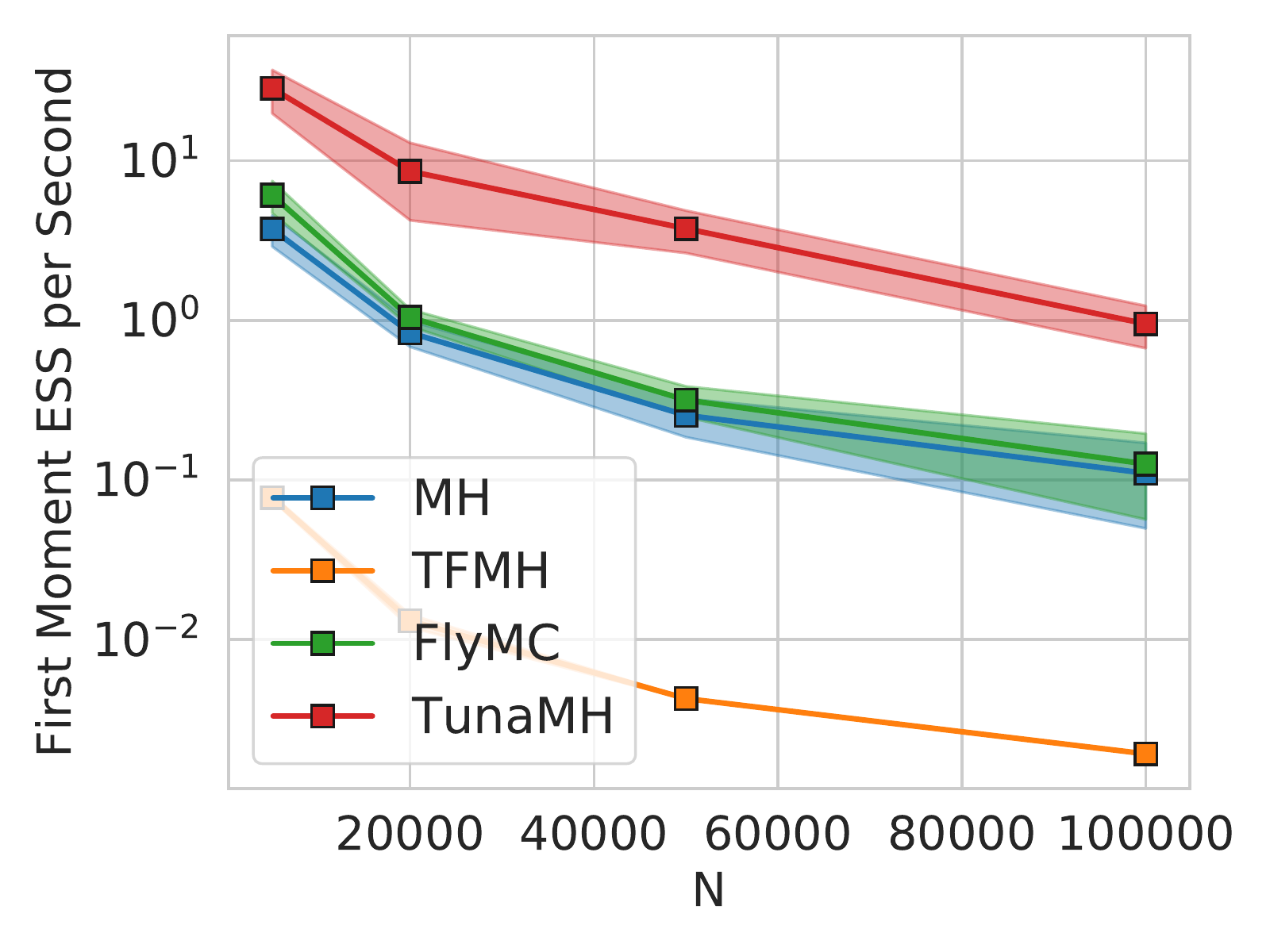} &
    	\includegraphics[width=4.5cm]{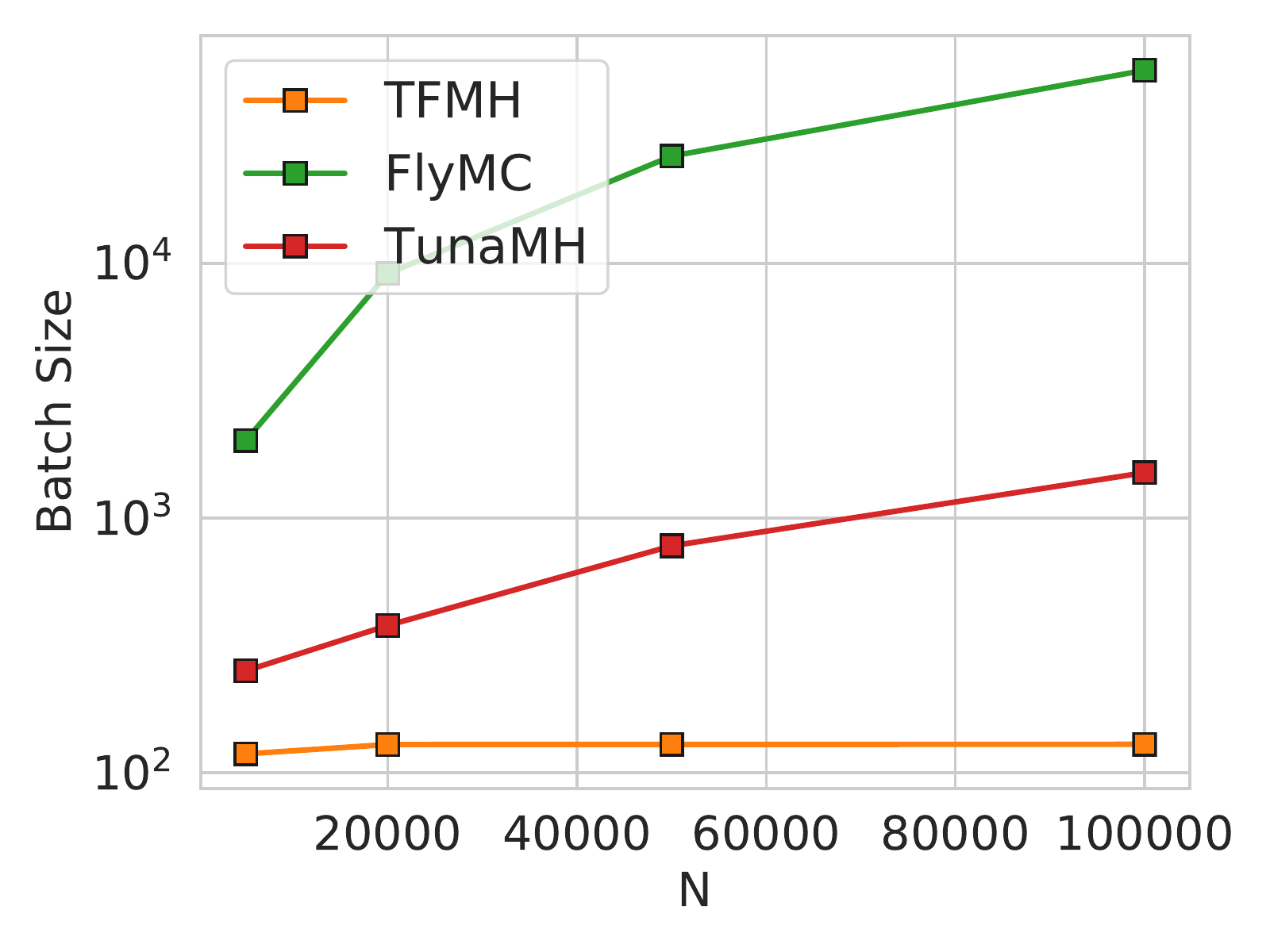}
    	\\
    	(a) &
    	(b) \\
    	\includegraphics[width=4.5cm]{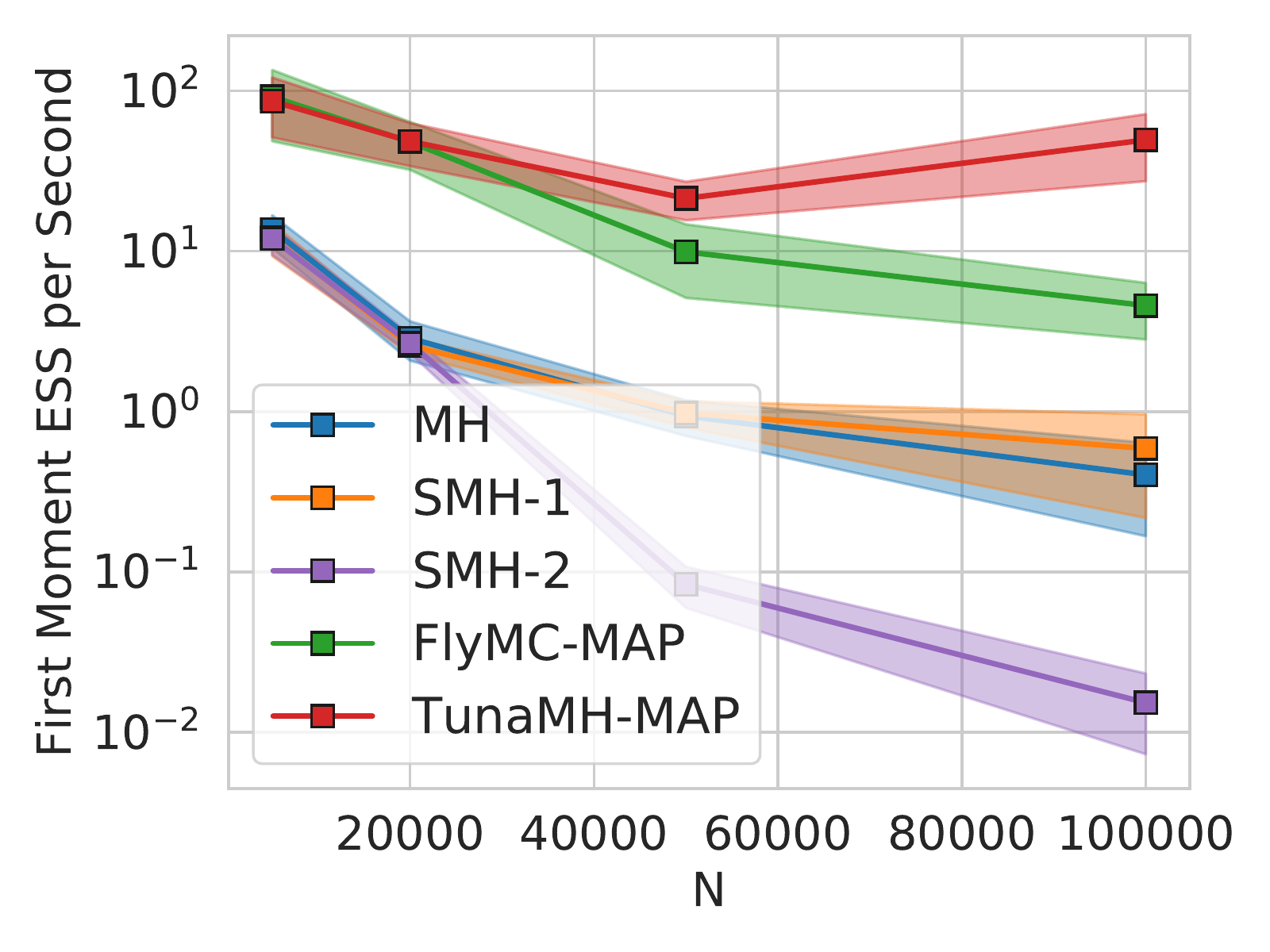} &
    	\includegraphics[width=4.5cm]{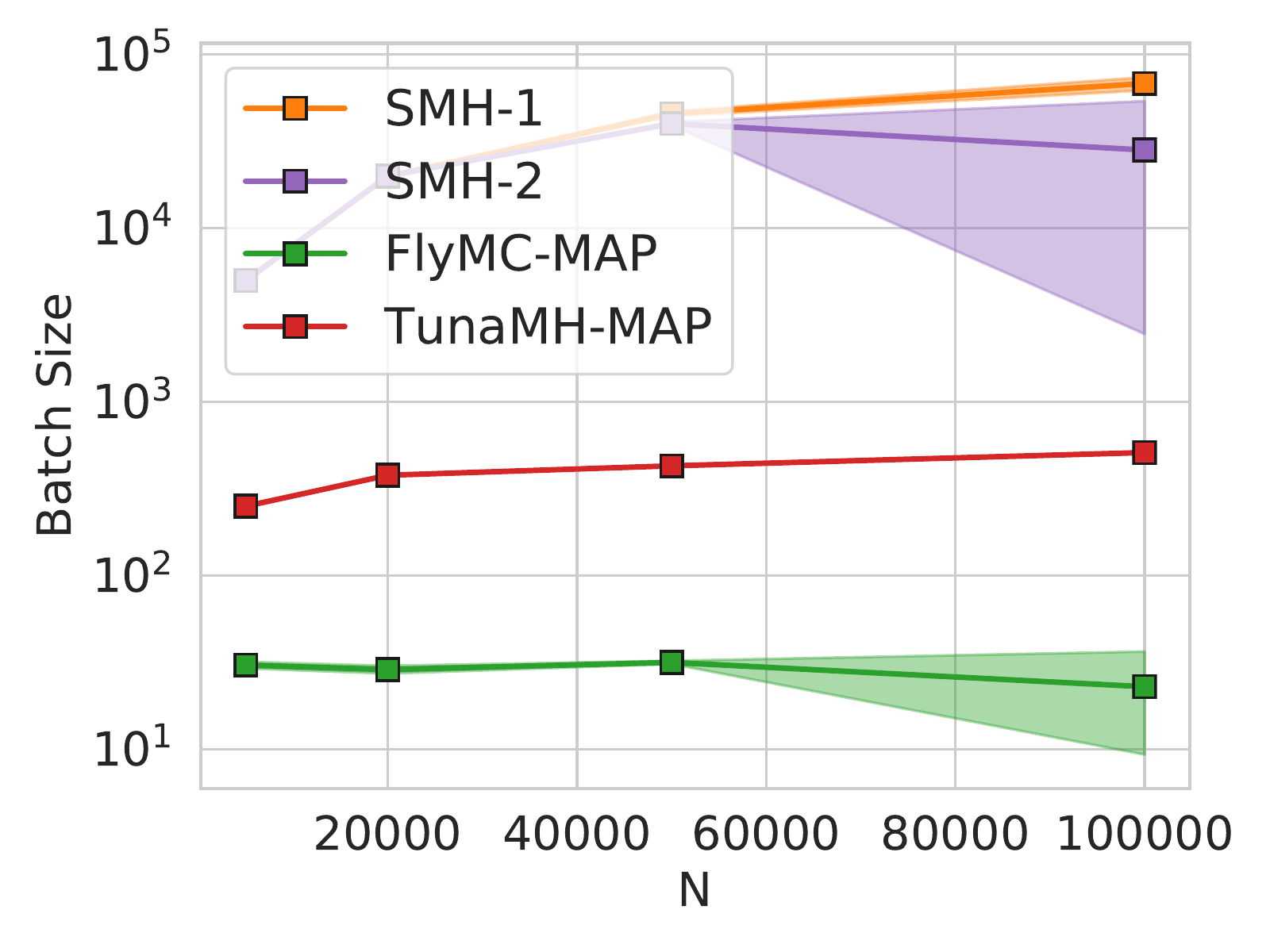}
    	\\		
    	(c) &
    	(d) \\		
    \end{tabular}
    \caption{Robust linear regression, $d=100$. (a) ESS/second without MAP. (b) Average batch size without MAP. (c)  ESS/second with MAP. (d) Average batch size with MAP.}
    \label{fig:linear}
\end{figure*}

\subsection{Truncated Gaussian Mixture}\label{sec:mog}
Next we test on a task with a multimodal posterior, a very common problem in machine learning. This demonstrates the advantage of \methodname{} not relying on MAP, because MAP is a single solution and therefore is unable to reflect all possible modes in multimodal distributions. As a result, methods that rely on MAP tuning or MAP-based control variates are unable to perform well on such problems.

We consider a Gaussian mixture. To get bounds on \methodname, TFMH, SMH, and FlyMC, we truncate the posterior, bounding $\theta_1, \theta_2\in [-3, 3]$ similar to \citet{zhang2019poisson}. We can include PoissonMH because its required bound exists after truncation. As in \citet{seita2016efficient}, we use a tempered posterior $\pi(\theta)\propto \exp \left(-\beta\sum_i U_i(\theta)\right)$ with $N = 10^6$ and $\beta=10^{-4}$. Figure \ref{fig:mog}a compares performance, showing symmetric KL versus wall-clock time. \methodname{} is the fastest, converging after 1 second, whereas the others take much longer. As expected, SMH-1 performs worse than TFMH, verifying the control variate is unhelpful for multimodal distributions. FlyMC and FlyMC-MAP are also inefficient; their performance is on par with standard MH, indicating negligible benefits from minibatching. 

\begin{figure*}[t!]
    \centering
    \begin{tabular}{ccc}		
    	\includegraphics[width=4.2cm]{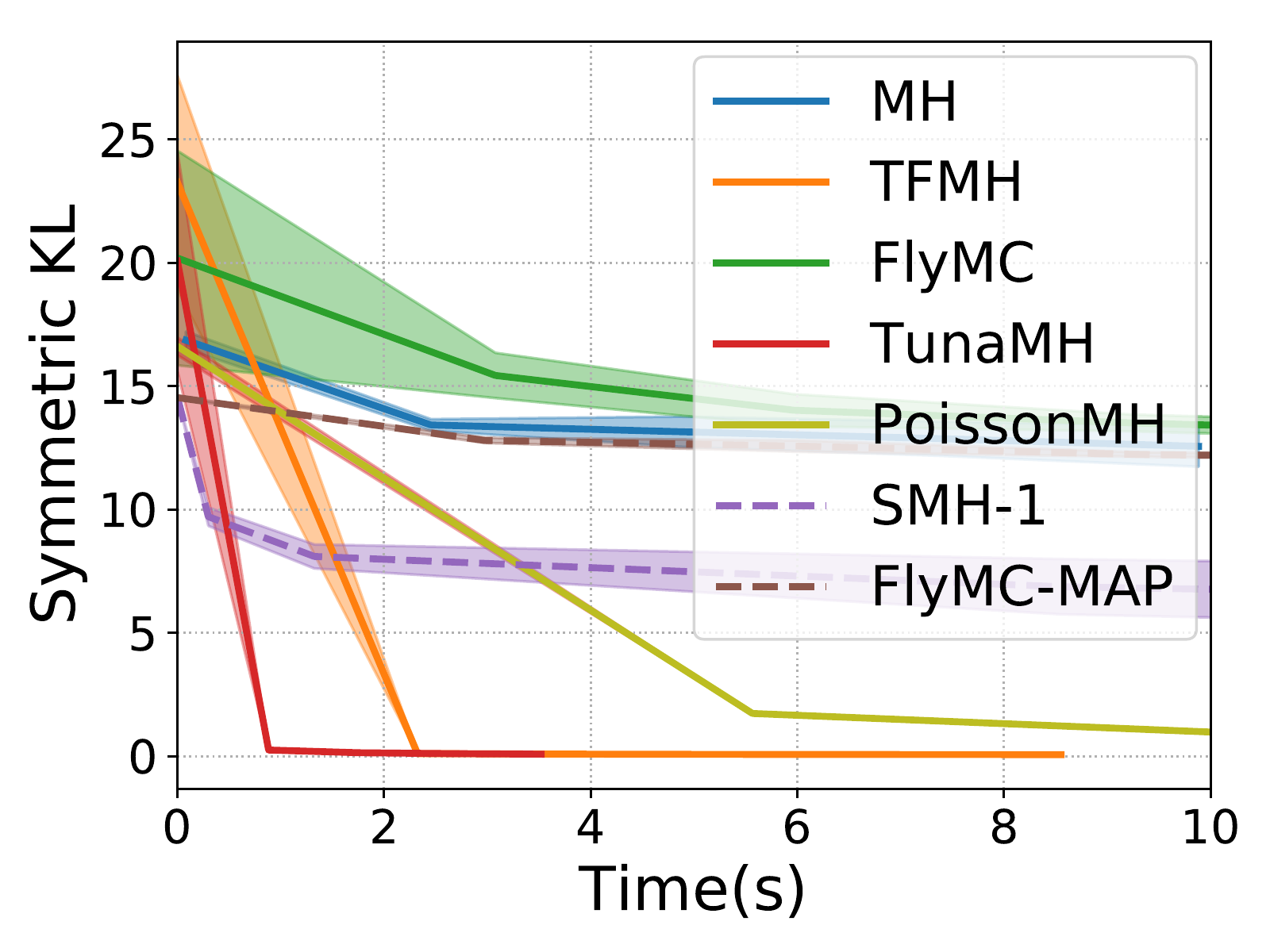}  &
    	\includegraphics[width=4.2cm]{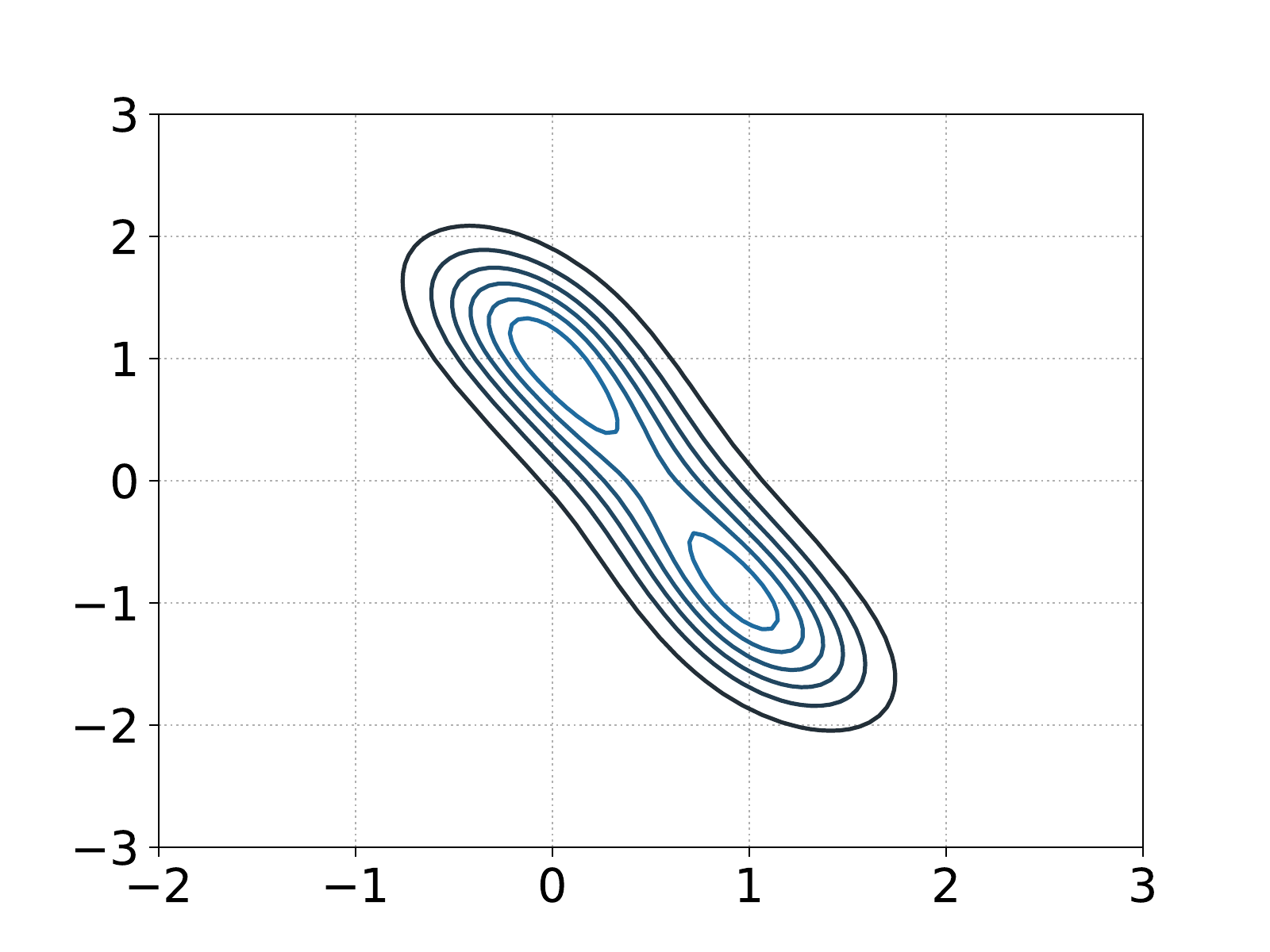} &
    	\includegraphics[width=4.2cm]{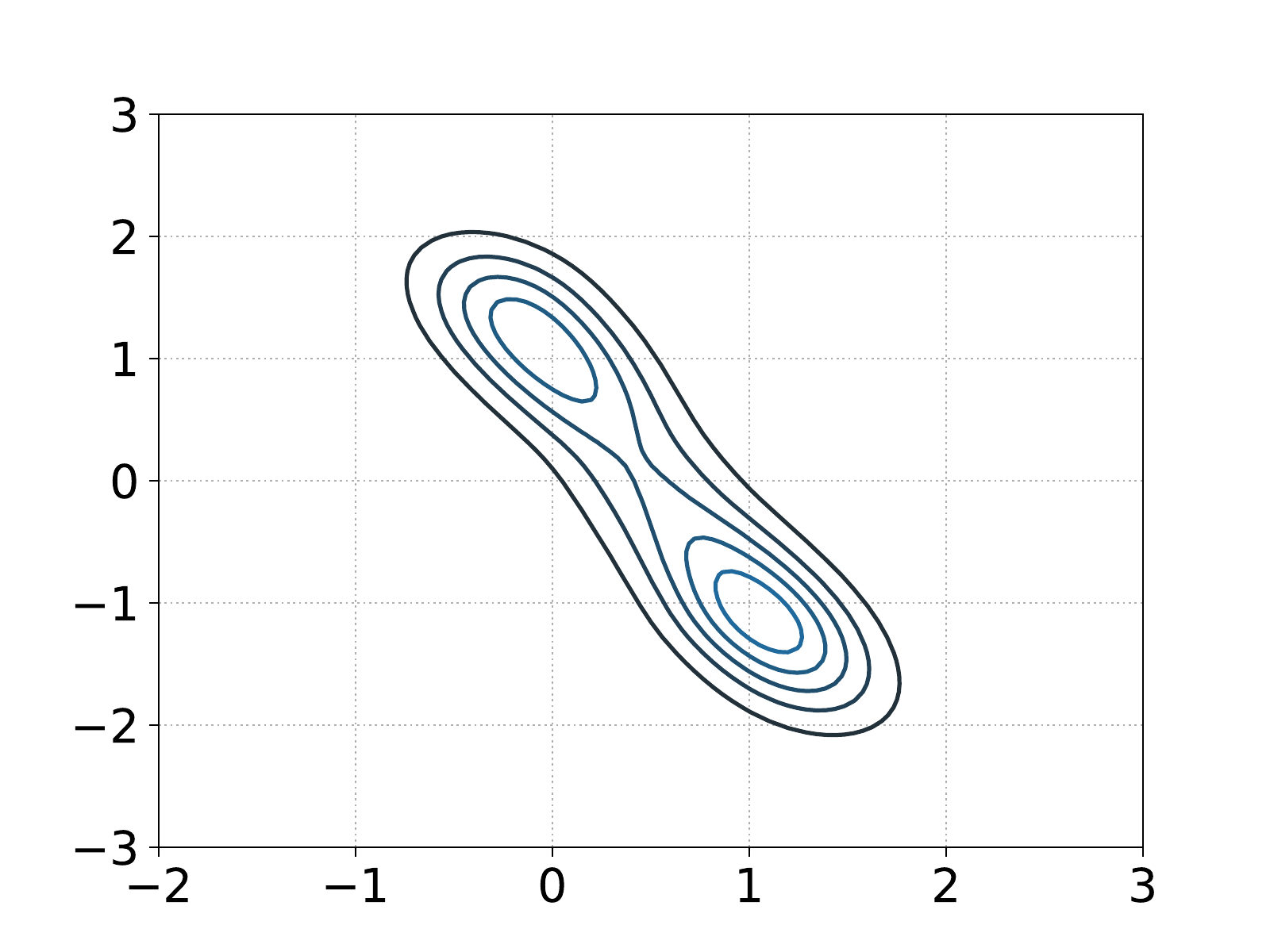}
    	\\		
    	(a) &
    	(b) &
    	(c) 
    	\hspace{-0mm}\\		
    \end{tabular}
    \caption{Truncated Gaussian mixture. (a) Symmetric KL comparison. (b) True distribution. (c) Denstity estimate of \methodname{} after 1 second.}
    \label{fig:mog}
\end{figure*}

\methodname{} also performs significantly better in terms of batch size, especially in comparison to PoissonMH (Table \ref{tab:bs}). This is due to \methodname's local bound on the energy, as opposed to PoissonMH's global bound. This also allows \methodname{} to run on more problem types, such as robust linear (Section \ref{sec:rlr}) and logistic (Section \ref{sec:lr}) regression. To illustrate the estimate quality, we also visualize the density estimate after 1 second; \methodname{}'s estimate (Figure \ref{fig:mog}c) is very close to the true distribution (Figure \ref{fig:mog}b), while the other methods do not provide on-par estimates within the same time budget (Appendix~\ref{app:experiments:mog}).

\subsection{Logistic Regression on MNIST}\label{sec:lr}

Lastly we apply \methodname{} to logistic regression on the MNIST image dataset of handwritten number digits. Mirroring the work of FlyMC \cite{maclaurin2015firefly}, we aim to classify 7s and 9s using the first 50 principal components as features. We set $\chi=10^{-5}$ following our heuristic. In Figure \ref{fig:logistic}a we see that \methodname{} is the fastest of all methods to converge, as measured by wall-clock time. We also compare average batch size in Table \ref{tab:bs}. \methodname's average batch size is 4x smaller than FlyMC's. TFMH again has the smallest batch size, but sacrifices efficiency by using a small step size in order to achieve the target acceptance rate. Thus, overall, TFMH is again inefficient in these experiments.

\addtolength{\tabcolsep}{-1pt} 
\begin{table}[h]
  \caption{Avg. batch size $\pm$ SE from the mean on 3 runs. PoissonMH not applicable to logistic reg.}
  \label{tab:bs}
  \centering
  \begin{tabular}{lcccccc}
    \toprule
    Tasks   & TFMH & FlyMC & PoissonMH & \methodname{}\\
    \midrule
    Gaussian Mixture   & $13.91\pm0.016$  & $811.52\pm234.16$ & $3969.67\pm327.26$ & $86.45\pm0.04$ \\
    Logistic Regression & $39.28\pm0.12$ & $1960.19\pm150.96$ & --- & $504.07\pm0.33$ \\
    \bottomrule
  \end{tabular}
\end{table}
\addtolength{\tabcolsep}{1pt} 

\textbf{Effect of Hyperparameter $\chi$}
To understand the effect of $\chi$ in \methodname{}, we report results with varying $\chi$. Figure \ref{fig:logistic}b plots test accuracy as a function of the number of iterations. As $\chi$ increases, \methodname's convergence rate approaches standard MH. This verifies our theoretical work: $\chi$ acts like a dial to control convergence rate and batch size trade-off---mapping to the efficiency-scalability trade-off. Figure \ref{fig:logistic}c shows \methodname's wall-clock time performance is not sensitive to $\chi$, as the performance is superior to standard MH regardless of how we set it. However, $\chi$ needs to be tuned in order to achieve the best performance. Previous methods do not have such a dial, so they are unable to control this trade-off to improve the sampling efficiency.

\begin{figure*}[t!]
    \centering
    \begin{tabular}{cccc}		
    	\includegraphics[width=4.2cm]{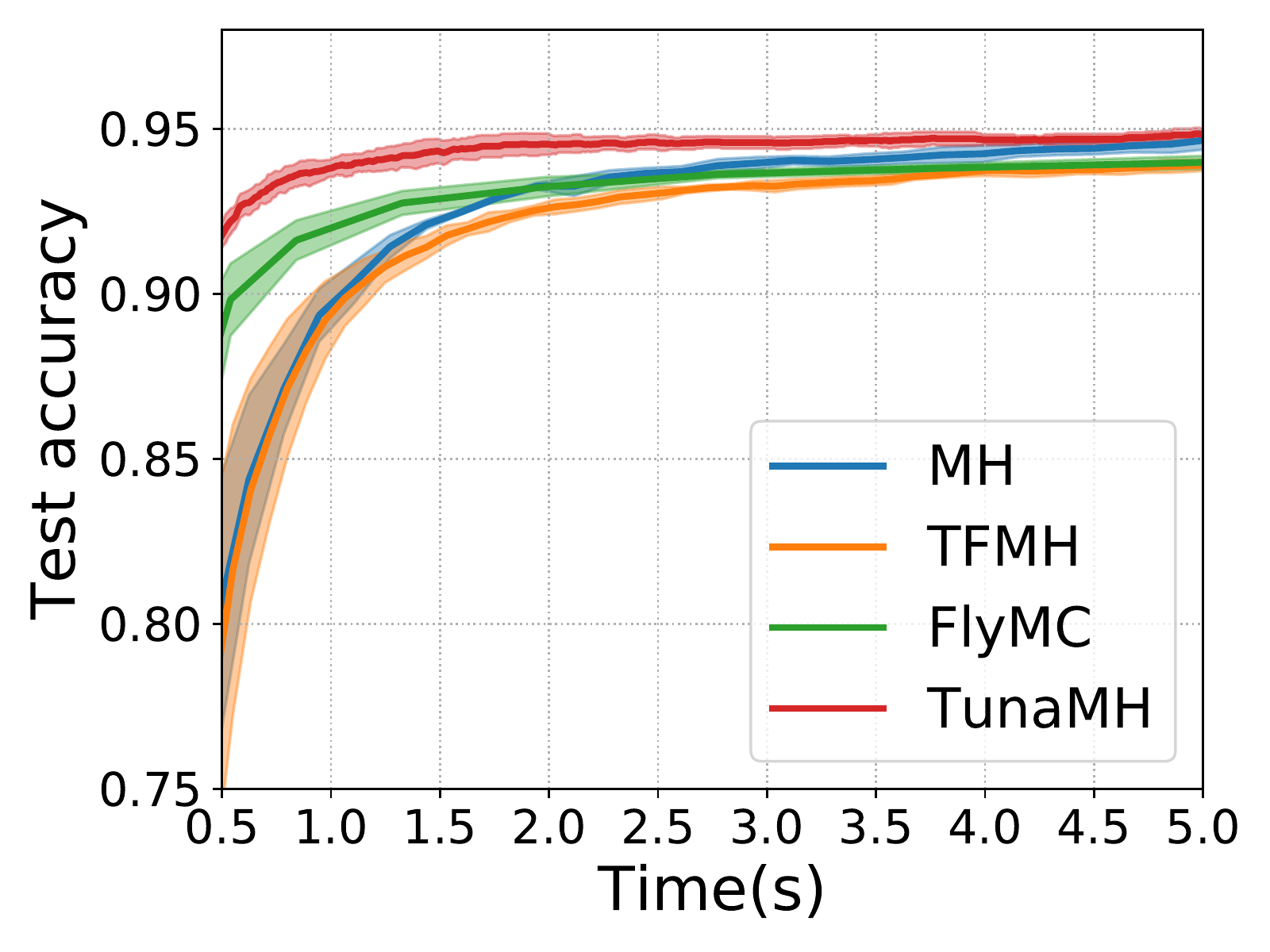}  &
    	\includegraphics[width=4.2cm]{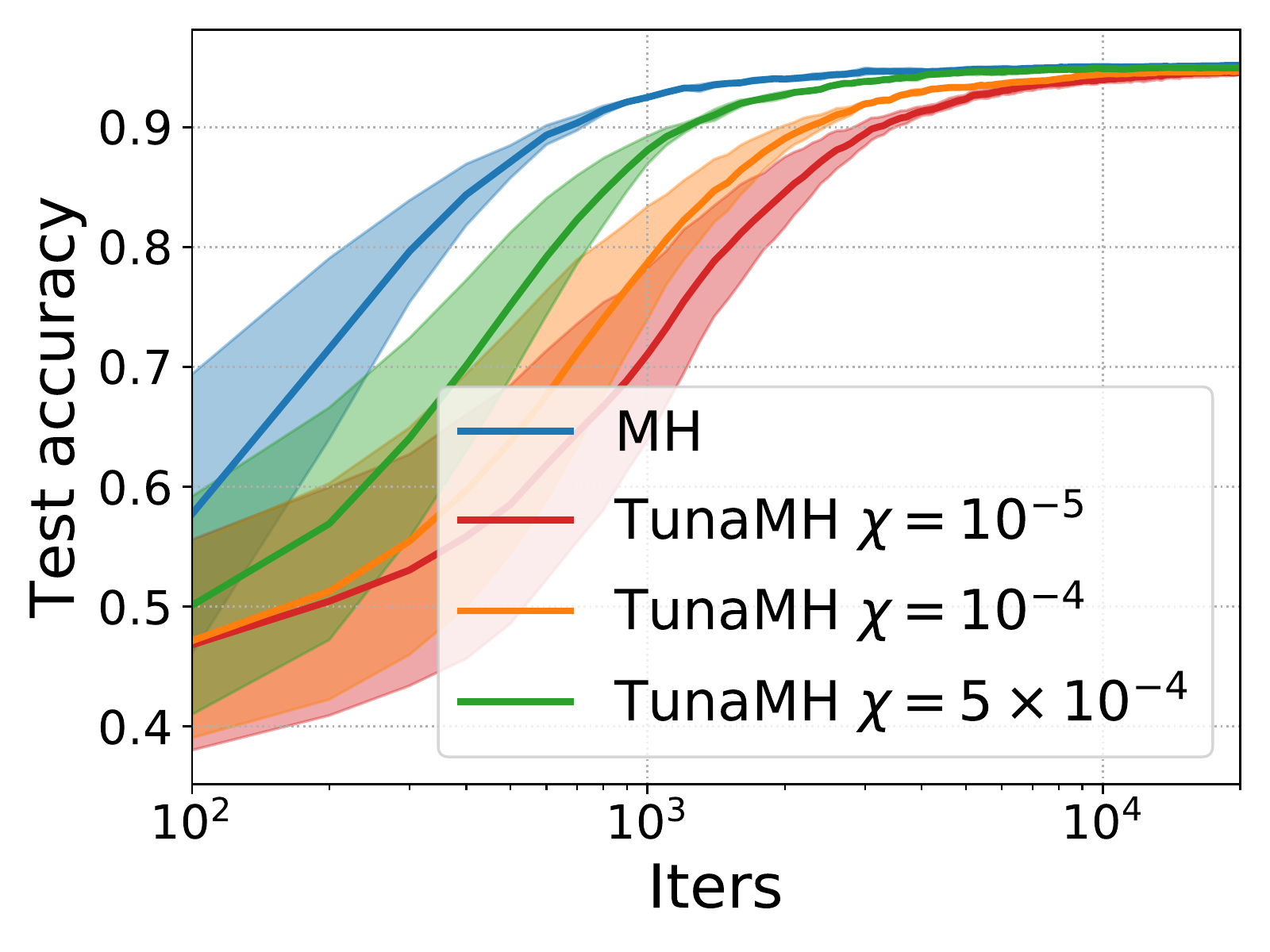} &
    	\includegraphics[width=4.2cm]{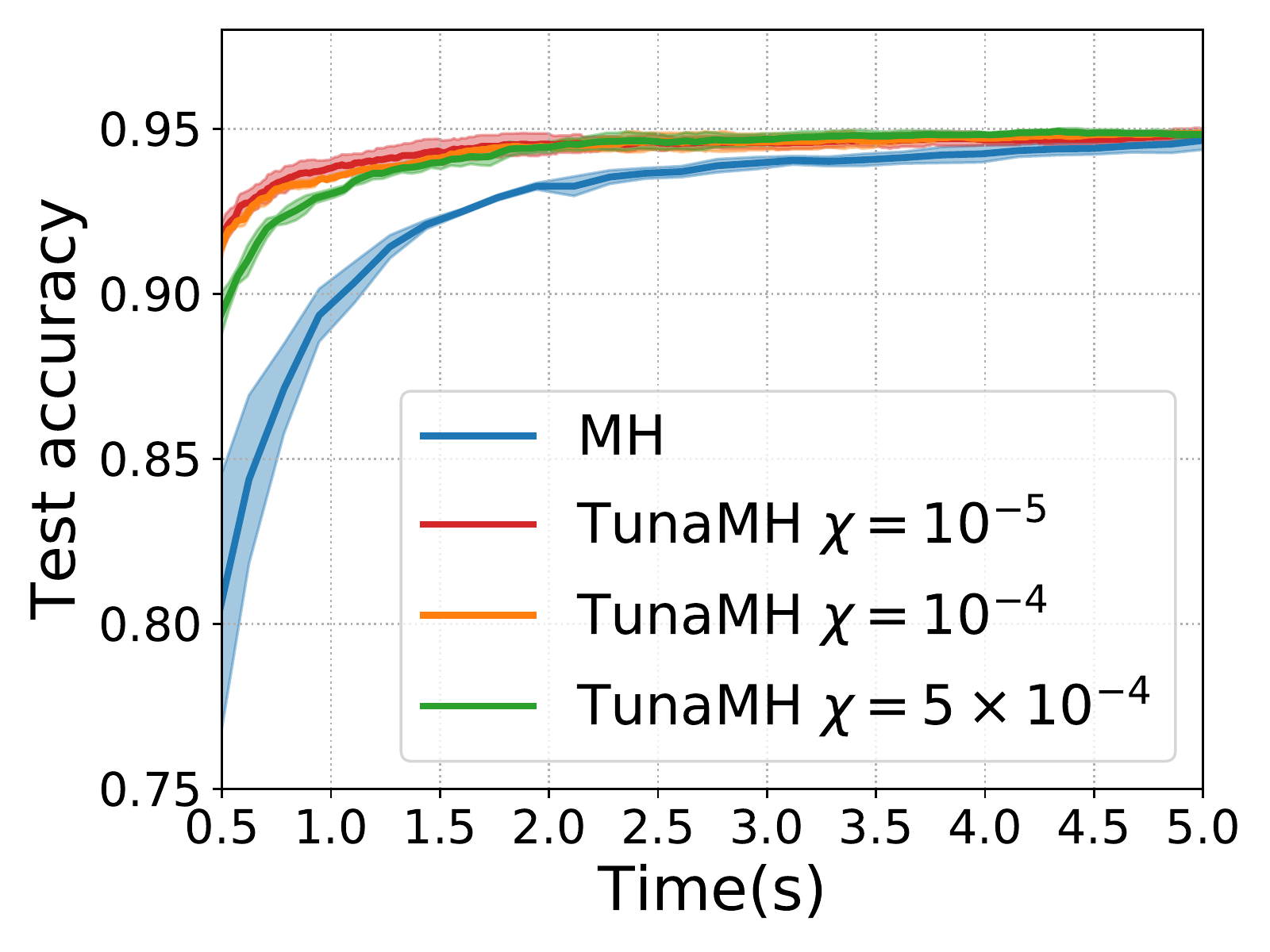}
    	\\		
    	(a)  &
    	(b) &
    	(c) 
    	\hspace{-0mm}\\		
    \end{tabular}
    \caption{MNIST logistic regression. (a) Test accuracy comparison. (b)-(c) \methodname{}'s test accuracy for various $\chi$. Batch size for $\chi=10^{-5}, 10^{-4}, 5\times 10^{-4}$ is 504.07, 810.35 and 2047.91 respectively.}
    \label{fig:logistic}
\end{figure*}

\section{Conclusion and Future Work}
After demonstrating that inexact methods can lead to arbitrarily incorrect inference, we focus our work in this paper on exact minibatch MH methods. We propose a new exact method, \methodname{}, which lets users trade off between batch size and guaranteed convergence rate---between scalability and efficiency. We prove a lower bound on the batch size that any minibatch MH method must use to maintain exactness and convergence rate, and show \methodname{} is asymptotically optimal. Our experiments validate these results, demonstrating that \methodname{} outperforms state-of-the-art exact methods, particularly on high-dimensional and multimodal distributions.

To guide our analysis, we formalized a class of stateless, energy-difference-based minibatch MH methods, to which most prior methods belong. While \methodname{} is asymptotically optimal for this class, future work could develop new exact methods that are better by a constant factor or on some restricted class of distributions. It would also be interesting to develop effective theoretical tools for analyzing stateful methods, since these methods could potentially bypass our lower bound.

\section*{Broader Impact}

Our work shines a light on how to scale MCMC methods responsibly. We make the case that inexact minibatch MH methods can lead to egregious errors in inference, which suggests that---particularly for high-impact applications \cite{gelman2007racialbias, pierson2018fasttt}---we should avoid their use. We provide an alternative: a minibatch MH method that guarantees correctness, while also maintaining an optimal balance between efficiency and scalability, enabling its safe use on large-scale applications. 

\section*{Acknowledgements}

This work was supported by a gift from SambaNova Systems, Inc. and funding from Adrian Sampson. We thank Jerry Chee, Yingzhen Li, and Wing Wong for helpful feedback on the manuscript. 

\bibliography{references}
\bibliographystyle{plainnat}

\appendix
\newpage



\section{Proof of Theorem \ref{statement:counterexample}}\label{app:proof:counterexample}

In this section, we prove Theorem \ref{statement:counterexample}, which asserts that any inexact stateless MH algorithm can produce arbitrarily large bias between its target distribution (the distribution we are trying to sample from) and its stationary distribution (the distribution that the chain actually produces samples from asymptotically). 
\begin{proof}
Let $\mathcal{A}$ denote the \texttt{SubsMH} in Algorithm~\ref{alg:subsampledMH} of the minibatch MH method in question.
Since $\mathcal{A}$ is inexact, there must exist a state space $\Theta$, proposal distribution $q$, and target distribution $\mu$, satisfying Assumption~\ref{assump} with parameters $c_1, \ldots, c_N, C, M$, where
\[
    \mu(\theta) \propto \exp\left( -\sum_{i=1}^N V_i(\theta) \right)
\]
for some $N$ and energy functions $V_1, \ldots, V_N$, such that $\mathcal{A}$ run on $\mu$ with proposal distribution $q$ does not have stationary distribution $\mu$.

Next, let $a_{\mu}(\theta, \theta')$ denote the acceptance probability of algorithm $\mathcal{A}$ on the above task for a proposed transition from $\theta$ to $\theta'$.
Assume by way of contradiction that on this problem, it is always true that
\[
    \frac{a_{\mu}(\theta, \theta')}{a_{\mu}(\theta', \theta)}
    =
    \frac{
        \mu(\theta') q(\theta|\theta')
    }{
        \mu(\theta) q(\theta'|\theta)
    }.
\]
If this were true, then the overall transition probability of this chain, for $\theta \ne \theta'$, would be
\[
    T_{\mu}(\theta, \theta') = q(\theta'|\theta) \cdot a_{\mu}(\theta, \theta')
\]
and it would hold that
\[
    \mu(\theta) T_{\mu}(\theta, \theta') = \mu(\theta') T_{\mu}(\theta', \theta).
\]
That is, the chain would be reversible, also known as satisfying detailed balance.
But it is a standard result that for any reversible chain, $\mu$ must be a stationary distribution of that chain.
We have now derived a contradiction, which establishes that our assumption is false.
That is, there exists a $\theta, \theta' \in \Theta$ such that
\[
    \frac{a_{\mu}(\theta, \theta')}{a_{\mu}(\theta', \theta)}
    \ne
    \frac{
        \mu(\theta') \cdot q(\theta|\theta')
    }{
        \mu(\theta) \cdot q(\theta'|\theta)
    }.
\]
Explicitly, this means that if we define the function $\Delta V$ such that
\[
    \Delta V(i) = V_i(\theta) - V_i(\theta'),
\]
then for this subsampling problem,
\begin{equation}
    \label{eqnStmt1Proof1}
    \frac{
        \mathbf{E}\left[\mathcal{A}(\Delta V, N, q(\theta|\theta')/q(\theta'|\theta), c_1, \ldots, c_N, C, M(\theta,\theta')) \right]
    }{
        \mathbf{E}\left[\mathcal{A}(-\Delta V, N, q(\theta'|\theta)/q(\theta|\theta'), c_1, \ldots, c_N, C, M(\theta,\theta')) \right]
    }
    \ne
    \frac{
        \mu(\theta') \cdot q(\theta|\theta')
    }{
        \mu(\theta) \cdot q(\theta'|\theta)
    }.
\end{equation}
Without loss of generality, assume that
\[
    q(\theta|\theta')/q(\theta'|\theta) \le 1.
\]
(This is without loss of generality since we can ensure it is the case by swapping $\theta$ and $\theta'$.) We fixed $\theta$ and $\theta'$ to be the pair satisfying Equation~\ref{eqnStmt1Proof1} throughout this section.

\paragraph{Constructing an example.} We use this to prove the theorem by a constructive example. Let $x_1, \ldots, x_N$ be defined by
\[
    x_i = \Delta V(i) = V_i(\theta) - V_i(\theta').
\]
Define $X$ as the sum
\[
    X = \sum_{i=1}^N x_i.
\]
For some parameter $K \in \mathbb{N}$ (to be defined later), consider the state space $\Omega$ defined as
\[
    \Omega = \{(k,z) \mid k \in \{0, \ldots, K-1\}, \; 0 \le z \le \exp(k X) \},
\]
using the natural measure for a finite disjoint union of measure spaces.
Define a target distribution over $\Omega$ given by the density
\[
  \pi(k,z) \propto \exp\left(-\sum_{i=1}^N k \cdot x_i\right),
\]
or equivalently
\[
    \pi(k,z) \propto \exp\left(-\sum_{i=1}^N U_i(k,z) \right) \; \text{where} \; U_i(k,z) = k x_i.
\]
Define a proposal distribution $\hat q$, such that, starting from $(k,z)$:
\begin{itemize}
    \item With probability $1/4$, we sample $z'$ uniformly from $[0, \exp(k X)]$ and propose a transition to $(k,z')$.
    \item With probability $1/4$, we propose a transition to $(k-1,z)$, if it is in $\Omega$.
    \item With probability $\frac{1}{4} \cdot \frac{q(\theta|\theta')}{q(\theta'|\theta)}$, we propose a transition to $(k+1,z)$, if it is in $\Omega$.
    \item With the remaining probability, we just propose to stay at $(k,z)$.
\end{itemize}
This is effectively acting as a random walk over $k$, and our goal will be to show that while the true target distribution $\pi$ has a marginal in $k$ that is the uniform distribution, the minibatch MH method causes the chain's transition to be biased to step more in one direction than another, resulting in a highly biased stationary distribution (where we can make the bias arbitrarily large by setting $K$).

We use the same $c_i$ and $C$ as before, and define a new function $\hat M$ such that
\[
    \hat M((k,z),(k+1,z)) = \hat M((k,z),(k-1,z)) = M(\theta, \theta')
\]
and $\hat M(\cdots) = 0$ for other proposed transitions (we can set $\hat M$ however we want for pairs of states that are never proposed in a transition, since this will not affect the algorithm).
Clearly, this setup satisfies Assumption~\ref{assump}, since the original distribution did.

Now, consider what our minibatch MH method will do when run on this task. There are three cases to consider.

\paragraph{Proposed changes in $z$.} When a proposed change in $z$ is made, the resulting $\Delta U$ will be uniformly $0$, and the probability of the reverse transition will be equal (1/4 in both directions), so the algorithm will be passed the arguments
\[
    \mathcal{A}(0, N, 1, c_1, \ldots, c_N, C, 0).
\]
Since this does not depend at all on $z$ or $k$, this means that the acceptance probability of these transitions will be the same regardless of the state. Call this probability $\alpha_0$.

\paragraph{A proposal to decrease $k$.} When a proposal is made to decrease $k$, the probability of the forward and reverse transitions will be
\[
    \hat q((k-1,z)|(k,z)) = \frac{1}{4}
    \;\text{and}\;
    \hat q((k,z)|(k-1,z)) = \frac{1}{4} \cdot \frac{q(\theta|\theta')}{q(\theta'|\theta)}.
\]
It follows that
\[
    \frac{\hat q((k,z)|(k-1,z))}{\hat q((k-1,z)|(k,z))} = \frac{q(\theta|\theta')}{q(\theta'|\theta)}.
\]
The energy function difference for this proposal will be
\[
    \Delta U(i) = U_i((k,z)) - U_i((k-1,z)) = k x_i - (k - 1) x_i = x_i,
\]
so in particular $\Delta U = \Delta V$. And, of course for this transition $\hat M$ will take on the value $M(\theta, \theta')$.
So, the minibatch MH algorithm will be passed the arguments
\[
    \mathcal{A}(\Delta V, N, q(\theta|\theta')/q(\theta'|\theta), c_1, \ldots, c_N, C, M(\theta, \theta')),
\]
and so it will accept with probability
\[
    \mathbf{E}\left[ \mathcal{A}(\Delta V, N, q(\theta|\theta')/q(\theta'|\theta), c_1, \ldots, c_N, C, M(\theta, \theta')) \right].
\]
Call this probability $\alpha_-$.

\paragraph{A proposal to increase $k$.} When a proposal is made to increase $k$, the probability of the forward and reverse transitions will be
\[
    \hat q((k+1,z)|(k,z)) = \frac{1}{4} \cdot \frac{q(\theta|\theta')}{q(\theta'|\theta)}.
    \;\text{and}\;
    \hat q((k,z)|(k+1,z)) = \frac{1}{4}.
\]
It follows that
\[
    \frac{\hat q((k,z)|(k+1,z))}{\hat q((k+1,z)|(k,z))} = \frac{q(\theta'|\theta)}{q(\theta|\theta')}.
\]
The energy function difference for this proposal will be
\[
    \Delta U(i) = U_i((k,z)) - U_i((k+1,z)) = k x_i - (k + 1) x_i = -x_i,
\]
so in particular $\Delta U = -\Delta V$. And, as before for this transition $\hat M$ will take on the value $M(\theta, \theta')$.
So, the minibatch MH algorithm will be passed the arguments
\[
    \mathcal{A}(-\Delta V, N, q(\theta'|\theta)/q(\theta|\theta'), c_1, \ldots, c_N, C, M(\theta, \theta')),
\]
and so it will accept with probability
\[
    \mathbf{E}\left[ \mathcal{A}(-\Delta V, N, q(\theta'|\theta)/q(\theta|\theta'), c_1, \ldots, c_N, C, M(\theta, \theta')) \right].
\]
Define the probability $\alpha_+$ as
\[
    \alpha_+ = \mathbf{E}\left[ \mathcal{A}(-\Delta V, N, q(\theta'|\theta)/q(\theta|\theta'), c_1, \ldots, c_N, C, M(\theta, \theta')) \right] \cdot \frac{q(\theta|\theta')}{q(\theta'|\theta)}.
\]

\paragraph{The resulting Markov chain.}
From the above analysis, we can conclude that the Markov chain that results from subsampling algorithm $\mathcal{A}$ applied to this method is as follows.
Starting from $(k,z)$, if we let $\hat T$ denote the transition operator of this Markov chain,
\begin{itemize}
    \item With probability $\frac{1}{4} \cdot \alpha_0$, we sample $z'$ uniformly from $[0, \exp(k X)]$ and transition to $(k,z')$.
    \item With probability $\frac{1}{4} \cdot \alpha_-$, we transition to $(k-1,z)$, if it is in $\Omega$.
    \item With probability $\frac{1}{4} \cdot \alpha_+$, we transition to $(k+1,z)$, if it is in $\Omega$.
    \item With the remaining probability, we just stay at $(k,z)$.
\end{itemize}
Consider the distribution
\[
    \nu(k,z) \propto \left( \frac{\alpha_+}{\alpha_-} \right)^k.
\]
It is easy to see that this Markov chain satisfies detailed balance with $\nu$ as its stationary distribution. In particular,
\begin{align*}
        \nu(k,z) \cdot T((k-1,z)|(k,z)) &= \left( \frac{\alpha_+}{\alpha_-} \right)^k \cdot \frac{1}{4} \cdot \alpha_-\\
    &=
    \left( \frac{\alpha_+}{\alpha_-} \right)^{k-1} \cdot \frac{1}{4} \cdot \alpha_+ \\
    &= \nu(k-1,z) \cdot T((k,z)|(k-1,z)). 
\end{align*}

So $\nu$ will be a stationary distribution of the minibatch MH chain $\hat T$.

Observe that the marginal distribution of $k$ in $\pi$ is
\[
    \pi(k) = \int_{0}^{\exp(kX)} \pi(k,z) \; dz \propto \exp\left(-\sum_{i=1}^N k \cdot x_i\right) \cdot \exp(kX) = 1,
\]
so the marginal distribution of $k$ in the target distribution is actually the uniform distribution.
On the other hand, using the same derivation, the marginal distribution of $k$ in $\nu$ is
\[
    \nu(k) \propto \left( \frac{\alpha_+}{\alpha_-} \right)^k \cdot \exp(kX) = \left( \frac{\alpha_+}{\alpha_-} \cdot \exp(X) \right)^k.
\]

We know immediately by substituting our definitions of $\alpha_+$ and $\alpha_-$ into (\ref{eqnStmt1Proof1}) that
\[
    \frac{
        \alpha_-
    }{
        \alpha_+
    }
    \ne
    \frac{
        \mu(\theta')
    }{
        \mu(\theta)
    }
    =
    \exp\left( \sum_{i=1}^N (V_i(\theta) - V_i(\theta') \right)
    =
    \exp\left( \sum_{i=1}^N x_i \right)
    =
    \exp(X).
\]
As a consequence, we know that
\[
    \frac{\alpha_+}{\alpha_-} \cdot \exp(X) \ne 1.
\]
Call this constant
\[
    A = \frac{\alpha_+}{\alpha_-} \cdot \exp(X),
\]
and observe that $A \ne 1$ and that $A$ is independent of our choice of $K$ (which still remains unset).
This gives
\[
    \nu(k) \propto A^k.
\]
Explicitly, this distribution will be
\[
    \nu(k) = \frac{1}{\sum_{k=0}^{K-1} A^{k}}\cdot A^{k} = \frac{1-A}{1 - A^K} \cdot A^k.
\]
Since the total variation distance between two probability measures is lower bounded by the TV-distance between their marginal distributions in any one variable, and similarly the KL divergence is \emph{also} lower bounded by the KL divergence between its marginal distributions in any one variable (both these facts follow directly from the monotonicity property of the $f$-divergence, of which the KL-divergence and TV-distance are both instances), to prove this theorem it suffices to show both TV-distance and KL-divergence bounds on the marginal distributions in $k$. We do this now.

\paragraph{Bounding the total variation distance.}

Now, we compute the total variation distance between $\pi$ and $\nu$. For this bit of the proof, we will just consider the marginal distribution in $k$, as this provides a lower bound on the TV distance between the joint distribution. For simplicity, for the rest of the proof, we let $\tilde \pi$ denote this marginal distribution of $k$ in $\nu$, and also let $\pi$ denote the marginal distribution of $k$ in $\pi$.
By the definition of total variation distance,
\begin{align*}
\text{TV}(\pi, \tilde \pi) 
&= \frac{1}{2}\sum_{k=0}^{K-1}\left| \tilde \pi(k) - \pi(k) \right| \\
&= \frac{1}{2}\sum_{k=0}^{K-1}\left|\frac{1-A}{1 - A^{K}}\cdot A^{k} - \frac{1}{K}\right|.
\end{align*}
If $A<1$,
\begin{align}\label{eq:tv}
\text{TV}(\pi, \tilde{\pi})
& =
\sum_{k=0}^{K_0}\left(\frac{1-A}{1 - A^{K}}\cdot A^{k} - \frac{1}{K}\right)\nonumber\\
& =
\frac{1-A^{K_0}}{1 - A^{K}} - \frac{K_0}{K}
\end{align}
where $K_0$ is the largest $k$ such that
\[
\frac{1-A}{1 - A^{K}}\cdot A^{k} > \frac{1}{K}.
\]
By solving the above equation, we have
\[
K_0 = \left\lfloor \frac{\log(1-A^{K}) - \log(1-A) - \log(K)}{\log(A)}\right\rfloor.
\]

We can lower bound $K_0$ by
\begin{align*}
    K_0&\ge \frac{\log(1-A^{K}) - \log(1-A) - \log(K)}{\log(A)} - 1\\
    &\ge \frac{- \log(1-A) - \log(K)}{\log(A)} - 1.
\end{align*}

It follows that the first term in (\ref{eq:tv}) becomes
\begin{align*}
    \frac{1-A^{K_0}}{1 - A^{K}}\ge
    \frac{1-\frac{1}{KA(1-A)}}{1 - A^{K}}
    \ge 1-\frac{1}{KA(1-A)}.
\end{align*}
We can also upper bound $K_0$ and then the second term can be bounded as the following
\begin{align*}
\frac{K_0}{K}
& \le
\frac{\log(1-A^{K}) - \log(K)}{K\log(A)}.
\end{align*}
When $K\ge \frac{\log\left(1 - \exp\left(-\frac{1}{2}\right)\right)}{\log(A)}$, we have $\log(1-A^{K})\ge -\frac{1}{2}$. Since $\log(K)\le K^{\frac{1}{2}}$ and $K^{-1} \le K^{-\frac{1}{2}}$, we have
\begin{align*}
\frac{K_0}{K}
& \le
\frac{-\frac{1}{2}K^{-1} - K^{-\frac{1}{2}}}{\log(A)}
\le
-\left(\frac{3}{2\log(A)}\right) K^{-\frac{1}{2}}.
\end{align*}
Therefore, the TV distance is bounded by
\begin{align*}
\text{TV}(\pi, \tilde{\pi})
& \ge
1-\frac{1}{KA(1-A)} + \left(\frac{3}{2\log(A)}\right) K^{-\frac{1}{2}}\\
&\ge 
1 + \left(\frac{3}{2\log(A)} -\frac{1}{A(1-A)}\right) K^{-\frac{1}{2}}.
\end{align*}

To make $\text{TV}(\pi, \tilde{\pi})\ge \delta$, we just need to set
\[
K \ge \frac{\left(\frac{3}{2\log(A)} -\frac{1}{A(1-A)}\right)^2}{(1 - \delta)^2}.
\]
Similarly, if $A>1$, 
\begin{align*}
\text{TV}(\pi, \tilde{\pi})
& =
\sum_{k=K_0}^{K-1}\left(\frac{1-A}{1 - A^{K}}\cdot A^{k} - \frac{1}{K}\right)\\
& =
\frac{A^{K} - A^{K_0}}{A^{K} - 1} - \frac{K - K_0}{K}\\
& = 
\frac{K_0}{K} - \frac{A^{K_0}-1}{A^{K} - 1}
\end{align*}
where
\[
K_0 = \left\lceil \frac{\log(A^{K} - 1) - \log(A - 1) - \log(K)}{\log(A)} \right\rceil
\]
which is the smallest $k$ such that
\[
\frac{1-A}{1 - A^{K}}\cdot A^{k} > \frac{1}{K}.
\]
We can get an upper bound of $K_0$ by
\begin{align*}
K_0 &\le \frac{\log(A^{K} - 1) - \log(A - 1) - \log(K)}{\log(A)} + 1\\
& = 
\log_A \left(\frac{A^{K} - 1}{K(A-1)}\right) + 1.
\end{align*}
Therefore,
\begin{align*}
\frac{A^{K_0}-1}{A^{K} - 1} 
&\le 
\frac{A\cdot\left(\frac{A^{K} - 1}{K(A-1)}\right) - 1}{A^{K}-1}\\
& = 
\frac{A}{K(A-1)} - \frac{1}{A^{K} - 1}.
\end{align*}

We can lower bound $K_0$ by
\begin{align*}
K_0
& \ge 
\log_A \left(A^{K}-1\right) - \log_A(A-1) - \log_A(K).
\end{align*}
When $K\ge 1 - \log_A (A-1)$, $A^{K}-1\ge A^{K-1}$. Then we have
\begin{align*}
K_0
& \ge 
\log_A \left(A^{K-1}\right) - \log_A(A-1) - \log_A(K)\\
& = K - 1 - \log_A(A-1) - \log_A(K).
\end{align*}
It follows that
\begin{align*}
\frac{K_0}{K}
& \ge 
1 - \frac{1}{K} - \frac{\log_A(A-1)}{K} - \frac{\log_A(K)}{K}.
\end{align*}
Since $\log(K)\le K^{\frac{1}{2}}$ and $K^{-1}\le K^{-\frac{1}{2}}$, the TV distance can be bounded by
\begin{align*}
\text{TV}(\pi, \tilde{\pi})
&\ge
1 - \frac{1}{K} - \frac{\log_A(A-1)}{K} - \frac{\log_A(K)}{K} - \frac{A}{K(A-1)} + \frac{1}{A^{K} - 1}\\
& \ge
1 - \left(1+\log_A(A-1)+\frac{1}{\log(A)}+\frac{A}{A-1}\right)K^{-\frac{1}{2}}.
\end{align*}
To make $\text{TV}(\pi, \tilde{\pi})\ge \delta$, we just need 
\[
K \ge \left(\frac{1+\log_A(A-1)+\frac{1}{\log(A)}+\frac{A}{A-1}}{1 - \delta}\right)^2.
\]
Since we could set $K$ arbitrarily, it is clear that we can do this.

\paragraph{Bounding the KL divergence.}

We can compute \text{KL} divergence between $\pi$ and $\tilde{\pi}$ as follows
\begin{align*}
    \text{KL}(\pi, \tilde{\pi}) &= \sum_{k=0}^{K-1} \frac{1}{K} \cdot \log \left(\frac{1}{K}\cdot \frac{1 - A^{K}}{(1-A)A^{k}}\right)\\
    &=
    \frac{1}{K} \cdot \sum_{k=0}^{K-1}\bigg[  \log \left(\frac{1}{K}\cdot \frac{1 - A^{K}}{(1-A)}\right) - k\log(A)\bigg]\\
    &= \log \left(\frac{1 - A^{K}}{K(1-A)}\right) -
    \frac{\log \left(A\right)}{K} \sum_{k=0}^{K-1} k \\
    & = \log \left(\frac{1 - A^{K}}{K(1-A)}\right) -
    \frac{(K-1) \log \left(A\right)}{2}\\
\end{align*}
If $A<1$, we have
\begin{align*}
    \text{KL}(\pi, \tilde{\pi}) 
    & = \log \left(1 - A^{K}\right) - \log ((1-A)K) -
    \frac{K\log \left(A\right)}{2} + \frac{\log \left(A\right)}{2}\\
    & \ge \log \left(1 - A^{K}\right) - 
    \left( \frac{1-A +\log \left(A\right)}{2}\right) K + \frac{\log \left(A\right)}{2}.
\end{align*}
The last equation is because $\log(x)\le \frac{x}{2}$.

To further simplify the above equation, we first note that $1-A +\log \left(A\right)<0$ when $A\neq 1$. And then when $K\ge \log_A\left(1 - A^{\frac{1}{2}}\right)$, we have $1 - A^{K}\ge A^{\frac{1}{2}}$. It follows that we can simplify it to be
\begin{align*}
    \text{KL}(\pi, \tilde{\pi}) 
    & \ge \log \left(A\right) - 
    \left(\frac{1-A +\log \left(A\right)}{2}\right) K.
\end{align*}
To make $\text{KL}(\pi, \tilde{\pi})\ge \rho$, it is clear that we just need to set
\[
K \ge \frac{2(\rho-\log(A))}{A-1-\log(A)}.
\]

Consider when $A>1$,
\begin{align*}
    \text{KL}(\pi, \tilde{\pi}) 
    & = \log \left(\frac{A^{K}-1}{K(A-1)}\right) -
    \frac{(K-1)\log \left(A\right)}{2}.
\end{align*}
If $K\ge \frac{\log(2)}{\log(A)}$, we have that $A^{K} - 1\ge \frac{A^{K}}{2}$. It follows that
\begin{align*}
    \text{KL}(\pi, \tilde{\pi}) 
    & \ge K\log(A) - \log(K) - \log(2A-2) -
    \frac{K\log \left(A\right)}{2}\\
    &=
    \frac{K\log \left(A\right)}{2} - \log(K) - \log(2A-2).
\end{align*}
To make $\text{KL}(\pi, \tilde{\pi})\ge \rho$, we need 
\[
\frac{K\log \left(A\right)}{2} - \log(K) \ge \rho + \log(2A-2).
\]
Let $K=\exp(y)$. By Taylor series, we know $\exp(y) \ge \frac{y^2}{2}$. Then it follows that
\[
\frac{y^2\log \left(A\right)}{4} - y \ge \rho + \log(2A-2).
\]
Solve the above inequality, we can get
\[
y \ge \frac{1+2\cdot \frac{\log \left(A\right)}{4} \cdot\bigg( \rho + \log(2A-2)\bigg)}{2 \cdot \frac{\log \left(A\right)}{4}} = \frac{2 + \log(A)\bigg(\rho+\log(2A-2)\bigg)}{\log(A)}.
\]
It follows that it suffices to set
\[
K\ge \exp\left(\frac{2 + \log(A)\bigg(\rho+\log(2A-2)\bigg)}{\log(A)}\right).
\]

\paragraph{Concluding the proof.}
The theorem now follows from choosing a $K$ large enough that both the TV distance inequality we derived and the KL divergence inequality we derived are satisfied.
\end{proof}

\section{Connection between Theorem~\ref{statement:counterexample} and TV Bound of Inexact MH Methods}

Some inexact methods such as MHSubLhd~\citep{bardenet2014towards} have bounded TV distance between the target distribution and the approximate distribution (see Proposition 3.2 in \citet{bardenet2014towards}). We would like to emphasize that 
Theorem~\ref{statement:counterexample} is compatible with these results. Specifically, Proposition 3.2 assumes $P_{\text{MH}}$ has a bounded mixing time. It is well known that this produces a TV bound for any kernel by coupling~\citep{levin2017markov}. Our theorem does not have this assumption; it suggests that for MHSubLhd, with a given user-specified error, there exists a target distribution and proposal satisfying Theorem~\ref{statement:counterexample}, on which $P_{\text{MH}}$ either does not have bounded mixing time or the mixing time is large enough such that the TV bound is greater than $\delta$.

\section{Proof of Statement \ref{statement:fmh}} \label{app:proof:smh}
\begin{proof}
We prove this by construction. Consider a dataset $\{x_i\}_{i=1}^N$. The data instances can take two values $\{-\frac{M}{N}, \frac{M}{N}\}$ where $M$ is a positive constant. Assume that half of the data instances take value $\frac{M}{N}$ and the remaining take $-\frac{M}{N}$. Let the target distribution be $\pi(\theta) = \frac{1}{Z}\exp\left(\theta\cdot\sum_{i=1}^N x_i\right)$ and the domain for $\theta$ be $\{0,1,\dots,K-1\}$. We define the proposal distribution to be the following
\[
p(\theta,\theta) = \frac{1}{2},\hspace{1em}\text{for all }\theta;\hspace{1em} p(\theta, \theta-1) = \frac{1}{4},\hspace{1em} p(\theta, \theta+1) = \frac{1}{4} \hspace{1em}\text{for }\theta\in\{1,\dots,K-2\}; 
\]
and $p(0,1)=p(K-1,K-2)=\frac{1}{2}$.

Recall that FMH factorizes the target distribution $\pi(\theta)$ and the proposal distribution $p(\theta)$ as follows
\[
\pi(\theta) \propto \prod^m_{i=1} \pi_i(\theta),\hspace{2em} p(\theta) \propto \prod^m_{i=1} p_i(\theta)
\]
where $m\geq 1$ and $\pi_i$ and $p_i$ are some non-negative functions. Then the acceptance rate is given by
\[
a_{\text{FMH}}(\theta,\theta') = \prod_{i=1}^m \min\left(1, \frac{\pi(\theta')p_i(\theta',\theta)}{\pi(\theta)p_i(\theta,\theta')}\right).
\]
A common choice is to set $m=N$. On this example, we can write the acceptance rate of transitioning from $\theta$ to $\theta'=\theta + 1$ in FMH as follows
\begin{align*}
    a_{\text{FMH}}(\theta,\theta') = \prod_{i=1}^N \min\left(1, \exp(x_i)\right)
    = \bigg(\exp\bigg(-\frac{M}{N}\bigg)\bigg)^{\frac{N}{2}}=\exp\bigg(-\frac{M}{2}\bigg).
\end{align*}
It is easy to show that the acceptance rate of transitioning from $\theta$ to $\theta'=\theta - 1$ in FMH is the same.

When $M>-2\log(p)$, it is clear that the acceptance rate of FMH is less than $p$. By contrast, the acceptance rate of standard MH is 
\[
a_{\text{MH}}(\theta,\theta') = \min\bigg(1, \exp\bigg(\pm\sum_{i=1}^N x_i\bigg)\bigg) = 1.
\]

In order to preserve geometric ergodicity, \citet{cornish2019scalable} introduces \emph{truncated FMH} (TFMH) which forces FMH degrade to standard MH when the energy exceeds a threshold $R$. If we set hyperparameter $R> M/2$, then in each step, the value of $a_{\text{TFMH}}$ will be the same as $a_{\text{FMH}}$. Therefore, if setting $M>-2\log(p)$, we have
\[
\frac{a_{\text{TFMH}}}{a_{\text{MH}}} \le \frac{p}{1} = p.
\]

If we set $R\le M/2$, TFMH falls back to standard, full-batch MH --- using the whole dataset at each step. This proves the statement.
\end{proof}

\section{Construction of Algorithm \ref{alg:poisson-mh}}\label{app:algo-derivation}

Algorithm \ref{alg:poisson-mh} can be derived by carefully replacing the global bounds on the energy in PoissonMH~\cite{zhang2019poisson} with local bounds on the energy differences (Assumption \ref{assump}). PoissonMH is a variant of Poisson Gibbs and therefore inherits the same assumptions for Gibbs sampling on graphical models, which are often violated in the applications of MH. In particular, PoissonMH works on \emph{factor graphs} which define a distribution $\pi(\theta)$ over a set of factors $\{\phi_i(\theta)\}_{i=1}^{N}$ as follows
\[
\pi(\theta) \propto \exp\left(\sum_{i=1}^N \phi_i(\theta)\right).
\]

PoissonMH assumes that each factor $\phi_i$ is non-negative without the loss of generality (we can add a positive constant to $\phi_i$ to make it non-negative without changing the distribution) and is bounded globally by a constant $M_i$. That is
\begin{align*}
   0 \le \phi_i(\theta) \le M_i \text{ for all } \theta.
\end{align*}
This assumption does not hold for most applications of MH, such as the linear and logistic regression experiments in Section \ref{sec:exp}. 

\begin{algorithm}[t]
  \caption{PoissonMH}
  \begin{algorithmic}
    \label{alg:old-poissonmh}
    \STATE \textbf{given:} initial state $\theta \in \Theta$; proposal dist. $q$; hyperparameter $\lambda$; Global bounds $M_i$, $L$
    \LOOP
      \STATE \textbf{propose} $\theta'\sim q(\cdot|\theta)$
      \FOR{$i \in \{1,\ldots,N\}$}
        \STATE \textbf{sample} $s_i \sim \text{Poisson}\left(\frac{\lambda M_i}{L} + \phi_i(\theta)\right)$ 
      \ENDFOR
      \STATE \textbf{form minibatch} $\mathcal{S}\leftarrow \{i | s_i>0\}$
      \STATE \textbf{compute MH ratio} $r \leftarrow \frac{\exp\left(\sum_{i\in \mathcal{S}} s_i\log\left( 1 + \frac{L}{\lambda M_{i}}\phi_i(\theta') \right)\right)q(\theta'|\theta)}{\exp\left(\sum_{i\in \mathcal{S}} s_i\log\left( 1 + \frac{L}{\lambda M_{i}}\phi_i(\theta) \right)\right)q(\theta|\theta')}$
      \STATE \textbf{with probability} $\min(1,r)$, set $\theta \leftarrow \theta'$
    \ENDLOOP
  \end{algorithmic}
\end{algorithm}

Let $L = \sum_i M_i$ and define Poisson auxiliary variable $s_i$ as the following

\[
s_i|\theta \sim \text{Poisson}\left(\frac{\lambda M_i}{L} + \phi_i(\theta)\right),
\]
where $\lambda>0$ is a hyperparameter. Running standard MH on the joint distribution of $\theta$ and $s_i$ results in the following acceptance ratio
\begin{align*}
    r_{\text{PoissonMH}}(\theta, \theta') = \frac{\exp\left(\sum_{i} s_{i} \log\left( 1 + \frac{L}{\lambda M_{i}}\phi_i(\theta') \right)\right)q(\theta'|\theta)}{\exp\left(\sum_{i} s_{i} \log\left( 1 + \frac{L}{\lambda M_{i}}\phi_i(\theta) \right)\right)q(\theta|\theta')}.
\end{align*}

Here, the sum is essentially performed over the set of index $i$ whose $s_i$ is greater than zero. When $s_i = 0$, it is clear that the factor $\phi_i$ will not appear in the acceptance ratio $r_{\text{PoissonMH}}$. Thus PoissonMH enables using a subset of factors for the MH decision step (Algorithm~\ref{alg:old-poissonmh}).

To construct our method from this, we can define the factor $\phi_i$ in the factor graph to be
\begin{align}\label{eq:phi}
    \phi_i(x) = \frac{U_i(\theta)+U_i(\theta')}{2} - U_i(x) + \frac{c_i}{2} M(\theta,\theta')
\end{align}

where $x\in\{\theta, \theta'\}$. It is easy to see that $\phi_i$ satisfy $0\le \phi_i(x) \le c_i M(\theta,\theta')$. And then we define the Poisson variables $s_i$ as the follows
\[
s_i|(\theta, \theta') \sim \text{Poisson}\left(\frac{\lambda c_{i}}{C} + \phi_i(\theta)\right) = \text{Poisson}\left(\frac{\lambda c_{i}}{C} + \frac{U_i(\theta') - U_i(\theta) + c_iM(\theta,\theta')}{2}\right).
\]

These Poisson auxiliary variables $\{s_i\}_{i=1}^N$ are called \emph{local}, because their distributions change each iteration depending on the current pair $(\theta, \theta')$ and only rely on local bounds in  Assumption \ref{assump}. This is in contrast to the \emph{global} auxiliary variables used in PoissonMH and FlyMC which are used to form a joint distribution with $\theta$ and both require global bounds in their conditional distributions.

The acceptance ratio $r_{\text{\methodname}}$ is the same as $r_{\text{PoissonMH}}$ but with the new definitions of $s_i$ and $\phi_i$. We outline \methodname{} using the notation of $\phi_i$ and $s_i$ in Algorithm \ref{alg:poisson-mh2}.

We now show that Algorithm \ref{alg:poisson-mh2} is statistically equivalent to Algorithm \ref{alg:poisson-mh}. To see this, we first use \emph{thinning}, a commonly used technique \citep{lewis1979simulation, bierkens2019zig,bouchard2018bouncy,cornish2019scalable,zhang2019poisson}, to quickly resample all $s_i$ from their new distributions in each iteration in Algorithm~\ref{alg:poisson-mh2}. This is achieved by replacing the global bounds with the local bounds in Algorithm 4 in the Appendix of \citet{zhang2019poisson}. Specifically, we first sample $B$ from a Poisson distribution
\[
B\sim \text{Poisson}(\lambda + CM(\theta,\theta')).
\]
Here $\lambda + CM(\theta,\theta')$ is an upper bound on $\mathbf{E}[\sum_i s_i]$. We then form the minibatch by running

\begin{figure}[h]
  \centering
  \begin{minipage}{.7\linewidth}
\begin{algorithmic}
\FOR{$b \in \{1,\ldots,B\}$}
        \STATE \textbf{sample} $i_b$ such that $\mathbf{P}(i_b = i) = c_i/C$, for $i=1\ldots N$
        \STATE \textbf{with probability} $\frac{\lambda c_{i_b} + C\phi_{i_b}(\theta)}{\lambda c_{i_b} + Cc_{i_b}M(\theta, \theta')}$ \textbf{add} $i_b$ to $\mathcal{I}$ 
        \ENDFOR
\end{algorithmic}
  \end{minipage}
\end{figure}

By substituting $\lambda=\chi C^2 M^2(\theta,\theta')$ and the expression of $\phi_i$, we can get the part of ``form minibatch $\mathcal{I}$'' in Algorithm~\ref{alg:poisson-mh}. 

To see that the MH ratio in Algorithm~\ref{alg:poisson-mh} and~\ref{alg:poisson-mh2} are equivalent, we can write out $r$ in Algorithm~\ref{alg:poisson-mh2} using the above fast way of resampling $s_i$
\[
r_{\text{\methodname}} = \frac{\exp\left(\sum_{i \in \mathcal{I}} \log\left( 1 + \frac{C}{\lambda c_{i}}\phi_i(\theta') \right)\right)q(\theta'|\theta)}{\exp\left(\sum_{i \in \mathcal{I}} \log\left( 1 + \frac{C}{\lambda c_{i}}\phi_i(\theta) \right)\right)q(\theta|\theta')}.
\]

\begin{algorithm}[t]
  \caption{\methodname{}}
  \begin{algorithmic}
    \label{alg:poisson-mh2}
    \STATE \textbf{given:} initial state $\theta \in \Theta$; proposal dist. $q$; $\lambda$; Asm.~\ref{assump} parameters $c_i$, $C$, $M$; function $\phi_i$ defined in~(\ref{eq:phi})
    \LOOP
      \STATE \textbf{propose} $\theta'\sim q(\cdot|\theta)$ and \textbf{compute} $M(\theta, \theta')$
      \FOR{$i \in \{1,\ldots,N\}$}
        \STATE \textbf{sample} $s_i \sim \text{Poisson}\left(\frac{\lambda c_i}{C} + \phi_i(\theta)\right)$ 
      \ENDFOR
      \STATE \textbf{form minibatch} $\mathcal{S}\leftarrow \{i | s_i>0\}$
      \vspace{0.5em}
      \STATE \textbf{compute MH ratio} $r \leftarrow \frac{\exp\left(\sum_{i \in \mathcal{S}} s_i \log\left( 1 + \frac{C}{\lambda c_{i}}\phi_i(\theta') \right)\right)q(\theta'|\theta)}{\exp\left(\sum_{i \in \mathcal{S}} s_i \log\left( 1 + \frac{C}{\lambda c_{i}}\phi_i(\theta) \right)\right)q(\theta|\theta')}$
      \STATE \textbf{with probability} $\min(1,r)$, set $\theta \leftarrow \theta'$
    \ENDLOOP
  \end{algorithmic}
\end{algorithm}

We then substitute the definition of $\phi_i$ in (\ref{eq:phi}) and it follows that
\begin{align*}
    r_{\text{\methodname}}
    &=
    \exp\bigg(\sum_{i \in \mathcal{I}} \bigg(\log\bigg(\frac{2\lambda c_i + C\left(U_i(\theta) - U_i(\theta') + c_iM(\theta,\theta')\right)}{2\lambda c_i + C\left(U_i(\theta') - U_i(\theta) + c_iM(\theta,\theta')\right)}\bigg)\bigg)\bigg)
        \cdot \frac{q(\theta'|\theta)}{q(\theta|\theta')}.
\end{align*}

We can rearrange the $\log$ term inside $r_{\text{\methodname}}$ as
\begin{align*}
    &\hspace{-2em}\log\left(\frac{2\lambda c_i + C\left(U_i(\theta) - U_i(\theta') + c_iM(\theta,\theta')\right)}{2\lambda c_i + C\left(U_i(\theta') - U_i(\theta) + c_iM(\theta,\theta')\right)}\right)
    \\&=
    \log\left(\frac{2\lambda c_i + C\left(U_i(\theta) - U_i(\theta') \right) + c_i C M(\theta,\theta')}{2\lambda c_i + C\left(U_i(\theta') - U_i(\theta) \right) + c_i C M(\theta,\theta')}\right)
    \\&=
    \log\left(\frac{1 + \frac{C}{2 \lambda c_i + c_i C M(\theta,\theta')} \left(U_i(\theta) - U_i(\theta') \right)}{1 + \frac{C}{2 \lambda c_i + c_i C M(\theta,\theta')} \left(U_i(\theta') - U_i(\theta) \right)}\right)
    \\&=
    2 \operatorname{artanh}\left(
        \frac{C \left(U_i(\theta) - U_i(\theta') \right)}{c_i (2 \lambda + C M(\theta,\theta'))} 
    \right).
\end{align*}
So $r_{\text{\methodname}}$ can be written as
\[
r_{\text{\methodname}} =
        \exp\left(2 \sum_{i \in \mathcal{I}} \operatorname{artanh}\left(
        \frac{C \left(U_i(\theta) - U_i(\theta') \right)}{c_i (2 \lambda + C M(\theta,\theta'))} 
    \right) \right)
        \cdot \frac{q(\theta'|\theta)}{q(\theta|\theta')}.
\]
Finally setting $\lambda$ to be $\chi C^2M^2(\theta,\theta')$ produces the MH ratio in Algorithm \ref{alg:poisson-mh}.

By proving the equivalence of the minibatch and the MH ratio, we show that Algorithm~\ref{alg:poisson-mh} and \ref{alg:poisson-mh2} are statistically equivalent.



\section{Proof of Theorem \ref{thm:spectral-gap}} \label{app:proof:spectral-gap}

In this section, we prove Theorem \ref{thm:spectral-gap}, which asserts that \methodname{} is reversible and has stationary distribution $\pi$, and gives bounds on its spectral gap relative to the spectral gap of the original Metropolis-Hastings algorithm.

\begin{proof}  
For convenience, we prove Theorem \ref{thm:spectral-gap} using Algorithm \ref{alg:poisson-mh2} statement which is statistically equivalent to Algorithm \ref{alg:poisson-mh}.
The transition operator can be written as the following 
\begin{align*}
  &T( \theta,  \theta')\\
  &=
\mathbf{E}\left\{q( \theta'| \theta)\min\left(1, \frac{q( \theta| \theta')\exp\left(  \sum_i  \left[s_{i}\log\left( \frac{\lambda c_i}{C}+ \phi_i(\theta')\right) - \log s_{i}!\right]\right)}
{q( \theta'| \theta)\exp\left(  \sum_i  \left[s_{i}\log\left( \frac{\lambda c_i}{C}+ \phi_i(\theta)\right) - \log s_{i}!\right]\right)}\right)\right\}\\
&=
\mathbf{E}\left\{q( \theta'| \theta)\min\left(1, \frac{q( \theta| \theta')\exp\left(  \sum_i  \left[s_{i}\log\left( \frac{\lambda c_i}{C}+ \phi_i(\theta')\right)\right]\right)}
{q( \theta'| \theta)\exp\left(  \sum_i  \left[s_{i}\log\left( \frac{\lambda c_i}{C}+ \phi_i(\theta)\right) \right]\right)}\right)\right\}\\
&=
\sum_s\left\{q( \theta'| \theta)\min\left(1, \frac{q( \theta| \theta')\exp\left(  \sum_i  \left[s_{i}\log\left( \frac{\lambda c_i}{C}+ \phi_i(\theta')\right)\right]\right)}
{q( \theta'| \theta)\exp\left(  \sum_i  \left[s_{i}\log\left( \frac{\lambda c_i}{C}+ \phi_i(\theta)\right) \right]\right)}\right)\right\}\prod_{i} p(s_{i}| \theta,\theta')\\
&=
\sum_s\left\{q( \theta'| \theta)\min\left(\exp\left(\sum_i  \left[s_{i}\log\left( \frac{\lambda c_i}{C}+ \phi_i(\theta)\right) 
- \phi_i(\theta) -  \frac{\lambda c_i}{C}- \log s_{i}!\right] \right),\right.\right.\\
&\hspace{2em}\left.\left.\frac{q( \theta| \theta')\exp\left(  \sum_i  \left[s_{i}\log\left( \frac{\lambda c_i}{C}+ \phi_i(\theta')\right)\right]\right)}
{q( \theta'| \theta)\exp\left(  \sum_i \phi_i(\theta) +  \frac{\lambda c_i}{C}+ \log s_{i}! \right)}\right)\right\}\\
&=
\sum_s\left\{q( \theta'| \theta)\min\left(\exp\left(\sum_i  \left[s_{i}\log\left( \frac{\lambda c_i}{C}+ \phi_i(\theta)\right) 
- \phi_i(\theta) -  \frac{\lambda c_i}{C}- \log s_{i}!\right] \right),\right.\right.\\
&\hspace{2em}\left.\left.\frac{q( \theta| \theta') }
{q( \theta'| \theta) }\exp\left(\sum_i  \left[s_{i}\log\left( \frac{\lambda c_i}{C}+ \phi_i(\theta')\right) 
- \phi_i(\theta) -  \frac{\lambda c_i}{C}- \log s_{i}!\right] \right)\right)\right\}
\end{align*}

Multiplying $\pi(\theta)$ to both sides produces
\begin{align*}
&\hspace{-2em}\pi(\theta)T( \theta,  \theta')\\
  &=
\frac{1}{Z}\exp\left(-\sum_i U_i(\theta) \right)T( \theta,  \theta')\\
  &=
\frac{1}{Z}\sum_s
\min\Bigg(q( \theta'| \theta) \exp\bigg(\sum_i  \bigg[s_{i}\log\left( \frac{\lambda c_i}{C}+ \phi_i(\theta)\right) \\
&\hspace{3em}- \frac{U_i(\theta)+U_i(\theta')}{2} - \frac{c_i}{2} M(\theta,\theta') -  \frac{\lambda c_i}{C}- \log s_{i}!\bigg] \bigg),\\
&\hspace{2em}q( \theta| \theta') 
\exp\bigg(\sum_i  \bigg[s_{i}\log\left( \frac{\lambda c_i}{C}+ \phi_i(\theta')\right) 
\\&\hspace{3em}- \frac{U_i(\theta)+U_i(\theta')}{2} - \frac{c_i}{2} M(\theta,\theta') -  \frac{\lambda c_i}{C}- \log s_{i}!\bigg] \bigg)\bigg)\Bigg).  
\end{align*}
It is clear that the expression is symmetric in $\theta$ and $\theta'$. Therefore the chain is reversible and its stationary distribution is $\pi(\theta)$. This proves the first part of the theorem.


To prove the second part of the theorem, the bound on the spectral gap, we continue to reduce the transition probability in the previous proof to
\begin{align*}
  &\hspace{-2em}\pi( \theta) T( \theta,  \theta')\\
  &=
\frac{1}{Z}\sum_s
\min\Bigg(q( \theta'| \theta) \exp\bigg(\sum_i  \bigg[s_{i}\log\left( \frac{\lambda c_i}{C}+ \phi_i(\theta)\right) \\&\hspace{3em}- \frac{U_i(\theta)+U_i(\theta')}{2} - \frac{c_i}{2} M(\theta,\theta') - s_i\log\frac{\lambda c_i}{C}\bigg] \bigg),
\\&\hspace{2em}q( \theta| \theta') 
\exp\bigg(\sum_i  \bigg[s_{i}\log\left( \frac{\lambda c_i}{C}+ \phi_i(\theta')\right) 
\\&\hspace{3em}- \frac{U_i(\theta)+U_i(\theta')}{2} - \frac{c_i}{2} M(\theta,\theta') -  s_i\log\frac{\lambda c_i}{C}\bigg] \bigg)\Bigg)
\\&\hspace{2em}\cdot \prod_{i} \frac{1}{s_{i}!}\exp\left(- \frac{\lambda c_{i}}{C}\right)\left( \frac{\lambda c_{i}}{C}\right)^{s_{i}}\\
  &=
\frac{1}{Z}\sum_s
\min\Bigg(q( \theta'| \theta) \exp\bigg(\sum_i  \bigg[s_{i}\log\left(1 + \frac{C}{\lambda c_i} \phi_i(\theta)\right) 
\\&\hspace{3em}- \frac{U_i(\theta)+U_i(\theta')}{2} - \frac{c_i}{2} M(\theta,\theta')\bigg] \bigg),\\
&\hspace{2em}q( \theta| \theta') 
\exp\left(\sum_i  \left[s_{i}\log\left(1 + \frac{C}{\lambda c_i} \phi_i(\theta')\right) - \frac{U_i(\theta)+U_i(\theta')}{2} - \frac{c_i}{2} M(\theta,\theta') \right] \right)\Bigg)
\\&\hspace{2em}\cdot \prod_{i} \frac{1}{s_{i}!}\exp\left(- \frac{\lambda c_{i}}{C}\right)\left( \frac{\lambda c_{i}}{C}\right)^{s_{i}}.
\end{align*}

Note that $s_{i}$ here are non-negative integers that a Poisson variable can take, not variables. So if we let $r_{i} \sim \text{Poisson}\left( \frac{\lambda c_{i}}{C} \right)$ and $r_{i}$ to be all independent, we can write this as
\begin{align*}
  \pi( \theta) T( \theta,  \theta')
  &=
\frac{1}{Z}\mathbf{E}
\min\left(q( \theta'| \theta) \exp\left(\sum_{i} r_{i}\log\left( 1+ \frac{C}{\lambda c_{i}}\phi_i( \theta) \right)\right),\right.\\
&\left.q( \theta| \theta') 
\exp\left(\sum_{i} r_{i}\log\left( 1+ \frac{C}{\lambda c_{i}}\phi_i( \theta')\right)\right)\right) 
\\&\hspace{2em}\cdot \exp\bigg[ -\frac{1}{2}\bigg( \sum_i U_i(\theta) + \sum_i U_i(\theta') + C M(\theta,\theta')\bigg)\bigg].
\end{align*}

Assume $G( \theta,  \theta')$ is the transition operator of standard MH. Consider the ratio
\begin{align*}
&\frac{\pi( \theta) T( \theta,  \theta')}
  {\pi( \theta) G( \theta,  \theta')}\\
&=
\frac{1}{Z}\mathbf{E}
\min\left(q( \theta'| \theta) \exp\left(\sum_{i} r_{i}\log\left( 1+ \frac{C}{\lambda c_{i}}\phi_i( \theta) \right)\right),\right.\\
&\hspace{2em}\left.q( \theta| \theta') 
\exp\left(\sum_{i} r_{i}\log\left( 1+ \frac{C}{\lambda c_{i}}\phi_i( \theta')\right)\right)\right) 
\\&\hspace{2em}\cdot \exp\bigg[ -\frac{1}{2}\bigg( \sum_i U_i(\theta) + \sum_i U_i(\theta') + C M(\theta,\theta')\bigg)\bigg]
\\&\hspace{2em}\cdot\Bigg[1\bigg/\Bigg(\frac{1}{Z}
\min\left(q( \theta'| \theta)\exp\left(-\sum_i U_i( \theta)\right),
q( \theta| \theta')\exp\left(-\sum_i U_i( \theta')\right) \right)\Bigg)\Bigg].
\end{align*}

We know that $\frac{\min(A, B)}{\min(C,D)} = \min\left(\frac{A}{\min(C,D)}, \frac{B}{\min(C,D)}\right) \geq \min\left(\frac{A}{C}, \frac{B}{D}\right)$. The last inequality is due to the fact that $\frac{1}{\min(C,D)}\geq \frac{1}{C}$ and $\frac{1}{\min(C,D)}\geq \frac{1}{D}$.

With this inequality, we can continue simplifying the ratio,
\begin{align*}
  &\hspace{-1em}\frac{\pi( \theta) T( \theta,  \theta')}
  {\pi( \theta) G( \theta,  \theta')}\\
&\geq
\mathbf{E}\Bigg[\min \Bigg(\frac{
\exp\left(\sum_i r_i\log\left( 1+ \frac{C}{\lambda c_{i}}\phi_i( \theta) \right)  \right)}
{
\exp\left(-\sum_i U_i( \theta)\right)},
\frac{
\exp\left(\sum_i r_i\log\left( 1+ \frac{C}{\lambda c_{i}}\phi_i( \theta')\right)  \right)}
{
\exp\left(-\sum_i U_i( \theta')\right)}
\Bigg)\Bigg]
\\&\hspace{2em}\cdot \exp\bigg[ -\frac{1}{2}\bigg( \sum_i U_i(\theta) + \sum_i U_i(\theta') + C M(\theta,\theta')\bigg)\bigg]\\
&=
\mathbf{E}\Bigg[\min \Bigg(
\exp\left(\sum_i\Bigg( r_i\log\left( 1+ \frac{C}{\lambda c_{i}}\phi_i( \theta) \right) - \phi_i( \theta) \Bigg)\right),
\\&\hspace{2em}\exp\left(\sum_i\Bigg( r_i\log\left( 1+ \frac{C}{\lambda c_{i}}\phi_i( \theta')\right) - \phi_i( \theta') \Bigg)\right)
\Bigg)\Bigg]\\
&=
\mathbf{E}\Bigg[\max \Bigg(
\exp\left(\sum_i \Bigg(\phi_i( \theta) - r_i\log\left( 1+ \frac{C}{\lambda c_{i}}\phi_i( \theta) \right)\Bigg)\right),
\\&\hspace{2em}\exp\left(\sum_i \Bigg(\phi_i( \theta') - r_i\log\left( 1+ \frac{C}{\lambda c_{i}}\phi_i( \theta')\right)\Bigg) \right)
\Bigg)^{-1}\Bigg].
\end{align*}

Because $f(x) = \frac{1}{x}$ is a convex function, by Jensen's inequality it follows
\begin{align*}
  \frac{\pi( \theta) T( \theta,  \theta')}
  {\pi( \theta) G( \theta,  \theta')}
&\geq
\mathbf{E}\Bigg[\max \Bigg(
\exp\left(\sum_i \Bigg(\phi_i( \theta) - r_i\log\left( 1+ \frac{C}{\lambda c_{i}} \phi_i( \theta) \right)\Bigg) \right),
\\&\hspace{2em}\exp\left(\sum_i\Bigg( \phi_i( \theta') - r_i\log\left(1+ \frac{C}{\lambda c_{i}}\phi_i( \theta')\right) \Bigg)\right)
\Bigg)\Bigg]^{-1}.
\end{align*}

We use $\max(A,B)\leq (A^p + B^p)^{\frac{1}{p}}$ to remove the $\max$ function.
\begin{align*}
  \frac{\pi( \theta) T( \theta,  \theta')}
  {\pi( \theta) G( \theta,  \theta')}
&\geq
\mathbf{E}\Bigg[\bigg(
\exp \bigg(p\sum_i\Bigg(\phi_i( \theta) - r_i\log\left( 1+ \frac{C}{\lambda c_{i}} \phi_i( \theta) \right)\Bigg)\bigg) +
\\&\hspace{2em}\exp\bigg(p\sum_i\Bigg( \phi_i( \theta') - r_i\log\left(1+ \frac{C}{\lambda c_{i}}\phi_i( \theta')\right)  \Bigg)\bigg)\bigg)^{\frac{1}{p}}\Bigg]^{-1}.
\end{align*}
Since $x^{\frac{1}{p}}$ is concave, by Jensen's inequality
\begin{align*}
  \frac{\pi( \theta) T( \theta,  \theta')}
  {\pi( \theta) G( \theta,  \theta')}
&\geq
\mathbf{E}\Bigg[
\exp \bigg(p\sum_i\Bigg(\phi_i( \theta) - r_i\log\left( 1+ \frac{C}{\lambda c_{i}} \phi_i( \theta) \right)\Bigg)\bigg) +
\\&\hspace{2em}\exp\bigg(p\sum_i\Bigg( \phi_i( \theta') - r_i\log\left(1+ \frac{C}{\lambda c_{i}}\phi_i( \theta')\right)  \Bigg)\bigg)\Bigg]^{-\frac{1}{p}}\\
&=
\Bigg[\prod_i\mathbf{E}
\exp \Bigg(p\phi_i( \theta) - pr_i\log\left( 1+ \frac{C}{\lambda c_{i}} \phi_i( \theta) \right)\Bigg) +
\\&\hspace{2em}\prod_i\mathbf{E}\exp\Bigg( p\phi_i( \theta') - pr_i\log\left(1+ \frac{C}{\lambda c_{i}}\phi_i( \theta')\right)  \Bigg)\Bigg]^{-\frac{1}{p}}.
\end{align*}

$\mathbf{E}\Bigg[
\exp \Bigg( - pr_i\log\left( 1+ \frac{C}{\lambda c_{i}} \phi_i( \theta) \right)\Bigg)\Bigg]$ is the moment generating function of the Poisson random variable $r_i$ evaluated at
\[
t = -p\log\left( 1+ \frac{C}{\lambda c_{i}} \phi_i( \theta) \right).
\]
We know that
\begin{align*}
\mathbf{E}\exp(r_it) 
&= 
\exp\left( \frac{\lambda c_{i}}{C}\left(\exp(t) - 1\right)\right),
\end{align*}
therefore,
\begin{align*}
\mathbf{E}\Bigg[
\exp \Bigg( - pr_i\log\left( 1+ \frac{C}{\lambda c_{i}} \phi_i( \theta) \right)\Bigg)\Bigg]
&= 
\exp\left( \frac{\lambda c_{i}}{C}\left(1+ \frac{C}{\lambda c_{i}} \phi_i( \theta)\right)^{-p} - \frac{\lambda c_{i}}{C}\right).
\end{align*}
Substituting this into the original expression produces
\begin{align*}
  \frac{\pi( \theta) T(\theta, \theta')}
  {\pi( \theta) G(\theta, \theta')}
  &\geq
  \Bigg[\prod_i\exp\left( \frac{\lambda c_{i}}{C}\left(1+ \frac{C}{\lambda c_{i}} \phi_i( \theta)\right)^{-p} - \frac{\lambda c_{i}}{C} + p\phi_i(\theta)\right) \\&\hspace{4em}+ \prod_i\exp\left( \frac{\lambda c_{i}}{C}\left(1+ \frac{C}{\lambda c_{i}} \phi_i( \theta')\right)^{-p} - \frac{\lambda c_{i}}{C} + p\phi_i(\theta')\right)\Bigg]^{-\frac{1}{p}}.
\end{align*}

Considering the term inside $\exp$. Define a function $f(y) = \frac{\lambda c_{i}}{C}\left(1+ \frac{C}{\lambda c_{i}}y\right)^{-p} - \frac{\lambda c_{i}}{C} + py$ for $y\ge 0$. It is clear that $f(0) = 0$. The first derivative is 
\[
f'(y) = p + (-p)\left(1 + \frac{C}{\lambda c_i}y\right)^{-p-1}
\]
which is also 0 at $y=0$. The second and third derivatives are
\begin{align}\label{eq:second-derivative}
f''(y) &= (-p)(-p-1)\frac{C}{\lambda c_i}\left(1 + \frac{C}{\lambda c_i}y\right)^{-p-2}, \\
f'''(y) &= (-p)(-p-1)(-p-2)\left(\frac{C}{\lambda c_i}\right)^2\left(1 + \frac{C}{\lambda c_i}y\right)^{-p-3}.    \label{eq:third-derivative}
\end{align}

By Taylor series, we have
\begin{align*}
    f(y) = f(0) + f'(0)y + \frac{f''(0)}{2!}y^2 + \frac{f'''(v)}{3!}y^3
\end{align*}
where $v$ is between 0 and $y$. By (\ref{eq:third-derivative}), we know that $f'''(v)\le 0$, therefore since $y \ge 0$, we have
\begin{align*}
    f(y) &\le f(0) + f'(0)y + \frac{f''(0)}{2!}y^2\\
    &=\frac{f''(0)}{2!}y^2.
\end{align*}
Substituting $y=\phi_i(\theta)$ produces
\begin{align*}
    f(\phi_i(\theta)) &\le (-p)(-p-1)\frac{C}{\lambda c_i}\phi_i^2(\theta)\\
    &\le (-p)(-p-1)\frac{C}{\lambda c_i}c_i^2M^2(\theta,\theta').
\end{align*}
Similarly, we can get 
\begin{align*}
    f(\phi_i(\theta'))
    &\le p(p+1)\frac{C}{\lambda c_i}c_i^2M^2(\theta,\theta').
\end{align*}
Substituting these to the spectral ratio, we get
\begin{align*}
  \frac{\pi( \theta) T(\theta, \theta')}
  {\pi( \theta) G(\theta, \theta')}
&\geq
\left[2\prod_i\exp\left( p(p+1)\frac{C}{\lambda c_i}c_i^2M^2(\theta,\theta')\right)\right]^{-\frac{1}{p}}\\
&=
\left[2\exp\left(\sum_i p(p+1)\frac{C}{\lambda }c_i M^2(\theta,\theta')\right)\right]^{-\frac{1}{p}}\\
&=
\left[2\exp\left( p(p+1)\frac{C^2}{\lambda } M^2(\theta,\theta')\right)\right]^{-\frac{1}{p}}\\
&=
2^{-\frac{1}{p}}\exp\left(-(p+1)\frac{C^2}{\lambda } M^2(\theta,\theta')\right).
\end{align*}
Now, we maximize the R.H.S. with respect to $p$. Let $E=\frac{C^2}{\lambda } M^2(\theta,\theta')$, then it becomes
\begin{align*}
  2^{-\frac{1}{p}}\exp\left(-(p+1)E\right) &= \exp\left(-E-pE - \frac{1}{p}\log 2\right).  
\end{align*}
The maximum is attained at $p = \sqrt{\frac{\log 2}{E}}$ and the value is 
\begin{align*}
 \exp\left(-E-2\sqrt{E\log 2}\right).
\end{align*}
It follows that
\begin{align*}
  \frac{\pi( \theta) T(\theta, \theta')}
  {\pi( \theta) G(\theta, \theta')}
&\geq
\exp\left(-\frac{C^2}{\lambda } M^2(\theta,\theta')-2\sqrt{\frac{C^2}{\lambda } M^2(\theta,\theta')\log 2}\right). 
\end{align*}
We set $\lambda = \chi C^2M^2(\theta,\theta')$, it becomes 
  \[
      \frac{\pi( \theta) T( \theta,  \theta')}
  {\pi( \theta) G( \theta,  \theta')}
    \ge
    \exp \Bigg(-\frac{1}{\chi} - 2\sqrt{\frac{\log 2}{\chi}} \Bigg).
  \]

We complete the theorem by a Dirichlet form argument. We can write the Dirichlet form $\mathcal{E}(f)$ of a Markov chain with transition operator $G$ as \citep{fukushima2010dirichlet}:
\begin{align*}
\mathcal{E}(f) 
	= 
	\frac{1}{2}\int\int\left[\left(f(\theta)-f(\theta')\right)^2\right]G(\theta,\theta')\pi(\theta)d\theta d\theta'.
\end{align*}
If we let $L^2_0(\pi)$ to be the Hilbert space of functions $f$ such that $f$ has mean zero and is square integrable with respect to probability measure $\pi$. It follows that the spectral gap $\gamma$ of a Markov chain is \citep{aida1998uniform}
\[
\gamma = \inf_{f\in L^2_0(\pi): Var_{\pi}[f] = 1} \mathcal{E}(f).
\]
From this, it is easy to get that
\begin{align*}
\bar\gamma &= \inf_{f\in L^2_0(\pi): Var_{\pi}[f] = 1} \left[\frac{1}{2}\int\int\left[\left(f(\theta)-f(\theta')\right)^2\right]T(\theta,\theta')\pi(\theta)d\theta d\theta'\right]\\
&\geq 
\exp \Bigg(-\frac{1}{\chi} - 2\sqrt{\frac{\log 2}{\chi}} \Bigg)
\cdot \inf_{f\in L^2_0(\pi): Var_{\pi}[f] = 1} \left[\frac{1}{2}\int\int\left[\left(f(\theta)-f(\theta')\right)^2\right] G(\theta,\theta')\pi(\theta)d\theta d\theta'\right]\\
&= \exp \Bigg(-\frac{1}{\chi} - 2\sqrt{\frac{\log 2}{\chi}} \Bigg)\cdot \gamma.
\end{align*}
\end{proof}

\section{Derivation of Equation~(\ref{eq:TunaMHEB})}\label{app:chi-value}

Based on the bound in Theorem~\ref{thm:spectral-gap}, to make sure that the spectral ratio $\bar\gamma/\gamma \ge \kappa$, we can set $\chi$ such that
\[
\exp \Bigg(-\frac{1}{\chi} - 2\sqrt{\frac{\log 2}{\chi}} \Bigg) = \kappa.
\]
Solving the above equation gives us
\[
\chi = \frac{(2\log2-\log \kappa+2\sqrt{\log2(\log2-\log\kappa)})}{\log^2\kappa}\le \frac{4}{(1-\kappa)\log(1/\kappa)}.
\]
Since the spectral gap ratio is monotonically increasing w.r.t. $\chi$, we can instead set $\chi$ to the upper bound
\[
\chi = \frac{4}{(1-\kappa)\log(1/\kappa)}
\]
which guarantees that $\bar\gamma/\gamma \ge \kappa$.

\section{Theoretically Optimal Value of $\chi$}\label{app:optimal-value}

The overall wall-clock time $L$ for a chain to converge can be represented as the number of steps times the wall-clock time $l$ of each step. We then minimize an upper bound of this overall wall-clock time to get the optimal value of $\chi$. 

Consider a lazy Markov chain on a finite state $\Theta$. The \emph{relaxation time} $t_{\text{rel}}$ of a Markov chain is defined to be the inverse of the spectral gap $\gamma$: $t_{\text{rel}} = 1/\gamma$. 
The \emph{mixing time} $t_{\text{mix}}$, i.e. the number of steps required for a chain to converge to within TV distance $\delta$ to the target distribution $\pi$, is bounded by~\citet{levin2017markov}
\[
t_{\text{mix}}\le t_{\text{rel}} \log\left(\frac{1}{\delta\cdot\min_{\theta\in\Theta}\pi(\theta)}\right).
\]

It follows that the overall wall-clock time $L$ is upper bouned by
\[
L = l \cdot t_{\text{mix}}\le l\cdot t_{\text{rel}}\log\left(\frac{1}{\delta\cdot\min_{\theta\in\Theta}\pi(\theta)}\right).
\]

We assume that the expected wall clock time to run a step is proportional to the batch size plus some constant, which measures the cost of computing the proposal. Specifically, We use $\eta$ and $\xi$ to denote the time to get a proposal $\theta'$ and compute a $U_i$ in a step. Then we can write the time of a step $l$ as 
\[
l = B\xi + \eta.
\]

In order to minimize $L$, we can instead minimize its upper bound, which is equivalent to minimize
\begin{align}\label{eq:wall-clock-time}
    l\cdot t_{\text{rel}} = (B\xi + \eta)\cdot \frac{1}{\gamma}.
\end{align}

Recall that for \methodname{}, the average batch size over all steps is 
\[
\mathbf{E}_{(\theta,\theta')\sim \pi(\theta)q(\theta'|\theta)}[\chi C^2M^2(\theta,\theta') + CM(\theta,\theta')],
\]

and the spectral gap $\bar \gamma$ is lower bounded by the spectral gap of standar MH $\gamma$ such that
\[
    \bar{\gamma}
    \ge
    \exp \Bigg(-\frac{1}{\chi} - 2\sqrt{\frac{\log 2}{\chi}} \Bigg)\cdot\gamma.
\]

Substituting the expression of batch size and spectral gap to (\ref{eq:wall-clock-time}) gives
\[
l\cdot t_{\text{rel}} \le \left(\mathbf{E}_{(\theta,\theta')\sim \pi(\theta)q(\theta'|\theta)}[\chi C^2M^2(\theta,\theta') + CM(\theta,\theta')]\xi + \eta\right)\cdot \exp \Bigg(\frac{1}{\chi} + 2\sqrt{\frac{\log 2}{\chi}} \Bigg)\cdot\frac{1}{\gamma}.
\]

To minimize the RHS of the above equation over $\chi$, we let the derivative w.r.t. $\chi$ to be zero and get,
\begin{align*}
    &\hspace{-2em}\xi C^2\mathbf{E}_{(\theta,\theta')\sim \pi(\theta)q(\theta'|\theta)}[M^2(\theta,\theta')] \chi^{-1} + (\xi C\mathbf{E}_{(\theta,\theta')\sim \pi(\theta)q(\theta'|\theta)}[M(\theta,\theta')] + \eta)\chi^{-2} \\&+ \sqrt{\log 2}\xi C^2\mathbf{E}_{(\theta,\theta')\sim \pi(\theta)q(\theta'|\theta)}[M^2(\theta,\theta')]\chi^{-\frac{1}{2}} \\&+ \sqrt{\log 2}(\xi C\mathbf{E}_{(\theta,\theta')\sim \pi(\theta)q(\theta'|\theta)}[M(\theta,\theta')] + \eta)\chi^{-\frac{3}{2}} 
    \\&= \xi C^2\mathbf{E}_{(\theta,\theta')\sim \pi(\theta)q(\theta'|\theta)}[M^2(\theta,\theta')].
\end{align*}

When $\chi$ is small, the LHS is approximately $(\xi C\mathbf{E}_{(\theta,\theta')\sim \pi(\theta)q(\theta'|\theta)}[M(\theta,\theta')] + \eta)\chi^{-2}$ which gives us  
\[
\chi = \sqrt{\frac{\xi C\mathbf{E}_{(\theta,\theta')\sim \pi(\theta)q(\theta'|\theta)}[M(\theta,\theta')] + \eta}{\xi C^2\mathbf{E}_{(\theta,\theta')\sim \pi(\theta)q(\theta'|\theta)}[M^2(\theta,\theta')]}}.
\]

When it is quick to get a proposal ($\eta\approx 0$) and the variance of $M$ is small, we can further simplify it to
\[
\chi = \frac{1}{ \sqrt{C\mathbf{E}_{(\theta,\theta')\sim \pi(\theta)q(\theta'|\theta)}[M(\theta,\theta')]}}.
\]

In practice, we can get the above theoretically optimal value of $\chi$ by empirically estimating the mean and variance of $M(\theta,\theta')$. Note that even if these empirical estimates are accurate, there may exist better $\chi$, since the upper bounds (the mixing time bound and the spectral gap bound) we use to get the optimal value may be loose. We give a simpler heuristic to tune $\chi$ in practice in Section \ref{sec:exp}.

\section{Proof of Theorem \ref{thm:optimality}}\label{app:proof:optimality}

First, we will show the following lemma, which gives half of what we want to have in the theorem.

\begin{lemma}\label{lemma:optimality1}
Considering the same setting as the theorem, the average batch size $B$ of any exact, stateless minibatch MH algorithm at any iteration follows
\[
    \mathbf{E}[B] \ge 2^{-18} \cdot \kappa C^2 M^2(\theta,\theta') - 2^{-4} \cdot \kappa.
\]
\end{lemma}
\begin{proof}
We prove the lemma by construction. First, observe that since the state space $\Theta$ has at least two states, we can restrict our attention to just two of those states, by choosing a $\pi$ that has zero mass on any other state in the space and a $q$ that never proposes transitioning out to any of those other states (at which $\pi$ has zero mass). Such a proposal will still be ergodic, so it still satisfies our general assumption that we consider only ergodic chains in this paper. Without loss of generality, suppose that those two states are $\{-\frac{M}{2}, \frac{M}{2}\}$ (this is without loss of generality because we can always just rename the states), and let $C$ denote the constant in the theorem statement and define (with a bit of abuse of notation) the constant $M := M(-\frac{M}{2}, \frac{M}{2})$. By doing this, we can (again without loss of generality) restrict our attention to the case where $\Theta = \{-\frac{M}{2}, \frac{M}{2}\}$.

Next, we construct our counterexample. Let the dataset be $\{x_i\}_{i=1}^N$ where $x_i \in \{-1, 1\}$.
We let the domain for parameter $\theta$ to be $\{-\frac{M}{2}, \frac{M}{2}\}$, and the target distribution to be   
\[
\pi(\theta) = \frac{1}{Z}\exp\left(-\sum_{i=1}^N U_i(\theta)\right) = \frac{1}{Z}\exp\left(-\frac{C\theta}{N}\sum_{i=1}^N x_i\right)
\]
where $U_i(\theta) = \frac{C}{N}\cdot\theta x_i$.
Note that by letting $N$ become large, any minibatch MH algorithm that queries the energy difference oracle some number of times will observe a distribution of energy differences that is arbitrarily close to a sequence of independent identically distributed random variables supported on $\{\pm \frac{CM}{N}\}$.

We define $c_i = \frac{C}{N}$, and the proposal distribution to be
\[
p(\theta,\theta) = \frac{1}{2},\hspace{2em} p(\theta,-\theta) = \frac{1}{2} \hspace{2em}\text{for }\theta \in \bigg\{-\frac{M}{2}, \frac{M}{2}\bigg\}.
\]
Now, let $0 < q < 1$ be some constant, and consider two cases: (1) $\frac{1}{N}\sum_i x_i = q$ and (2) $\frac{1}{N}\sum_i x_i = -q<0$. Suppose that in both cases the $x_i$ are shuffled at random. These two cases will have different stationary distributions,
\[
    \pi_1(\theta) = \frac{1}{Z} \exp\left(-C q \theta\right)
    \hspace{2em}\text{and}\hspace{2em}
    \pi_2(\theta) = \frac{1}{Z} \exp\left(C q \theta\right),
\]
and an exact algorithm must be able to distinguish between them. Therefore by using these cases, we can get a bound on the required batch size needed for the exact MH algorithm to distinguish between them.
First, we observe that the two cases are symmetric, such that if $T_1$ is the transition matrix of the chain in case (1) and $T_2$ is the transition matrix of the chain in case (2), then $T_1(\theta, \theta') = T_2(\theta', \theta)$.
Let $0 < \psi < \frac{1}{2}$ denote the probability that $T_1$ transitions from $\frac{M}{2}$ to $-\frac{M}{2}$.
Then because the MH method is exact and the chain is reversible, the probability of the reverse transition is $\psi \exp(-CMq)$.
So, explicitly, the transition operators will look like
\[
    T_1 = \begin{bmatrix} 1 - \psi & \psi e^{-CMq} \\ \psi & 1 - \psi e^{-CMq} \end{bmatrix}
    \hspace{2em}\text{and}\hspace{2em}
    T_2 = \begin{bmatrix} 1 - \psi e^{-CMq} & \psi \\ \psi e^{-CMq} & 1 - \psi \end{bmatrix}.
\]
The eigenvectors and eigenvalues of this are
\[
    T_1 \pi_1
    =
    \pi_1
    \hspace{2em}\text{and}\hspace{2em}
    T_1 \begin{bmatrix} -1 \\ 1 \end{bmatrix}
    =
    \left(1 - \psi - \psi \exp(-CMq) \right)
    \begin{bmatrix} -1 \\ 1 \end{bmatrix}.
\]
Suppose that we initialize both chains uniformly on $\{-\frac{M}{2}, \frac{M}{2}\}$.
Observe that
\[
    \begin{bmatrix} 1/2 \\ 1/2 \end{bmatrix}
    =
    \begin{bmatrix} \frac{\exp(-CMq)}{1 + \exp(-CMq)} \\ \frac{1}{1 + \exp(-CMq)} \end{bmatrix}
    +
    \frac{1 - \exp(-CMq)}{2 (1 + \exp(-CMq))}
    \cdot
    \begin{bmatrix} 1 \\ -1 \end{bmatrix},
\]
the first vector being $\pi_1$ and the second being a multiple of the other eigenvector.
Equivalently,
\[
    \begin{bmatrix} 1/2 \\ 1/2 \end{bmatrix}
    =
    \pi_1
    +
    \frac{1}{2}
    \tanh\left( \frac{CMq}{2} \right)
    \cdot
    \begin{bmatrix} 1 \\ -1 \end{bmatrix},
\]
and so for any $t$, after $t$ steps of the Markov chain, the distribution will be
\[
    T_1^t
    \begin{bmatrix} 1/2 \\ 1/2 \end{bmatrix}
    =
    \pi_1
    +
    \frac{1}{2}
    \tanh\left( \frac{CMq}{2} \right)
    \cdot
    \left(1 - \psi - \psi \exp(-CMq) \right)^t
    \cdot
    \begin{bmatrix} 1 \\ -1 \end{bmatrix}.
\]
Similarly,
\[
    T_2^t
    \begin{bmatrix} 1/2 \\ 1/2 \end{bmatrix}
    =
    \pi_2
    +
    \frac{1}{2}
    \tanh\left( \frac{CMq}{2} \right)
    \cdot
    \left(1 - \psi - \psi \exp(-CMq) \right)^t
    \cdot
    \begin{bmatrix} -1 \\ 1 \end{bmatrix}.
\]
So, the total variation distance between the state of the chains at time $t$ will be bounded by
\[
    \text{TV}\left( 
    T_1^t
    \begin{bmatrix} 1/2 \\ 1/2 \end{bmatrix},
    T_2^t
    \begin{bmatrix} 1/2 \\ 1/2 \end{bmatrix}
    \right)
    \ge
    \text{TV}\left( \pi_1, \pi_2 \right)
    -
    \tanh\left( \frac{CMq}{2} \right)
    \cdot
    \left(1 - \psi - \psi \exp(-CMq) \right)^t.
\]
Also observe that
\[
    \text{TV}\left( \pi_1, \pi_2 \right)
    =
    \frac{1}{2}
    \left\|
    \begin{bmatrix} \frac{\exp(-CMq)}{1 + \exp(-CMq)} \\ \frac{1}{1 + \exp(-CMq)} \end{bmatrix}
    -
    \begin{bmatrix} \frac{1}{1 + \exp(-CMq)} \\ \frac{\exp(-CMq)}{1 + \exp(-CMq)} \end{bmatrix} \right\|_1
    =
    \frac{1 - \exp(-CMq)}{1 + \exp(-CMq)}
    =
    \tanh\left( \frac{CMq}{2} \right),
\]
so
\[
    \text{TV}\left( 
    T_1^t
    \begin{bmatrix} 1/2 \\ 1/2 \end{bmatrix},
    T_2^t
    \begin{bmatrix} 1/2 \\ 1/2 \end{bmatrix}
    \right)
    \ge
    \tanh\left( \frac{CMq}{2} \right)
    \cdot
    \left(
    1
    -
    \left(1 - \psi - \psi \exp(-CMq) \right)^t
    \right).
\]
Also, since we know that our algorithm is guaranteed to have spectral gap ratio at least $\kappa$ with the original chain, it follows that $\psi \ge \kappa / 2$, and so
\[
    \text{TV}\left( 
    T_1^t
    \begin{bmatrix} 1/2 \\ 1/2 \end{bmatrix},
    T_2^t
    \begin{bmatrix} 1/2 \\ 1/2 \end{bmatrix}
    \right)
    \ge
    \tanh\left( \frac{CMq}{2} \right)
    \cdot
    \left(
    1
    -
    \left(1 - \frac{\kappa}{2} - \frac{\kappa}{2} \exp(-CMq) \right)^t
    \right).
\]

Now, denote the exact minibatch algorithm to be $\mathcal{A}$. As it runs, the algorithm $\mathcal{A}$ will request data examples by querying the energy difference oracle.
Under case (1), we let $y_i$ denote the $i$th sample that $\mathcal{A}$ \emph{would have observed} if it requested $i$ or more samples, and similarly we let $z_i$ denote the analogous sample in case (2).
Fix some constant $t \in \mathbf{N}$ (which we will set later).
We let $K_1$ denote the total number of samples observed by $\mathcal{A}$ across the first $t$ iterations in case (1), and set
\[
\mu = \{y_1, y_2,\dots, y_{K_1}\}.
\]
Similarly, we let $K_2$ denote the number of samples observed by $\mathcal{A}$ across the first $t$ iterations in case~(2), and set
\[
\nu = \{z_1, z_2, \dots, z_{K_2}\}.
\]
Now, we fix some constant $K$ (to be set later), and consider the following coupling between the behavior of $\mathcal{A}$ across its first $t$ iterations in case (1) and in case~(2). First, let all internal randomness of $\mathcal{A}$ and the proposal process under case (1) and (2) be the same, which means that for a given observation of data examples, the algorithm $\mathcal{A}$ will make the same decision, such as whether to require more data examples or not and whether to accept or not.
Second, choose a coupling that minimizes the probability that
\[
    (y_1, y_2, \ldots, y_{K1}) \ne (z_1, z_2, \ldots, z_{K2}).
\]
Such a coupling is guaranteed to exist by the Coupling Lemma, and the probability that these two are not equal will be equal to the total variation distance between their distributions.
Third, assign all the other $y_i$ and $z_i$, for $i > K$, independently according to their distribution.

We are interested in the quantity $p(\mu \ne \nu)$, which bounds the probability that the algorithm may make a different decision in cases (1) and (2).
We can decompose this probability into two terms,
\[
p(\mu \neq \nu) = p(\mu \neq \nu \text{ and } y_j = z_j \text{ for all } j\le K) + p(\mu \neq \nu \text{ and } y_j \neq z_j \text{ for some } j \le K ).
\]
If $\mu \ne \nu$ but $y_j = z_j$ for all $j \le K$, the only way that this is possible is for $K_1 > K$ (and, symmetrically, also $K_2 > K$), since otherwise the algorithms would behave identically.
So,
\begin{equation}\label{eq:prob_neq}
p(\mu \neq \nu) \le p(K_1 > K) + p(y_j \neq z_j \text{ for some } j \le K ).
\end{equation}
By Markov's inequality,
\[
p(\mu \neq \nu) \le \frac{\mathbf{E}[K_1]}{K} + p(y_j \neq z_j \text{ for some } j \le K ).
\]
For the second term of (\ref{eq:prob_neq}), we can reduce the case to only considering $K$ samples.
Let $S_y$ be the total number of samples $y_i$ that are $-1$ and let $S_z$ be the total number of samples $z_i$ that are $-1$.
Since $\mathcal{A}$ is effectively sampling a shuffled dataset at some arbitrary indices without replacement, both of these random variables $S_y$ and $S_z$ are---properly speaking---hypergeometric random variables. However, since our dataset size $N$ is arbitrary here, we can by setting $N$ very large work in the limit (as $N \rightarrow \infty$) in which these variables become binomial (since sampling with replacement and without replacement can be made to have arbitrarily close to the same distribution by making the dataset large).
Observe that (in this limit) $S_y$ follows a binomial distribution $B(K, \frac{1 - q}{2})$ and $S_z$ follows a binomial distribution $B(K, \frac{1 + q}{2})$.
Clearly, if $S_y = S_z$, then we can arrange the coupling so that $(y_1, \ldots, y_K) = (z_1, \ldots, z_K)$.
So, by the Coupling Lemma,
\begin{align*}
    p( y_j\neq z_j \text{ for some } j \le K) = p(S_y \neq S_z) = \text{TV}(S_y, S_z).
\end{align*}

From the analysis in \citet{adell2006exact}, we can bound the total variance distance between these two binomial variables with
\begin{align*}
    \text{TV}(S_y, S_z) \le \sqrt{e} \cdot\frac{\tau}{(1-\tau)^2}
\end{align*}
where $\tau = \sqrt{\frac{K+2}{2}}\cdot q<1$.
Substituting these bounds, we get
\begin{align*}
    p(\mu \neq \nu) &\le \frac{\mathbf{E}[K_1]}{K}
    + \sqrt{e} \cdot\frac{\tau}{(1-\tau)^2}.
\end{align*}
But the probability that $\mu \neq \nu$ must be an upper bound on the probability that the distributions of the chains in case (1) and (2) after $t$ steps are not equal, since if $\mu = \nu$ in the coupling then the two chains are in the same state.
So, using our bound from earlier, we get
\[
    \tanh\left( \frac{CMq}{2} \right)
    \cdot
    \left(
    1
    -
    \left(1 - \frac{1}{2} \kappa - \frac{1}{2} \kappa \exp(-CMq) \right)^t
    \right)
    \le
    \frac{\mathbf{E}[K_1]}{K}
    + \sqrt{e} \cdot\frac{\tau}{(1-\tau)^2}.
\]
Now isolating $\mathbf{E}[K_1]$ gives
\[
    K
    \cdot
    \tanh\left( \frac{CMq}{2} \right)
    \cdot
    \left(
    1
    -
    \left(1 - \frac{1}{2} \kappa - \frac{1}{2} \kappa \exp(-CMq) \right)^t
    \right)
    -
    K \cdot \sqrt{e} \cdot\frac{\tau}{(1-\tau)^2}
    \le
    \mathbf{E}[K_1].
\]
Also, observe that
\begin{align*}
    \left(1 - \frac{1}{2} \kappa - \frac{1}{2} \kappa \exp(-CMq) \right)^t
    \le
    \left(1 - \frac{1}{2} \kappa \right)^t
    \le
    \exp\left(-\frac{\kappa t}{2} \right),
\end{align*}
so
\[
    K
    \cdot
    \tanh\left( \frac{CMq}{2} \right)
    \cdot
    \left(
    1
    -
    \exp\left(-\frac{\kappa t}{2} \right)
    \right)
    -
    K \cdot \sqrt{e} \cdot\frac{\tau}{(1-\tau)^2}
    \le
    \mathbf{E}[K_1].
\]
This gives us the lower bound on $\mathbf{E}[K_1]$ that we are interested in.
Now, it remains to assign $q$, $K$, and $t$.
We start by assigning $t$ such that
\[
    t = \left\lceil 2 \kappa^{-1} \log(2) \right\rceil,
\]
in which case
\[
    \exp\left(-\frac{\kappa t}{2} \right) \le \frac{1}{2}
\]
and so
\[
    K
    \cdot
    \frac{1}{2}
    \cdot
    \tanh\left( \frac{CMq}{2} \right)
    -
    K \cdot \sqrt{e} \cdot\frac{\tau}{(1-\tau)^2}
    \le
    \mathbf{E}[K_1].
\]
Now, we add some simplifying assumptions, which we will validate are true later.
We assume that 
\[
    \tau = \sqrt{\frac{K+2}{2}} \cdot q \le \frac{1}{2};
\]
in this case
\[
    \sqrt{e} \cdot \frac{\tau}{(1-\tau)^2} \cdot K
    \le
    4 \sqrt{e} \cdot \tau
    \le
    5 \sqrt{K+2} \cdot q.
\]
We set $q$ such that
\[
    C M q = 1,
\]
and we assume that $CM$ is large enough that this assignment of $q$ is within range (i.e. $0 < q < 1$).
This gives us
\[
    K
    \cdot
    \frac{1}{2}
    \cdot
    \tanh\left( \frac{1}{2} \right)
    -
    5 K \sqrt{K+2} \cdot \frac{1}{CM}
    \le
    \mathbf{E}[K_1].
\]
Since $\tanh(1/2) > 5/16$, we can simplify this to
\[
    K
    \cdot
    \frac{5}{32}
    -
    5 K \sqrt{K+2} \cdot \frac{1}{CM}
    \le
    \mathbf{E}[K_1].
\]
All that remains is to assign $K$.
We assign $K$ such that
\[
    \sqrt{K+2} \cdot \frac{1}{CM}
    =
    \frac{1}{64}.
\]
In this case, we get
\[
    K = \frac{C^2 M^2}{4096} - 2,
\]
and our bound reduces to
\[
    \left( \frac{C^2 M^2}{4096} - 2 \right)
    \cdot
    \frac{5}{64}
    \le
    \mathbf{E}[K_1].
\]
We can simplify this further to
\[
    2^{-16} \cdot C^2 M^2 - \frac{5}{32}
    \le
    \mathbf{E}[K_1].
\]
Now, this is a bound on the expected number of samples taken across $t$ iterations.
This means that the number of samples taken in any given iteration will be bounded by
\[
    \frac{\mathbf{E}[K_1]}{t}
    \ge
    \frac{
        2^{-16} \cdot C^2 M^2 - \frac{5}{32}
    }{
        2 \kappa^{-1} \log(2) + 1
    }
    =
    \frac{
        2^{-16} \cdot \kappa C^2 M^2 - \frac{5 \kappa}{32}
    }{
        2 \log(2) + \kappa
    }.
\]
A few more loose bounds, leveraging $\kappa < 1$, gives us
\[
    \frac{\mathbf{E}[K_1]}{t}
    \ge
    2^{-18} \cdot \kappa C^2 M^2 - \frac{\kappa}{16}.
\]
This proves the lemma.
\end{proof}

Next, we will show the following lemma, which characterizes what happens when $C M$ is small.

\begin{lemma}
Considering minibatch MH algorithms in the same setting as the theorem, the expected batch size at any iteration must be lower bounded by
\[
    \mathbf{E}[B] \ge \frac{\kappa}{2} \min\left( C M(\theta, \theta'), 1 \right).
\]
\end{lemma}
\begin{proof}
Here, we will prove a lower bound that characterizes the limits of exact stateless minibatch MH algorithms when they use very few examples.
Again, without loss of generality we consider a reduction to the two-state case as we did in the proof of the previous lemma.
Suppose that a exact stateless minibatch MH algorithm with the same forward and backward proposal probabilities (given some $c_1,\ldots,c_N$, $C$, and $M$) requests any energy function examples at all only with probability $p$.
Consider two cases, which have the same $c_1,\ldots,c_N$, $C$ and $M$.
In the first case,
\[
    \sum_{i=1}^n (U_i(\theta) - U_i(\theta')) = C M(\theta, \theta'),
\]
while in the second case,
\[
    \sum_{i=1}^n (U_i(\theta) - U_i(\theta')) = -C M(\theta, \theta').
\]
These are clearly possible by setting $U_i$ to the limits of what is covered by the bounds.
In the first case, the baseline MH method would accept with probability $1$.
In the second case, it will accept with probability $\exp(-C M(\theta,\theta'))$.
Since the stateless MH algorithm is reversible, it must accept in the first case with some probability $a$ and in the second case with probability $a \cdot \exp(-C M(\theta,\theta'))$.
But, the algorithm can only distinguish the two cases if it requests samples, which only happens with probability at most $p$. So,
\[
    a - a \cdot \exp(-C M(\theta,\theta') \le p.
\]
Since we know that it must be the case that $a \ge \kappa$ (from a straightforward analysis of a two-state case), it follows that
\[
    \frac{p}{\kappa} \ge \frac{p}{a} \ge 1 - \exp(-C M(\theta, \theta')) \ge \frac{1}{2} \min\left( C M(\theta, \theta'), 1 \right).
\]
Since $p$ is an obvious lower bound on the expected value of the batch size, it follows that
\[
    \mathbf{E}[B] \ge \frac{\kappa}{2} \min\left( C M(\theta, \theta'), 1 \right).
\]
\end{proof}

To prove Theorem~\ref{thm:optimality} we now combine the results of these two lemmas.
We have
\[
    \mathbf{E}[B] \ge 2^{-18} \cdot \kappa C^2 M^2(\theta,\theta') - 2^{-4} \cdot \kappa.
\]
and
\[
    \mathbf{E}[B] \ge \frac{\kappa}{2} \min\left( C M(\theta, \theta'), 1 \right).
\]
Since these are both lower bounds, we can combine them to get
\begin{align*}
    \mathbf{E}[B] 
    &\ge 
    \max\left(2^{-18} \cdot \kappa C^2 M^2(\theta,\theta') - 2^{-4} \cdot \kappa, \frac{\kappa}{2} \min\left( C M(\theta, \theta'), 1 \right) \right) \\
    &=
    \kappa \cdot \max\left(2^{-18} \cdot C^2 M^2(\theta,\theta') - 2^{-4}, \frac{1}{2} \min\left( C M(\theta, \theta'), 1 \right) \right).
\end{align*}
It is obvious from a simple big-$\mathcal{O}$ analysis here that there exists a global constant $\zeta > 0$ such that
\[
    \mathbf{E}[B] \ge \zeta \cdot \kappa \left( C^2 M^2(\theta, \theta') + C M(\theta, \theta') \right).
\]
This proves the theorem.

\section{Proof of Corollary \ref{col:bound}} \label{app:proof:cor1}
\begin{proof}
Recall that the lower bound on the batch size in each iteration is
\[
\mathbf{E}[B] \ge \zeta \cdot \kappa \left( C^2 M^2(\theta, \theta') + C M(\theta, \theta') \right).
\]
Since $C = \bigTheta(N)$ and $M(\theta,\theta') = \bigTheta(N^{-(h+1)/2})$, the expectation of the batch size follows
\[
\mathbf{E}[B] = \bigTheta(C^2 M^2(\theta, \theta') + C M(\theta, \theta')) = \bigTheta(C M(\theta, \theta')) = \bigTheta(N^{1-h}/2).
\]

When $h = 1$, $\mathbf{E}[B] = \bigTheta(1)$ and when $h=2$, $\mathbf{E}[B] = \bigTheta(1/\sqrt{N})$.
\end{proof}

\section{Experimental Details and Additional Results}\label{app:experiments}
\subsection{Experiment in Section \ref{sec:inexactproblems}}\label{app:experiments:counterexample}
To verify Theorem \ref{statement:counterexample}, we empirically construct a distribution in the form of Section \ref{app:proof:counterexample} such that AustereMH and MHminibatch are biased on. Note that the proof in Section \ref{app:proof:counterexample} shows there must exist such a distribution for any inexact minibatch method but does not tell us how to find one for a specific method. Therefore, in order to find such a distribution, we construct an example and empirically test whether AustereMH and MHminibatch are biased on it. 

We let data $x_i$ take one of two values $\{-1, 5\}$. Consider a dataset of size 6000. We let 5000 data take value $-1$ and the remaining 1000 data take value $5$. Define the target distribution $\pi(\theta)$ to be
\[
\pi(\theta) \propto \exp\left(-\frac{1}{N}\sum_{i=1}^N \theta\cdot x_i\right)
\]
where the domain of $\theta$ is $\{ 0, 1, \dots, K-1\}$. Therefore the number of state is $K$. Since $\sum_i x_i = 0$, it is clear to see that the stationary distribution of $\theta$ is a uniform distribution. We define the proposal distribution to be the following
\[
p(\theta,\theta) = \frac{1}{2},\hspace{1em}\text{for all }\theta;\hspace{1em} p(\theta, \theta-1) = \frac{1}{4},\hspace{1em} p(\theta, \theta+1) = \frac{1}{4} \hspace{1em}\text{for }\theta\in\{1,\dots,K-2\}; 
\] 
and $p(0,1)=p(K-1,K-2)=\frac{1}{2}$.

We set the hyperparameter error $\epsilon$ in AustereMH to be 0.01 and $\delta$ in MHminibatch to be 5, following the setting in their original papers \cite{korattikara2014austerity, seita2016efficient}. We set batch size $m$ in both methods to be 30. We find that AustereMH and MHminibatch are both inexact on this example and the error increases as we increase $K$. Thus we empirically verify the statement in Theorem \ref{statement:counterexample}. 

Besides the density estimate comparison on $K=200$ shown in Figure \ref{fig:counter-example}b, we additionally report the estimate results on other values of $K$ in Figure \ref{app:fig:density}. We see that the results are similar, all showing that \methodname{} and standard MH can give accurate estimate whereas inexact methods are seriously wrong. 

\paragraph{On Robust Linear Regression} We further tested AustereMH on robust linear regression in Section~\ref{sec:rlr} with $N=5000$. We computed the MSE between estimated and true parameters. MH, TunaMH and AustereMH obtained MSE 0.149, 0.15 and 1.19 respectively, indicating inexact method error can be large on typical problems.

\begin{figure*}[t!]
    \centering
    \begin{tabular}{cccc}		
    	\includegraphics[width=4.5cm]{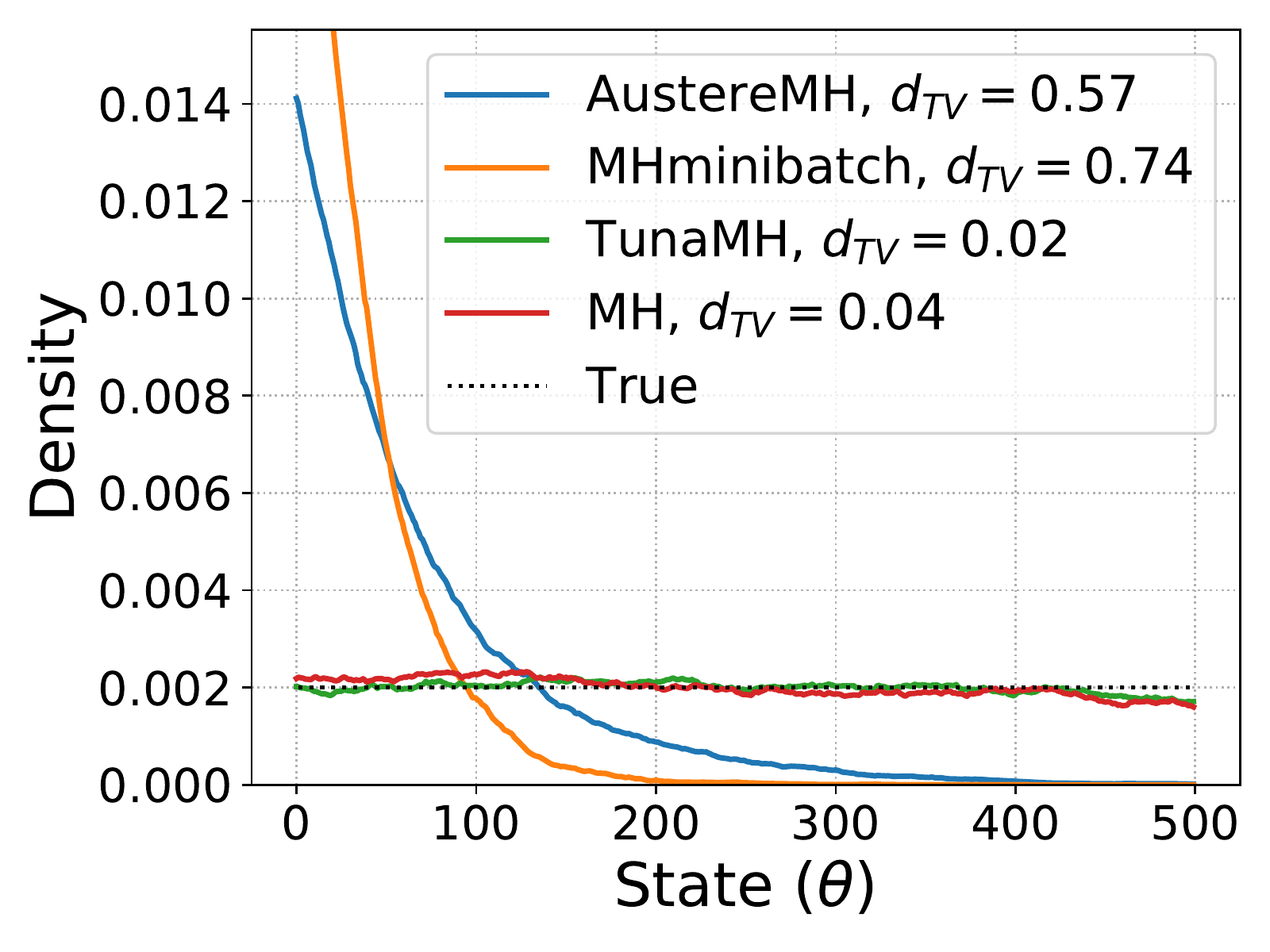}  &
    	\includegraphics[width=4.5cm]{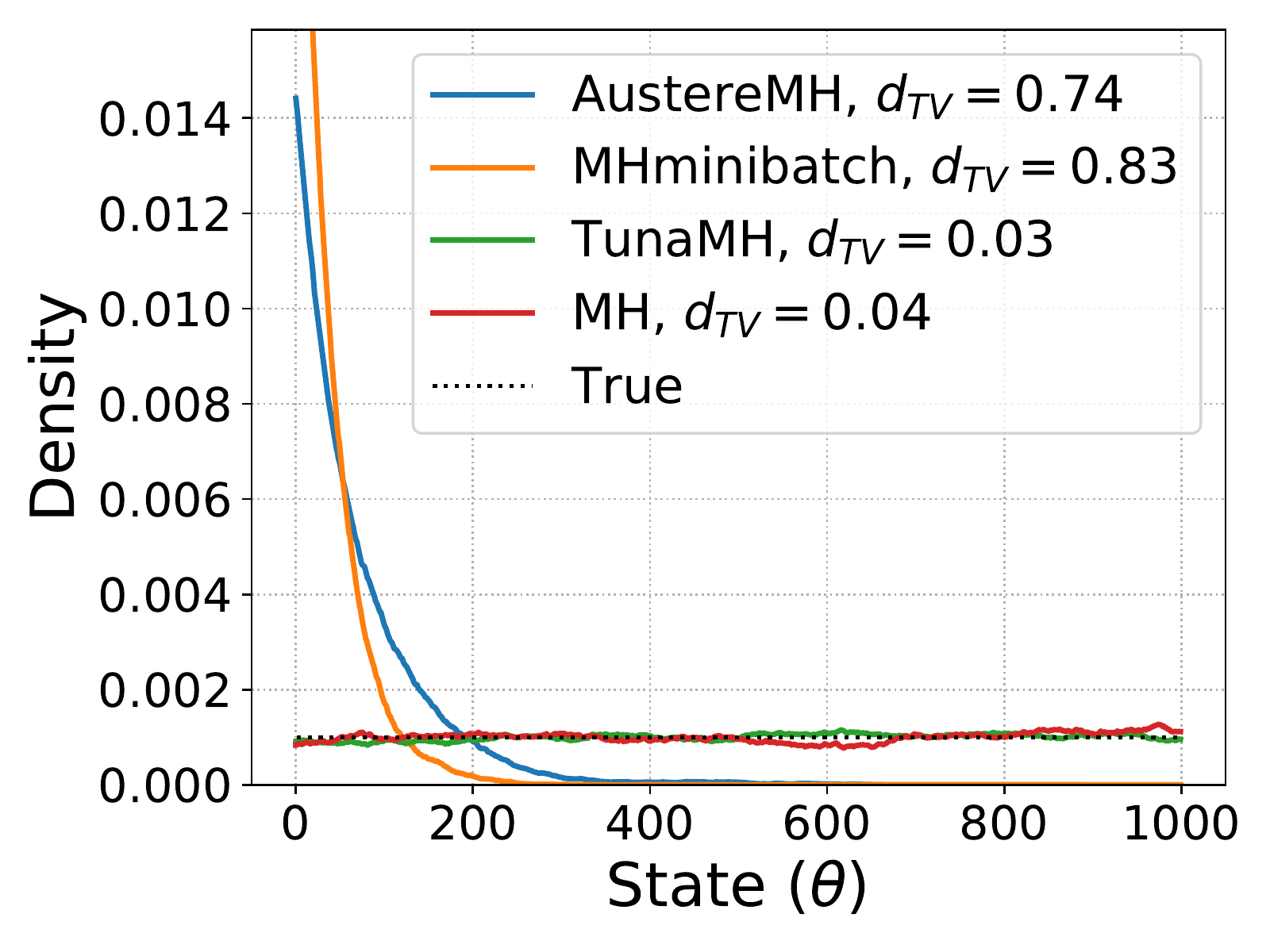}
    	\\		
    	(a) $K=500$&
    	(b) $K=1000$
    	\\
    	\includegraphics[width=4.5cm]{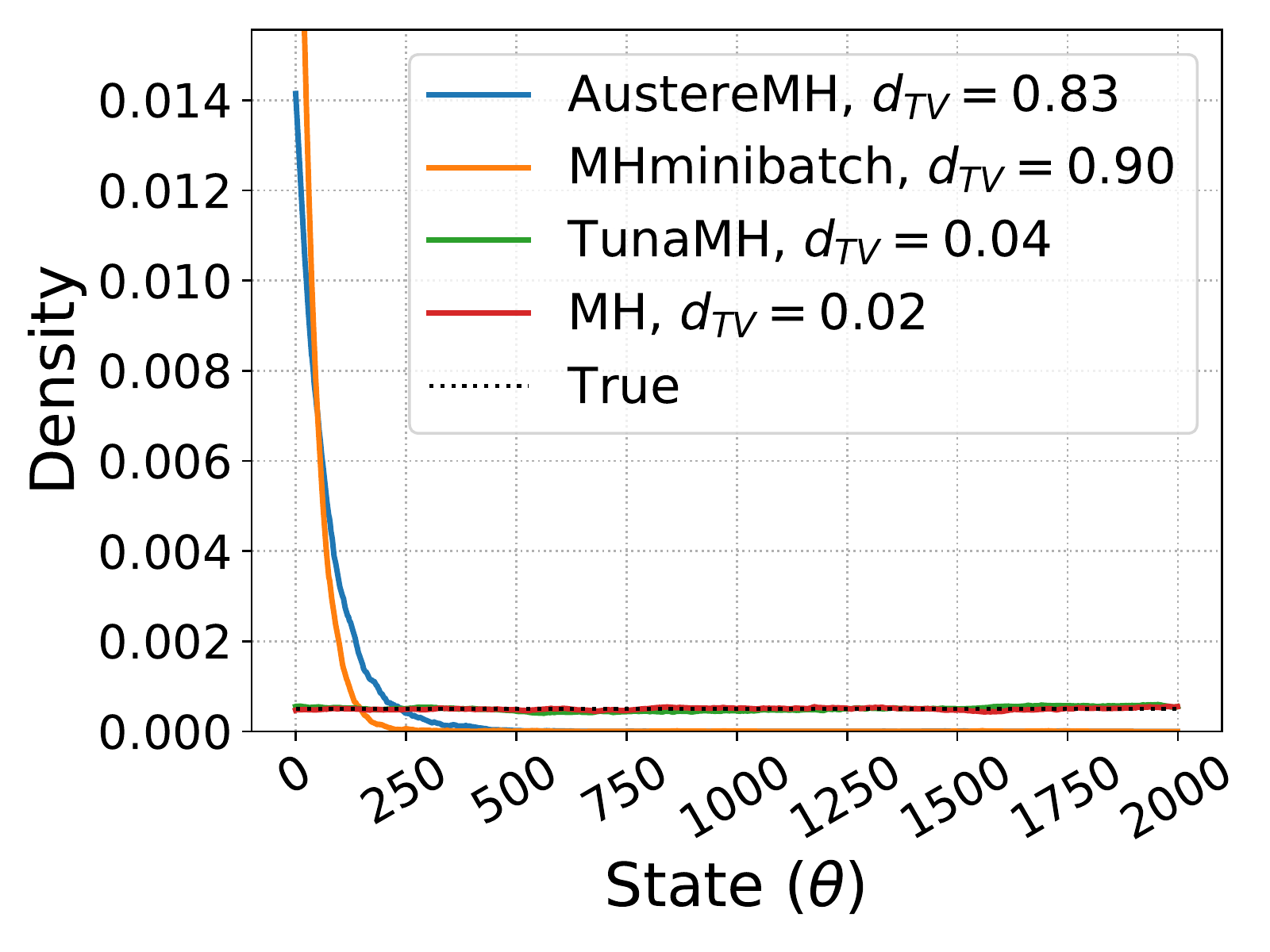}
    	&
    	\includegraphics[width=4.5cm]{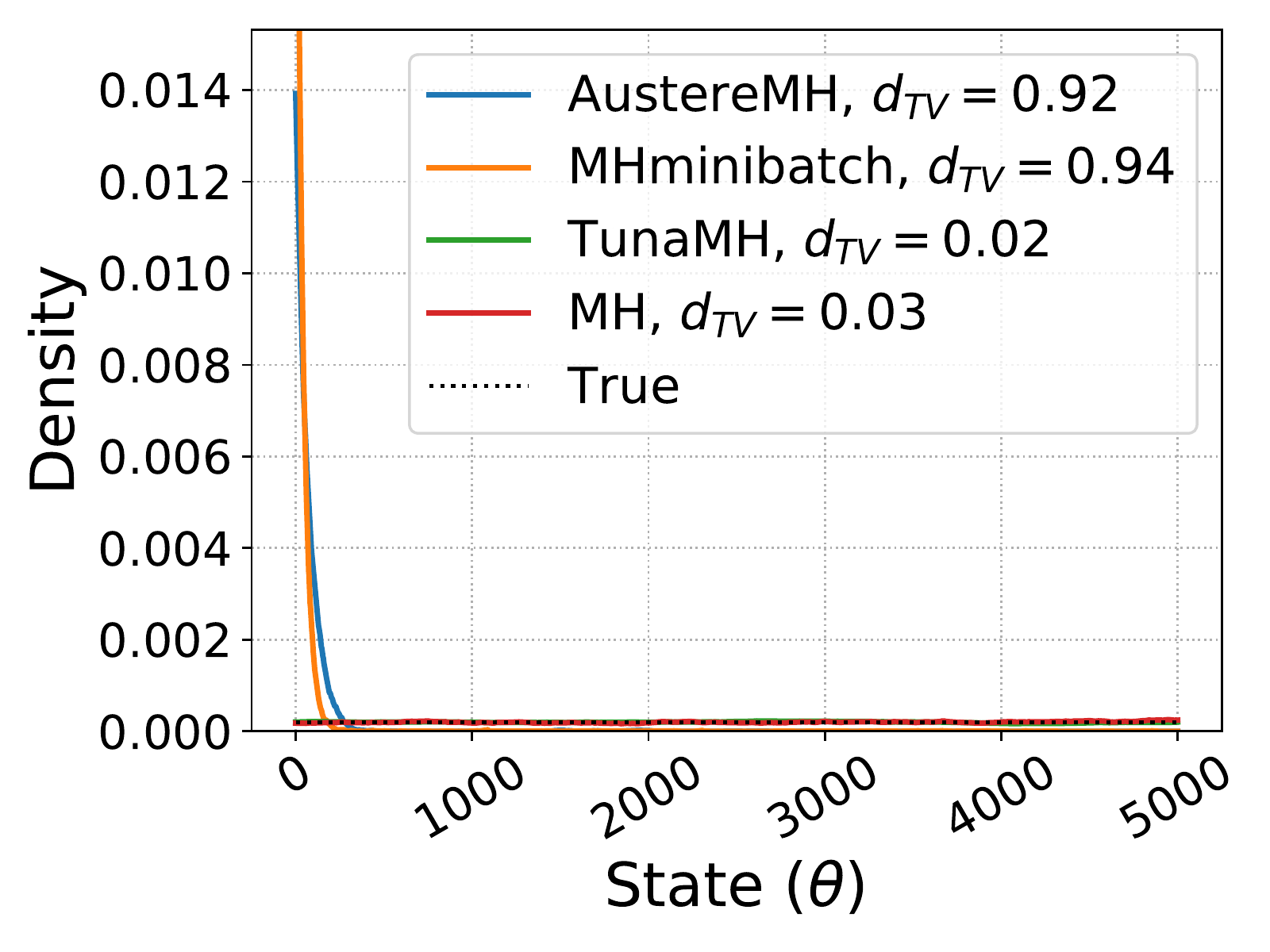}
    	
    	\\	
    	(e) $K=2000$&
    	(e) $K=5000$&
    	\hspace{-0mm}\\		
    \end{tabular}
    \caption{Density estimate comparison on $K = 500, 1000, 2000, 5000$.}
    \label{app:fig:density}
\end{figure*}

\subsection{Robust Linear Regression}\label{app:experiments:rlr}
We follow the experimental setup of robust linear regression (RLR) in \citet{cornish2019scalable}. Specifically, we have data $x_i\in \mathbb{R}^d$ and $y_i\in \mathbb{R}$. The likelihood is modeled by a student's t-distribution with degrees of freedom $v$:
\[
p(y_i|\theta,x_i) = \text{Student}(y_i - \theta^\intercal x_i|v).
\]
It follows that 
\[
U_i(\theta) = \frac{v+1}{2}\log\left(1 + \frac{(y_i-\theta^\intercal x_i)^2}{v}\right),
\]
and the first derivative
\[
\partial_j U_i(\theta) = -(v+1)\frac{x_{ij}(y_i-\theta^\intercal x_i)}{v + (y_i-\theta^\intercal x_i)^2}.
\]
Since the function $U_i$ is Lipschitz continuous, we can easily get the bound used in \methodname{}, TFMH and SMH. We set $M(\theta,\theta') = \norm{\theta - \theta'}_2$ and then it follows
\[
c_i = \sup_{\theta\in\mathbb{R}} \norm{\nabla U_i(\theta)}_2 = \frac{v+1}{2\sqrt{v}}\norm{x_i}_2.
\]
The data $x_i$ and $y_i$ is generated as follows
\[
y_i = \sum_j x_{ij} + \epsilon_i
\]
where $\epsilon_i\sim\mathcal{N}(0,1)$.

In Section \ref{sec:rlr}, we set $v=4$, $d=100$ and use a flat prior $p(\theta) = 1$. Note that our problem dimension $d$ is much larger than that in the SMH paper \cite{cornish2019scalable} ($d = 10$). This makes the control variates in SMH problematic since the bounds they require appear to scale badly in high dimensions. 

To reach the target acceptance rate, we set the stepsize in each method as in Table \ref{tab:stepsize} and \ref{tab:stepsize2}. For \methodname{} and \methodname{}-MAP, we set $\chi = 1e-5$ for $N=5000, 20000$ and $\chi=1e-4$ for $N=50000, 100000$. For FlyMC and FlyMC-MAP, we set the probability for a data going from dark to bright $q_{d\rightarrow b}$ to be 0.01. Without the MAP, we collect 80000 samples after 200000 step burnin. With the MAP, we collect 80000 samples without burnin.

\begin{table}[H]
  \caption{Stepsize of methods without the MAP.}
  \label{tab:stepsize}
  \centering
  \begin{tabular}{lcccccc}
    \toprule                 \\
       & MH & TFMH & FlyMC & \methodname{}\\
    \midrule
    \vspace{.1cm}
    RLR $N=5000$ & 4e-3 &1e-4 &2.7e-3 &8e-4, $\chi = 1e-5$ \\
    \vspace{.1cm}
    RLR $N=20000$ & 2e-3 &3e-5 &1.5e-3 &3e-4, $\chi = 1e-5$  \\
    \vspace{.1cm}
    RLR $N=50000$ & 1.3e-3 &1.2e-5 &9e-4 &2e-4, $\chi = 1e-4$  \\
    \vspace{.1cm}
    RLR $N=100000$ & 9e-4 &6e-6 &7e-4 &1.7e-4, $\chi = 1e-4$ \\
    TGM & 3e-1  &2.2e-2  &1e-2   &1e-1\\
    LR &5e-3 &1e-4  &2e-3   &1e-3 \\
    \bottomrule
  \end{tabular}
\end{table}
\begin{table}[H]
  \caption{Stepsize of methods with the MAP.}
  \label{tab:stepsize2}
  \centering
  \begin{tabular}{lccccccc}
    \toprule                 \\
       & MH-MAP & SMH-1 &SMH-2 & FlyMC-MAP & \methodname{}-MAP \\
    \midrule
    \vspace{.1cm}
    RLR $N=5000$ & 4e-3 & 4e-3  &4e-3   &  6e-3 &8e-4, $\chi = 1e-5$ \\
    \vspace{.1cm}
    RLR $N=20000$ & 2e-3 & 2e-3  & 2e-3  & 3.5e-3  &3e-4, $\chi = 1e-5$ \\
    \vspace{.1cm}
    RLR $N=50000$ & 1.2e-3 & 1.2e-3  & 1.2e-3  & 2.5e-3  &1.2e-4, $\chi = 1e-4$ \\
    \vspace{.1cm}
    RLR $N=100000$ &  9e-4 & 5.9e-4  & 8e-4  & 1.7e-3  &7e-5. $\chi = 1e-4$ \\
    TGM & -  &1e-1   &  -    &1e-2  & -\\
    \bottomrule
  \end{tabular}
\end{table}

\subsubsection{Additional Experimental Results with $d=10$}

We ran RLR experiment with $d=10$ and $N=10^5$ to compare the performance in low dimensions. The ESS/S for TFMH, FlyMC, TunaMH are 0.02, 0.75, \& 1.7, respectively; SMH-1, SMH-2, FlyMC-MAP and TunaMH-MAP are 174.7, 5969.5, 730.8, \& 730.1 respectively. This suggests TunaMH is significantly better without MAP/control variates. With MAP/control variates, TunaMH is better than SMH-1, similar to FlyMC and worse than SMH-2.

\subsection{Truncated Gaussian Mixture} \label{app:experiments:mog}
The data in this truncated Gaussian mixture (TGM) task is generated as follows 
\[
x_i \sim \frac{1}{2}\mathcal{N}(\theta_1, \sigma_x^2) + \frac{1}{2}\mathcal{N}(\theta_1+\theta_2, \sigma_x^2)
\]
where $\theta_1 = 0, \theta_2 = 1$ and $\sigma^2 = 2$. The posterior $\theta$ has two modes at $(\theta_1, \theta_2)=(0, 1)$ and $(\theta_1, \theta_2)=(1, -1)$. In order to get the bounds required by all methods, we truncate the Gaussian by setting $\theta_1, \theta_2 \in [-3,3]$. 

For simplicity we assume a flat prior $p(\theta)=1$. Then the energy is given by
\[
U_i(\theta) = -\log p(x_i|\theta) 
= \log(2\sqrt{2\pi}\sigma_x) - \log\bigg[\exp\bigg(-\frac{(x_i-\theta_1)^2}{2\sigma_x^2}\bigg) + \exp\bigg(-\frac{(x_i-\theta_1-\theta_2)^2}{2\sigma_x^2}\bigg)\bigg]. 
\]

Denote $E_1 = \exp\bigg(-\frac{(x_i-\theta_1)^2}{2\sigma_x^2}\bigg)$ and $E_2 = \exp\bigg(-\frac{(x_i-\theta_1-\theta_2)^2}{2\sigma_x^2}\bigg)$. 
To get the upper bound in \methodname{}, TFMH and SMH, we compute the gradient
\begin{align*}
 \frac{\partial U_i(\theta)}{\partial \theta_1} &= -\frac{1}{E_1 + E_2} \bigg(E_1\cdot \frac{x_i - \theta_1}{\sigma_x^2} + E_2 \cdot \frac{x_i-\theta_1-\theta_2}{\sigma_x^2}\bigg),\\ 
 \frac{\partial U_i(\theta)}{\partial \theta_2} &= -\frac{1}{E_1 + E_2} \bigg(E_2 \cdot \frac{x_i-\theta_1-\theta_2}{\sigma_x^2}\bigg).
\end{align*}
Since $\theta_i \in [-3, 3]$, it follows that
\begin{align*}
    \Abs{\frac{\partial U_i(\theta)}{\partial \theta_1}} &\le \frac{\Abs{x_i}+3}{\sigma_x^2} + \frac{\Abs{x_i}+3+3}{\sigma_x^2} \le \frac{2\Abs{x_i}+9}{\sigma_x^2},\\
    \Abs{\frac{\partial U_i(\theta)}{\partial \theta_2}} &\le \frac{\Abs{x_i}+3+3}{\sigma_x^2} \le \frac{\Abs{x_i}+6}{\sigma_x^2}.
\end{align*}
Therefore we can set $M(\theta,\theta') = \norm{\theta-\theta'}_2$ and 
\[
c_i = \sqrt{\bigg(\frac{2\Abs{x_i}+9}{\sigma_x^2}\bigg)^2 + \bigg(\frac{\Abs{x_i}+6}{\sigma_x^2}\bigg)^2}.
\]

To use the control variate in SMH, we need to compute the second derivatives
\begin{align*}
    \frac{\partial^2 U_i(\theta)}{\partial^2 \theta_1} &= \frac{1}{(E_1 + E_2)^2}\cdot \bigg(E_1\cdot \frac{x_i - \theta_1}{\sigma_x^2} + E_2 \cdot \frac{x_i-\theta_1-\theta_2}{\sigma_x^2}\bigg)^2 
    \\\hspace{2em}& - \bigg[E_1\cdot \bigg(\bigg(\frac{x_i - \theta_1}{\sigma_x^2}\bigg)^2 - \frac{1}{\sigma_x^2}\bigg) + E_2\cdot \bigg(\bigg(\frac{x_i - \theta_1 - \theta_2}{\sigma_x^2}\bigg)^2 - \frac{1}{\sigma_x^2}\bigg)\bigg] \cdot \frac{1}{E_1 + E_2}\\
    \frac{\partial^2 U_i(\theta)}{\partial \theta_1\partial \theta_2} &= \frac{1}{(E_1 + E_2)^2} \cdot \bigg(E_2\cdot \bigg(\frac{x_i - \theta_1 - \theta_2}{\sigma_x^2}\bigg)\bigg) \cdot \bigg(E_1\cdot \frac{x_i - \theta_1}{\sigma_x^2} + E_2 \cdot \frac{x_i-\theta_1-\theta_2}{\sigma_x^2}\bigg)
    \\\hspace{2em}& - \bigg[E_2 \bigg(\bigg(\frac{x_i - \theta_1 - \theta_2}{\sigma_x^2}\bigg)^2 - \frac{1}{\sigma_x^2}\bigg)\bigg]\cdot \frac{1}{E_1 + E_2}\\
    \frac{\partial^2 U_i(\theta)}{\partial^2 \theta_2} &= \frac{1}{(E_1 + E_2)^2}\cdot \bigg(E_1\cdot \frac{x_i - \theta_1}{\sigma_x^2} + E_2 \cdot \frac{x_i-\theta_1-\theta_2}{\sigma_x^2}\bigg)^2 
    \\\hspace{2em}& - \bigg[E_2\cdot \bigg(\bigg(\frac{x_i - \theta_1 - \theta_2}{\sigma_x^2}\bigg)^2 - \frac{1}{\sigma_x^2}\bigg)\bigg] \cdot \frac{1}{E_1 + E_2}.
\end{align*}
Given the parameter space, we have the upper bounds
\begin{align*}
    \Abs{\frac{\partial^2 U_i(\theta)}{\partial^2 \theta_1}} &\le  \bigg(\frac{2\Abs{x_i}+9}{\sigma_x^2}\bigg)^2 + \bigg(\frac{\Abs{x_i}+3}{\sigma_x^2}\bigg)^2 + \bigg(\frac{\Abs{x_i}+6}{\sigma_x^2}\bigg)^2 + \frac{2}{\sigma_x^2} \\
    \Abs{\frac{\partial^2 U_i(\theta)}{\partial \theta_1 \partial \theta_2}} 
    &\le \frac{2\Abs{x_i}+9}{\sigma_x^2} \cdot \frac{\Abs{x_i}+6}{\sigma_x^2} + \bigg(\frac{\Abs{x_i}+6}{\sigma_x^2}\bigg)^2 + \frac{1}{\sigma_x^2}\\
    \Abs{\frac{\partial^2 U_i(\theta)}{\partial^2 \theta_2}} &\le \bigg(\frac{2\Abs{x_i}+9}{\sigma_x^2}\bigg)^2 +  \bigg(\frac{\Abs{x_i}+6}{\sigma_x^2}\bigg)^2 + \frac{1}{\sigma_x^2}.
\end{align*}
It follows
\[
\Bar{U}_{2,i} = \bigg(\frac{2\Abs{x_i}+9}{\sigma_x^2}\bigg)^2 + \bigg(\frac{\Abs{x_i}+3}{\sigma_x^2}\bigg)^2 + \bigg(\frac{\Abs{x_i}+6}{\sigma_x^2}\bigg)^2 + \frac{2}{\sigma_x^2}. 
\]
which is required in SMH-1. 

To get the lower bounds in FlyMC, we use the first-order Taylor expansion for $U_i(\theta)$. Higher order approximation is possible but would require heavier computation. By Taylor expansion,
\begin{align*}
    U_i(\theta) = U_i(\theta^0) + \nabla U_i(\theta^0)^\intercal (\theta - \theta^0) + \frac{1}{2}(\theta - \theta^0)^\intercal \nabla^2 U_i(c) (\theta - \theta^0)
\end{align*}
where $c$ is between $\theta$ and $\theta^0$.

Then we can define $\log B_i(\theta)$ in FlyMC as the follows
\begin{align*}
    \log B_i(\theta) &=  -U_i(\theta^0) - \nabla U_i(\theta^0)^\intercal (\theta - \theta^0) - \frac{1}{2}\cdot\max_c\norm{\nabla^2 U_i(c)}_1 \cdot\norm{\theta - \theta^0}^2_1\\
    & = -U_i(\theta^0) - \nabla U_i(\theta^0)^\intercal (\theta - \theta^0) - \frac{1}{2}\cdot\Bar{U}_{2,i} \cdot\norm{\theta - \theta^0}^2_1.
\end{align*}

The sum of $\log B_i$ is
\begin{align*}
    \sum_{i=1}^N \log B_i(\theta) =  -N\cdot U_i(\theta^0) - \bigg(\sum_{i=1}^N\nabla U_i(\theta^0)\bigg)^\intercal (\theta - \theta^0) - \frac{1}{2}\cdot\sum_{i=1}^N \Bar{U}_{2,i} \cdot\norm{\theta - \theta^0}_1^2.
\end{align*}

We set $\theta^0$ to be 0 and the MAP solution in standard and MAP-tuned FlyMC respectively. 

We tune the stepsize of each method to reach the acceptance rate $60\%$ and the value of stepsize is summarized in Table \ref{tab:stepsize} and \ref{tab:stepsize2}. We set $\chi = 10^{-4}$ in \methodname{} and $q_{d\rightarrow b} = 0.01$ in FlyMC and FlyMC-MAP. We compute the symmetric KL between the run-average density estimate and the true distribution. Since this is a two-dimensional problem, we are able to visualize the density estimate. As shown in Figure \ref{app:fig:mog}, we plot the density estimate after running the method for 1 second. It is clear to see that the density estimate of \methodname{} is close to the truth whereas all other methods are unable to provide accurate density estimate given the time budget.

\begin{figure*}[t!]
    \centering
    \begin{tabular}{cccc}		
    	\includegraphics[width=3.5cm]{figs/mog_true.pdf}  &
    		\hspace{-6mm}
    	\includegraphics[width=3.5cm]{figs/mog_pmh.pdf} &
    	\hspace{-6mm}
    	\includegraphics[width=3.5cm]{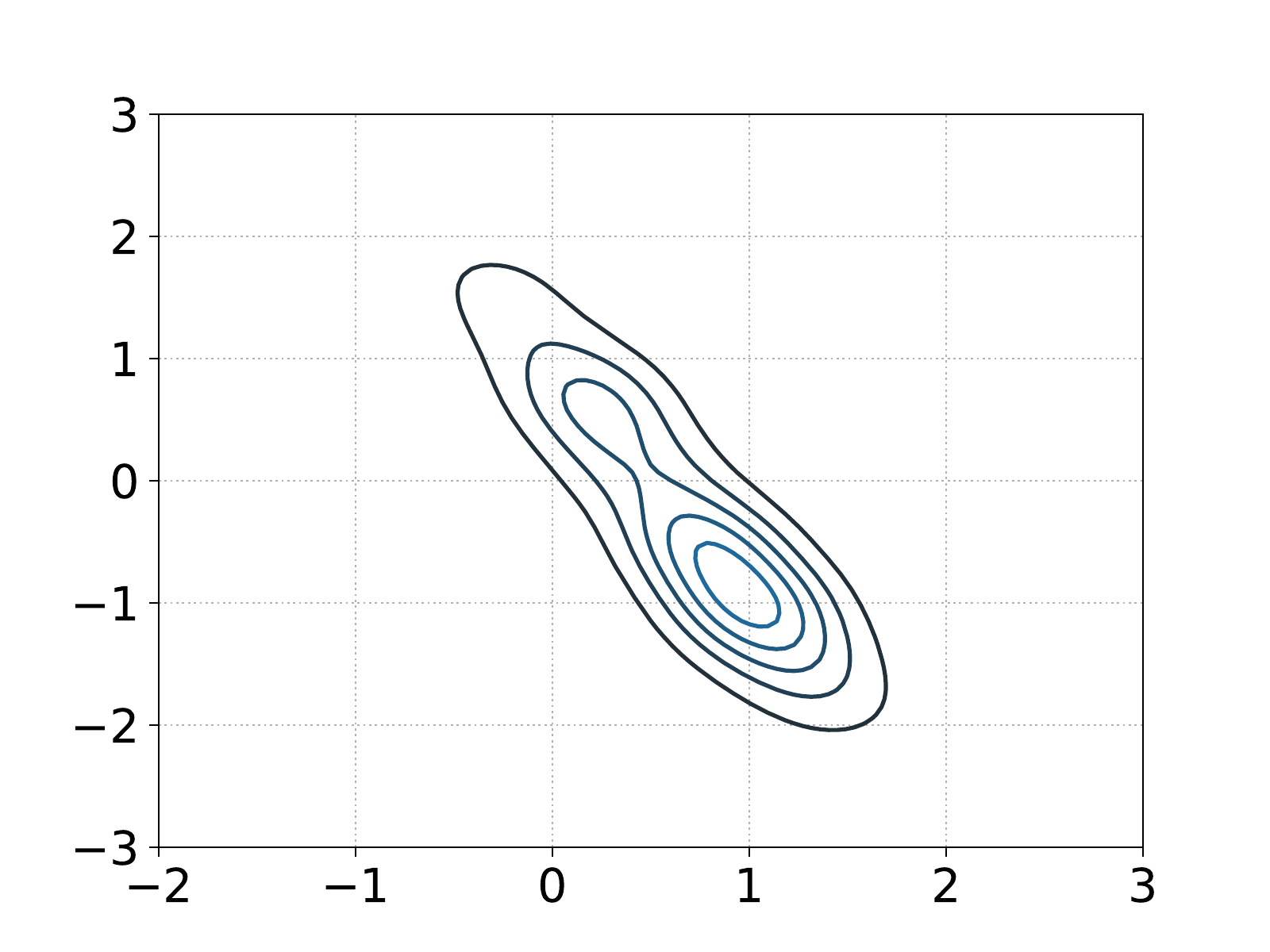} &
    	\hspace{-6mm}
    	\includegraphics[width=3.5cm]{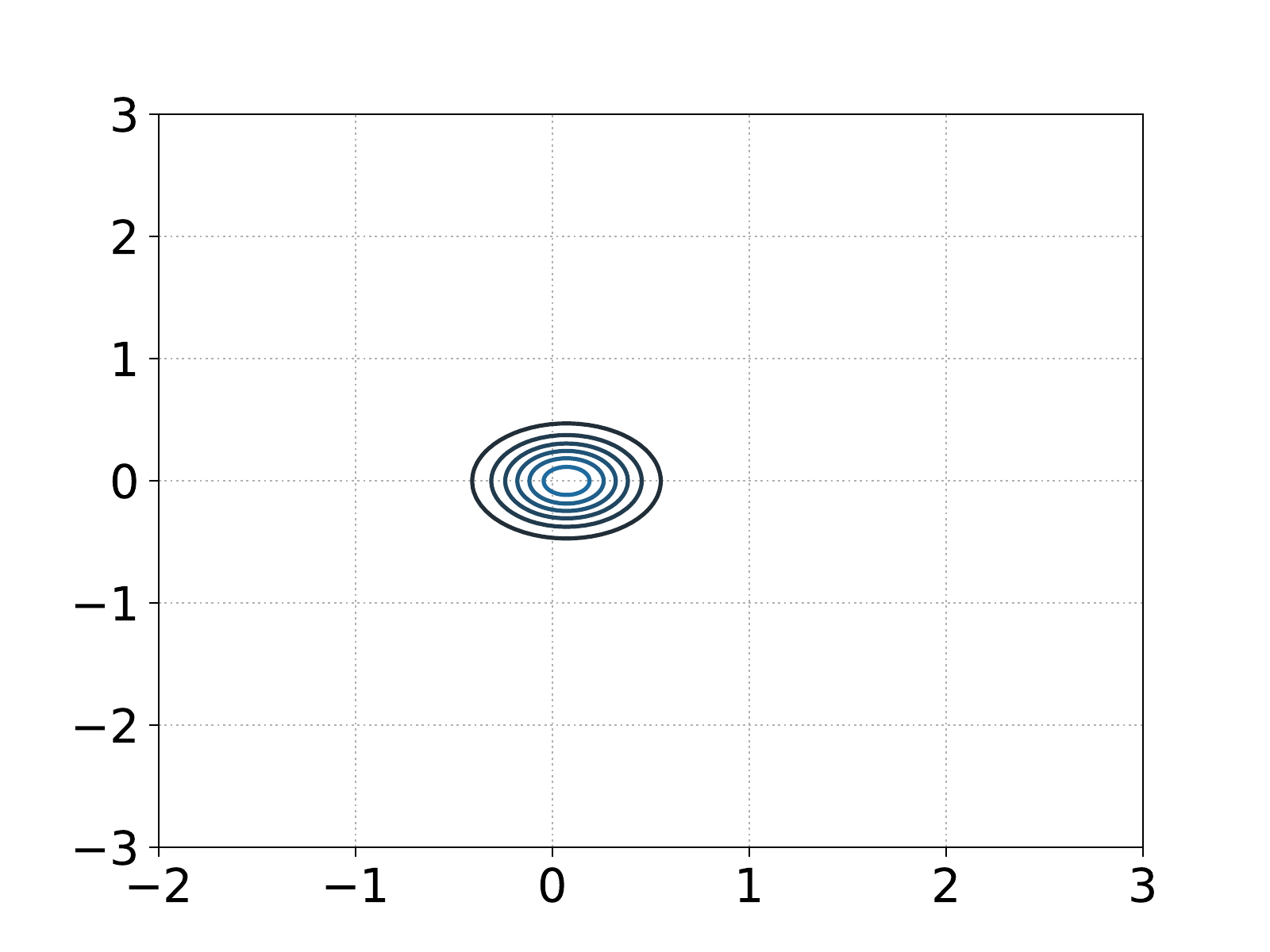}
    	\\		
    	(a) True&
    	(b) \methodname{}&
    	(c) TFMH&
    	(d) FlyMC
    	\\
    	\includegraphics[width=3.5cm]{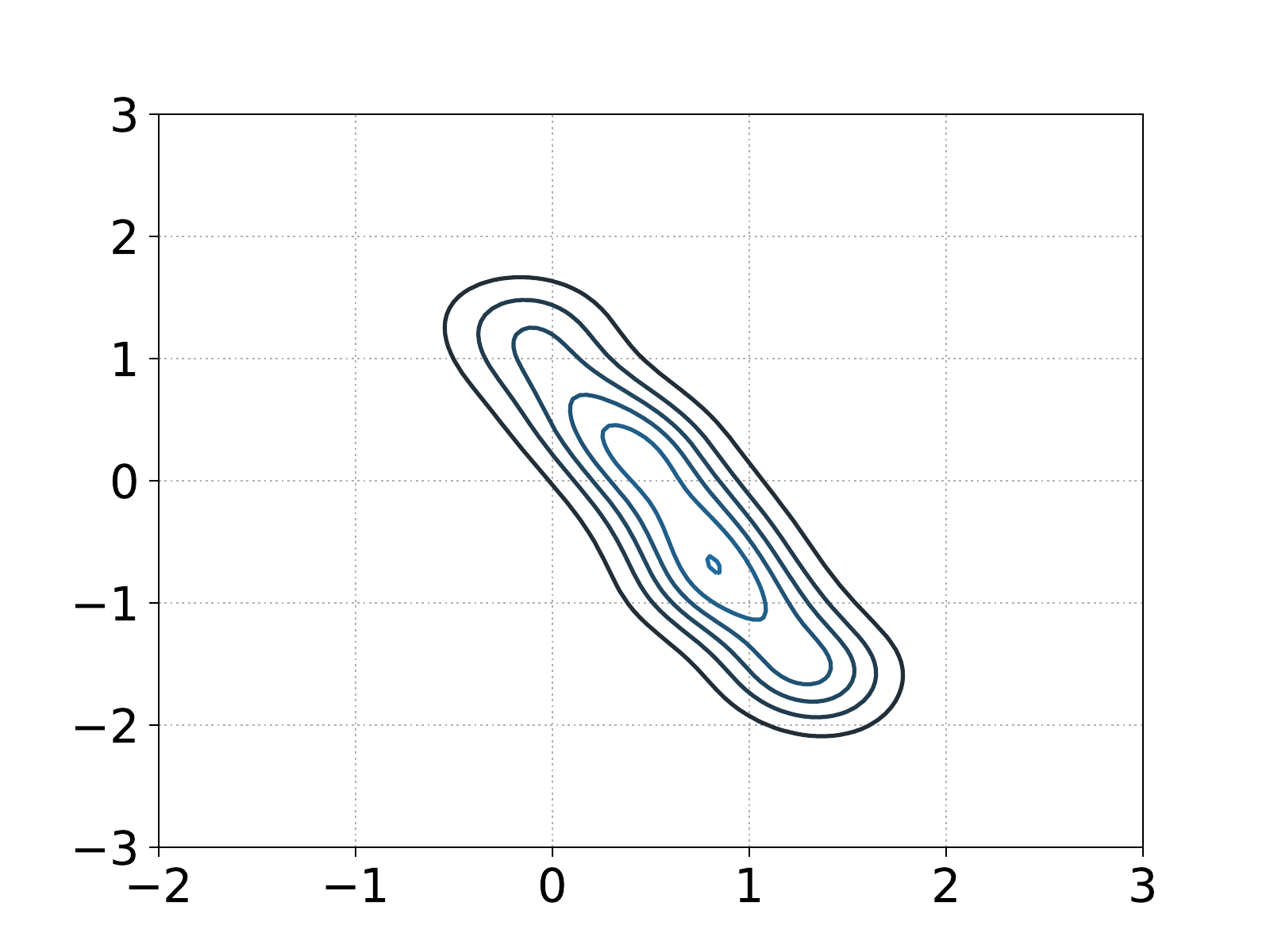}
    	&
    	\hspace{-6mm}
    	\includegraphics[width=3.5cm]{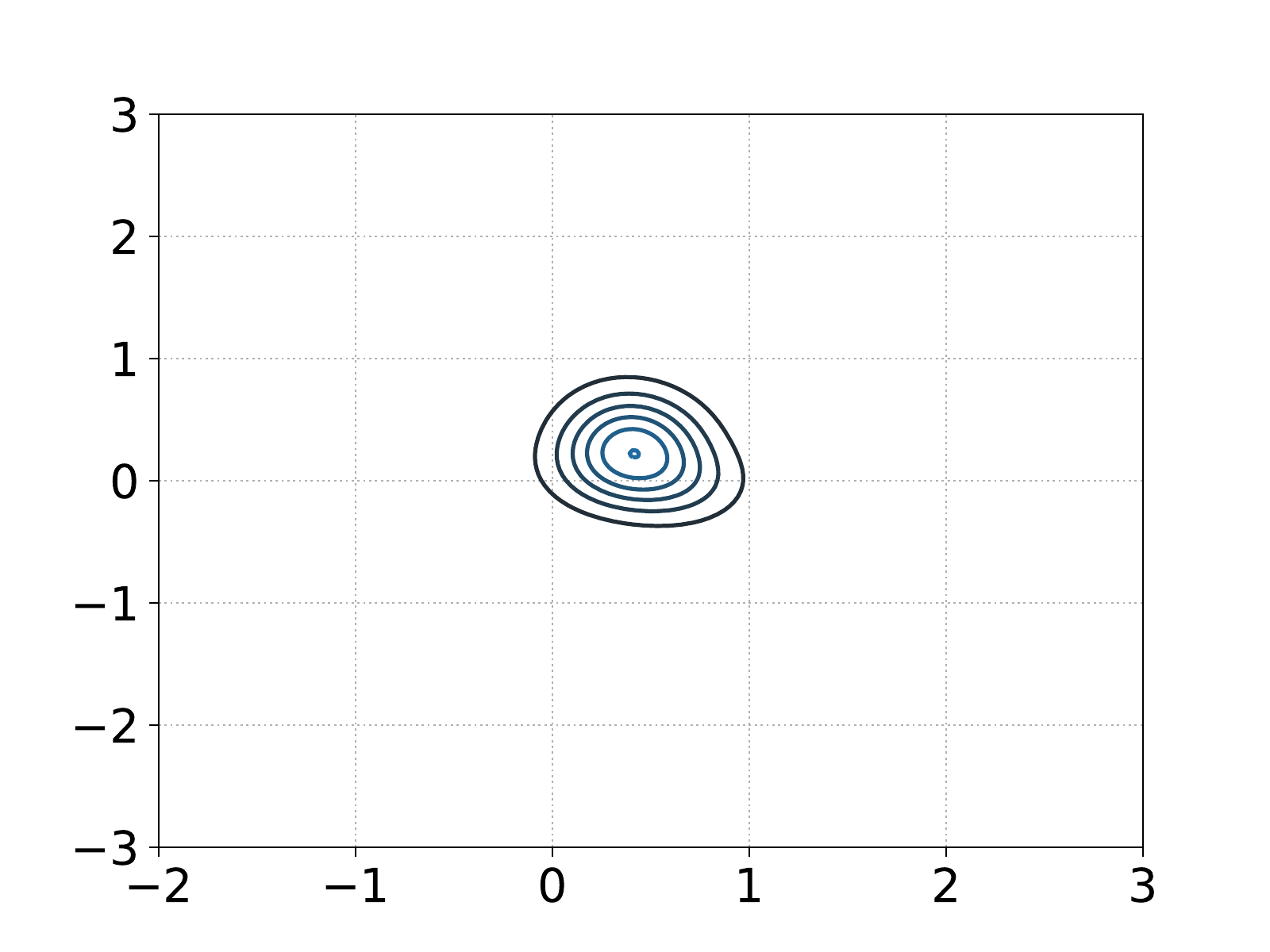}  &
    	\hspace{-6mm}
    	\includegraphics[width=3.5cm]{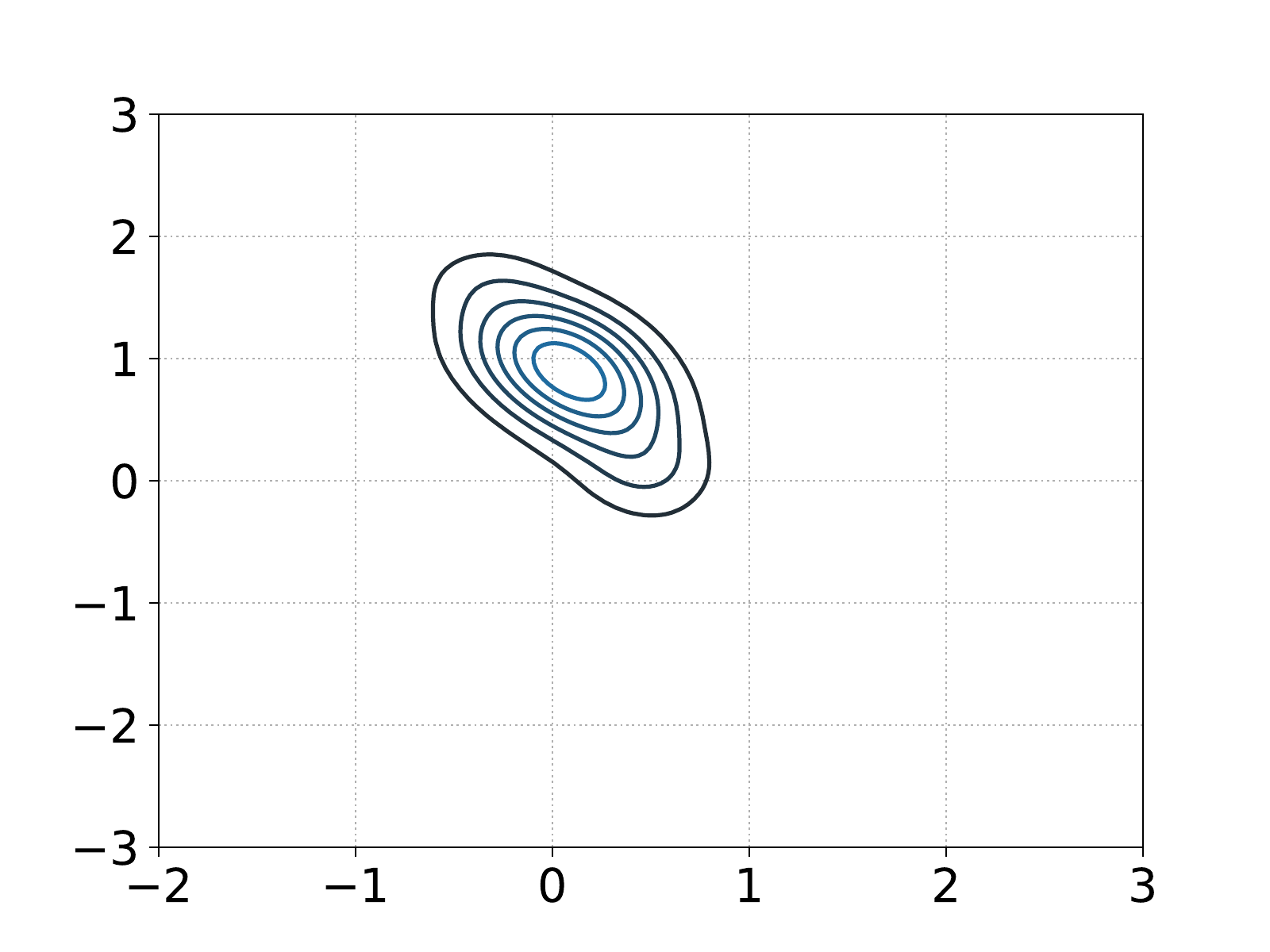} &
    	\hspace{-6mm}
    	\includegraphics[width=3.5cm]{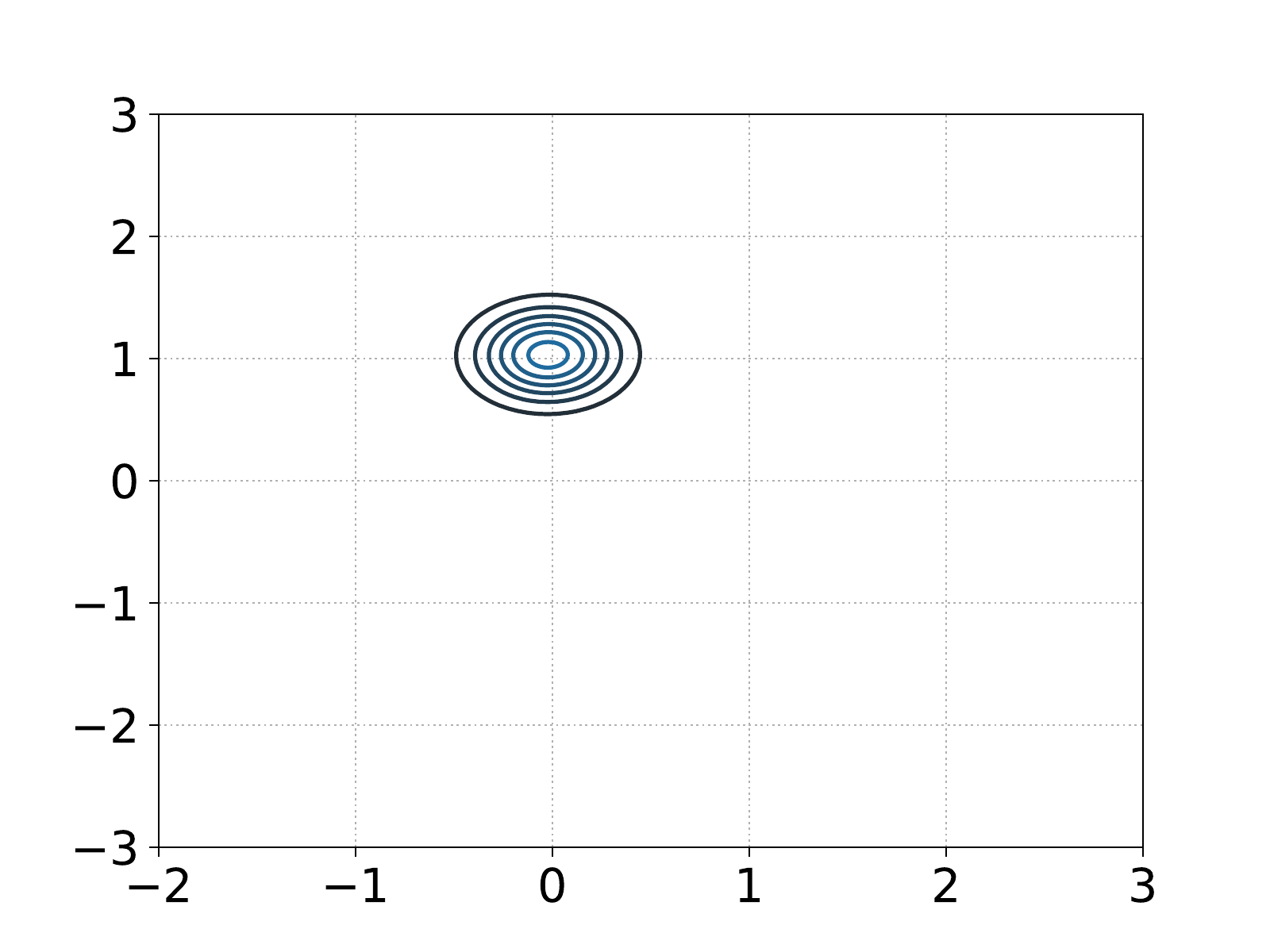}
    	
    	\\	
    	(e) PoissonMH&
    	(f) MH&
    	(g) SMH-1&
    	(h) FlyMC-MAP
    	\hspace{-0mm}\\		
    \end{tabular}
    \caption{Visualization of the density estimate after 1 second.}
    \label{app:fig:mog}
\end{figure*}

\subsection{Logistic Regression on MNIST}

MNIST with only 7s and 9s images contains 12214 training data and 2037 test data. Let $h$ be the sigmoid function. Let the label $y_i\in \{0, 1\}$, then the model in logistic regression (LR) is 
\[
p(y_i = 1) = h(\theta^\intercal x_i) = \frac{1}{1 + \exp\left(-\theta^\intercal x_i\right)}.
\]
It follows that 
\[
U_i(\theta) = -y_i\log h\left(\theta^\intercal x_i\right) - (1-y_i) \log h\left(-\theta^\intercal x_i\right).
\]

It is easy to see that
\[
\Abs{\partial_j U_i} = \Abs{(h(\theta^\intercal x_i)-y_i)x_{ij}} \le 1\cdot \Abs{x_{ij}}.
\] 
Thus we can set $M(\theta,\theta')$ to be $\norm{\theta-\theta'}_2$ and $c_i$ to be $\norm{x_{i}}_2$. We use this bound for \methodname{}, TFMH and SMH. For FlyMC, we use the same bound on logistic regression as in the FlyMC paper \cite{maclaurin2015firefly}.

We set the target acceptance rate to be $60\%$ and the resulted stepsize is reported in Table \ref{tab:stepsize}. We set $q_{d\rightarrow b}$ to be 0.1 following \cite{maclaurin2015firefly}.


\end{document}

%% file: mlVecMat.tex
\usepackage{amsopn,amsmath,amsthm,amssymb}
\usepackage{wrapfig}
\usepackage{subcaption}
\usepackage{multirow}
 \usepackage{mathtools}

\usepackage{anyfontsize}

\usepackage{color}
\usepackage{tikz}
\usetikzlibrary{arrows,shapes,snakes,automata,backgrounds,fit,petri}
\usepackage{adjustbox}

\newcommand{\RN}[1]{%
	\textup{\lowercase\expandafter{\it \romannumeral#1}}%
}

\makeatletter
\newcommand{\distas}[1]{\mathbin{\overset{#1}{\kern\z@\sim}}}%

\usepackage{enumitem}

\usepackage{algorithm}
\usepackage{algorithmic}

\newcommand{\beq}{\vspace{0mm}\begin{equation}}
\newcommand{\eeq}{\vspace{0mm}\end{equation}}
\newcommand{\beqs}{\vspace{0mm}\begin{eqnarray}}
\newcommand{\eeqs}{\vspace{0mm}\end{eqnarray}}
\newcommand{\barr}{\begin{array}}
\newcommand{\earr}{\end{array}}

\newcommand{\R}{\mathbb{R}}

\newtheorem{theorem}{Theorem} 
\newtheorem{lemma}{Lemma}

\newtheorem{corollary}{Corollary}

\ifx\assumption\undefined
\newtheorem{assumption}{Assumption}
\fi

\ifx\definition\undefined
\newtheorem{definition}{Definition}
\fi

\ifx\remark\undefined
\newtheorem{remark}{Remark}
\fi
\newcommand{\norm}[1]{\left\| #1 \right\|}

\newcommand{\Abs}[1]{\left| #1 \right| }